\documentclass{article}

\usepackage{microtype}
\usepackage{graphicx}
\usepackage{subcaption}
\usepackage{booktabs} % for professional tables

\usepackage{hyperref}

\usepackage[accepted]{icml2026}

\usepackage{amsmath}
\usepackage{amssymb}
\usepackage{mathtools}
\usepackage{amsthm}

\usepackage[capitalize,noabbrev]{cleveref}
\crefname{equation}{Eq.}{Eqs.}
\Crefname{equation}{Eq.}{Eqs.}
\crefname{figure}{Fig.}{Figs.}
\Crefname{figure}{Fig.}{Figs.}
\crefname{theorem}{Thm.}{Thms.}
\Crefname{theorem}{Theorem}{Theorems}
\Crefname{table}{Tab.}{Tabs.}

\newtheorem{definition}{Definition}[section]

\newtheorem{assumption}{Assumption}[section]

\newtheorem{lemma}{Lemma}[section]
\newtheorem{theorem}{Theorem}[section]
\newtheorem*{theorem*}{Theorem}

\usepackage{bbm}
\usepackage{bbold}

\usepackage[textsize=tiny]{todonotes}
\usepackage{xcolor}
\definecolor{eqcolor1}{rgb}{0.745, 0.314, 0.275}
\definecolor{eqcolor2}{rgb}{0.863, 0.529, 0.451}
\definecolor{eqcolor3}{rgb}{0.706, 0.373, 0.529}
\definecolor{eqcolor4}{rgb}{0.235, 0.490, 0.569}
\definecolor{eqcolor5}{rgb}{0.529, 0.569, 0.667}
\definecolor{eqcolor6}{rgb}{0.804, 0.725, 0.588}

\usepackage{amsmath, amsfonts, amsthm, amssymb}
\usepackage[mathscr]{eucal}
\usepackage{theoremref}

\newcommand{\N}{\mathbb{N}}
\newcommand{\R}{\mathbb{R}}
\newcommand{\spacefont}[1]{\mathcal{#1}}
\newcommand{\X}{\spacefont{X}}
\newcommand{\Y}{\spacefont{Y}}
\newcommand{\Z}{\spacefont{Z}}
\newcommand{\spX}{\mathcal{X}}

\newcommand{\norm}[1]{\lVert #1 \rVert}
\newcommand{\inner}[1]{\langle #1 \rangle}
\DeclareMathOperator{\tracespaced}{tr}
\newcommand{\trace}{\tracespaced}
\newcommand{\linearoperatorfont}[1]{\mathsf{#1}}
\DeclareMathOperator{\range}{range}
\DeclareMathOperator{\linspan}{span}
\DeclareMathOperator{\kernel}{kernel}

\newcommand{\LiiX}{L^2(X)}

\newcommand{\LiiYdot}{L^2(\ddot{Y})}
\newcommand{\LiiZ}{L^2(Z)}

\newcommand{\bgbox}[2]{#2}
\newcommand{\mX}{P_{X}} % measure on X
\newcommand{\mY}{P_{Y}} % measure on Y
\newcommand{\mZ}{P_{Z}} % measure on z
\newcommand{\mXY}{P_{X,Y}} %joint invariant measure
\newcommand{\mXYZ}{P_{X,Y,Z}}
\newcommand{\E}{\mathbb{E}}

\newcommand\independent{\protect\mathpalette{\protect\independenT}{\perp}}
\def\independenT#1#2{\mathrel{\rlap{$#1#2$}\mkern2mu{#1#2}}}

\DeclareMathOperator*{\argmin}{arg\,min}

\providecommand{\SVDd}[1]{[\![#1]\!]_d}

\newcommand{\covs}[2]{\linearoperatorfont{\Sigma}_{#1#2\vphantom{\ddot{#1}}}}
\newcommand{\covsPrivate}[2]{\linearoperatorfont{\Sigma}_{#1#2}}

\newcommand{\covsrddot}[2]{\covsPrivate{#1}{\ddot{#2}}}
\newcommand{\pcovs}[2]{\linearoperatorfont{\Sigma}_{#1 #2 \vphantom{\ddot{#1}} \boldsymbol{\cdot}\, Z}}
\newcommand{\pcovsPrivate}[2]{\linearoperatorfont{\Sigma}_{#1 #2 \boldsymbol{\cdot}\, Z}}

\newcommand{\pcovsrddot}[2]{\pcovsPrivate{#1}{\ddot{#2}}}
\newcommand{\DCE}{\pcovsrddot{X}{Y}}

\newcommand{\U}{\linearoperatorfont{U}}
\newcommand{\V}{\linearoperatorfont{V}}
\newcommand{\W}{\linearoperatorfont{W}}
\newcommand{\M}{\linearoperatorfont{M}}
\newcommand{\Nop}{\linearoperatorfont{N}}
\newcommand{\F}{\linearoperatorfont{F}}
\newcommand{\That}{\widehat{T}_n}

\newcommand{\ddy}{\ddot{y}}

\newcommand{\Unnest}{\widehat{\U}_\theta}

\newcommand{\unnest}[1]{\widehat{u}_{\theta,#1}}
\newcommand{\vnnest}[1]{\widehat{v}_{\theta,#1}}
\newcommand{\wnnest}[1]{\widehat{w}_{\theta,#1}}
\newcommand{\Unnwhite}{\widetilde{\U}_\theta}

\newcommand{\sqrtinvcovestUnnest}{\widehat{C}^{-1/2}_{\widetilde{U}_\theta\widetilde{U}_\theta}}

\newcommand{\covestUVnnest}{\widehat{C}_{\widehat{U}_\theta\widehat{V}_\theta}}
\newcommand{\covestUWnnest}{\widehat{C}_{\widehat{U}_\theta\widehat{W}_\theta}}
\newcommand{\covestWVnnest}{\widehat{C}_{\widehat{W}_\theta\widehat{V}_\theta}}

\newcommand{\unnestv}{\widehat{u}_\theta}
\newcommand{\vnnestv}{\widehat{v}_\theta}
\newcommand{\wnnestv}{\widehat{w}_\theta}

\newcommand{\traindata}{\mathcal{D}^{(\mathrm{train})}_m}
\newcommand{\testdata}{\mathcal{D}^{(\mathrm{test})}_n}

\newcommand{\Lossout}{\widehat{\mathcal{L}}_{\mathrm{out}}}
\newcommand{\Lossin}{ \widehat{\mathcal{L}}_{\mathrm{in}}}
\newcommand{\Regout}{\widehat{\Omega}_{\mathrm{out}}}
\newcommand{\Regin}{\widehat{\Omega}_{\mathrm{in}}}

\icmltitlerunning{Toward Scalable and Valid Conditional Independence Testing with Spectral Representations}

\begin{document}

\twocolumn[
  % \icmltitle{Spectral Representation Learning for Conditional Independence Testing}
    % \icmltitle{Toward Scalable and Valid Conditional Independence Testing \\ with Spectral Representations \texorpdfstring{\includegraphics[height=1.2em]{plots/emoji_ghost.pdf}}{}}
    \icmltitle{Toward Scalable and Valid Conditional Independence Testing \texorpdfstring{\\}{} with Spectral Representations}


  \icmlsetsymbol{equal}{*}

  \begin{icmlauthorlist}
    \icmlauthor{Alek Fr\"{o}hlich}{csml,unige}
    \icmlauthor{Vladimir Kosti\'{c}}{csml,novisad}
    \icmlauthor{Karim Lounici}{ecole}
    \icmlauthor{Daniel Perazzo}{csml,unige}
    \icmlauthor{Daniel Tiezzi}{usp}
    \icmlauthor{Massimiliano Pontil}{csml,ucl}
  \end{icmlauthorlist}

  \icmlaffiliation{csml}{CSML, Istituto Italiano di Tecnologia, Genoa, Italy}
  \icmlaffiliation{unige}{DIBRIS, University of Genoa, Genoa, Italy}
  \icmlaffiliation{novisad}{Faculty of Science, University of Novi Sad, Novi Sad, Serbia}
  \icmlaffiliation{ecole}{CMAP, \'{E}cole Polytechnique, Palaiseau, France}
  \icmlaffiliation{usp}{Breast Disease Division, University of São Paulo, Ribeirão Preto, Brazil}
  \icmlaffiliation{ucl}{Department of Computer Science, University College London, London, UK}

  \icmlcorrespondingauthor{Alek Fr\"{o}hlich}{alek.frohlich@iit.it}

  \icmlkeywords{Conditional Independence Testing, Representation Learning, Partial Cross-Covariance Operator}

  \vskip 0.3in
]



\printAffiliationsAndNotice{}  % no special notice (required even if empty)

\begin{abstract}
Conditional independence (CI) is central to causal inference, feature selection, and graphical modeling, yet it is untestable in many settings without additional assumptions. Existing CI tests often rely on restrictive structural conditions, limiting their validity. Kernel methods using partial covariance operators offer a more principled approach but suffer from limited adaptivity and scalability. In this work, we explore whether representation learning can help address these limitations. Specifically, we focus on representations derived from the singular value decomposition of partial covariance operators and use them to construct a simple test statistic. %, reminiscent of the Hilbert-Schmidt Independence Criterion (HSIC).
We also introduce a bi-level contrastive algorithm to learn these representations. Our theory links representation learning error to test performance and establishes asymptotic validity and power guarantees.
Experiments on real and synthetic data suggest that this approach offers a principled and statistically grounded path toward scalable CI testing, bridging kernel-based theory with modern representation learning.\footnote{Code repository: \href{https://github.com/alekfrohlich/SCIT}{https://github.com/alekfrohlich/SCIT}.}
%%%%%%%% V3:
% Conditional independence (CI) is central to causal inference, feature selection, and graphical modeling, yet it is untestable in many settings without additional assumptions. % If one takes "testable" to mean uniform type I error control, then CIT is untestable for 
% \textcolor{red}{An established approach represents conditional dependence via the partial covariance operator and estimates its norm using fixed kernel-induced features. While powerful, this approach suffers from limited representational flexibility, slow convergence, and poor scalability.
% In this work, we explore whether representation learning can help address these limitations.}
% Specifically, we focus on representations derived from the singular value decomposition of the partial covariance operator and use them to construct a simple test statistic, reminiscent of the Hilbert-Schmidt Independence Criterion (HSIC).
% We also introduce a practical contrastive algorithm to learn these representations. 
% Our theory connects representation quality to test performance and establishes asymptotic validity and power guarantees.
% Preliminary experiments suggest that this approach offers a practical and statistically grounded path toward scalable CI testing, bridging kernel-based theory with modern representation learning.
\end{abstract}

\section{Introduction}

% P1: Problem formulation, motivation, and concrete example
Given random variables $X$, $Y$, and $Z$, the goal of conditional independence (CI) testing is to decide between
\begin{equation}
\label{eq:H0vsH1contiguous}
\mathcal{H}_0: X \independent Y \mid Z \quad \text{vs.} \quad \mathcal{H}_1: X \not\independent Y \mid Z,
\end{equation}
using an i.i.d. sample from their joint distribution $P_{X,Y,Z}$. CI is a fundamental concept in statistics \cite{dawid79conditionalindependence} and machine learning \cite{scholkopf_toward_2021}, underlying causal inference \cite{Pearl2009,Spirtes2001,Imbens2015}, graphical models \cite{Lauritzen1996-ts,Koller2009-oq}, and variable selection \cite{Candes2018,huang_kernel_2022}.
% In causal discovery, for instance, CI relationships encode constraints on the underlying causal graph, enabling the identification of causal structures from observational data \cite{Glymour2019}.
% A concrete example comes from cancer genetics: distinguishing mutations that truly influence drug response from those that are merely correlated \cite{Barretina2012}. Here, $Y$ is the measured response of a cancer cell line to a targeted therapy, $X$ is the presence of a particular mutation, and $Z$ represents additional mutations that could affect drug response. A naïve association test between $X$ and $Y$ can be misleading, because many mutations co-occur or correlate through shared biological pathways, so any apparent dependence may be entirely explained by features in $Z$. A CI test asks whether the mutation $X$ still provides information about drug response $Y$ once we account for $Z$. 
A concrete example arises in computational pathology when integrating multi-modal data $(X,Z)$ to predict patient outcomes $Y$.
Let $X$ represent a tumor's molecular profile and $Z$ its visual (histological) features.
Tumor phenotype is known to reflect the underlying molecular state \cite{PMID:35465400}.
As a result, while $X$ may correlate with $Y$, this association can be redundant if the histological patterns encoded in $Z$ already capture the same biological signal.
A CI test determines whether $X$ offers incremental predictive power for $Y$ beyond the information already contained in the image features $Z$.
Despite its far-ranging applications, CI testing in nonparametric settings is known to be fundamentally challenging.
% \citet{bergsma2004testing} established that no test can distinguish CI when $Z$ is continuous without imposing additional assumptions.
% More recently, \citet{shah_hardness_2020} strengthened this impossibility result by proving that any test that controls type I error uniformly over all CI distributions must have no power against \textit{any} alternative.
% More recently, \citet{shah_hardness_2020} strengthened this impossibility result, proving that any test controlling type I error uniformly over all CI distributions lacks power against \textit{any} alternative.
\citet{shah_hardness_2020} established that any test controlling type I error uniformly over all conditionally independent distributions lacks power against \textit{any} alternative, helping explain why existing CI tests can fail to control type I error in practice.
The difficulty stems from the counterintuitive fact that any conditionally dependent sample can be approximated arbitrarily well by one that is conditionally independent.\footnote{Cf. \citep[Lemmas 3.2 \& A.1]{neykov_minimax_2021}.}
% The key insight lies in the counterintuitive fact that for any conditionally dependent sample, one can construct a conditionally independent sample arbitrarily close in a precise sense
% The key insight lies in the counterintuitive fact that for any conditionally dependent sample $\{(X_i, Y_i, Z_i)\}_{i=1}^n$, one can construct a CI sample $\{(\tilde{X}_i, \tilde{Y}_i, \tilde{Z}_i)\}_{i=1}^n$ arbitrarily close in a precise sense% \citep[Lemma A.1]{neykov_minimax_2021}.
% .\footnote{Cf. \citep[Lemmas 3.2 \& A.1]{neykov_minimax_2021}.}
In other words, CI distributions are so rich that even strongly dependent data could have arisen from a CI model, unless structural assumptions are imposed on $\mXYZ$, for example that $P_{X,Y|Z=z}$ varies continuously with $z$.
% Thus, additional structural assumptions are necessary to make CI testing well-posed.
% This counterintuitive fact stems from the existence of CI distributions $P_{X,Y|Z=z}$ that are widely discontinuous as functions of $z$.
This stands in stark contrast to the unconditional case, for which there exist uniformly valid tests, even in finite samples \citep{hoeffding1948anonparametric,Berrett2019nonparametric}.
%and cannot be replicated for the unconditional case ($\mathcal{H}_0\!:\! X\independent Y$ vs. $\mathcal{H}_1 \!:\! X\not\independent Y$). In fact, there exists nonparametric tests for unconditional independence that that control type I error uniformly over all unconditionally independent distributions and finite sample sizes, such as \citep{hoeffding1948anonparametric}.
% \citet{neykov_minimax_2021} illustrated this hardness by showing that conditionally independent distributions $P_{XY|Z=z}$ can be widely discontinuous as functions of $z$, which can be used to approximate any tuple $(X,Y,Z)$ with $(\Tilde{X},\Tilde{Y},\Tilde{Z})$ which is arbitrarily close but is CI.

\begin{figure*}[t]
    \centering
    \includegraphics[width=0.85\textwidth]{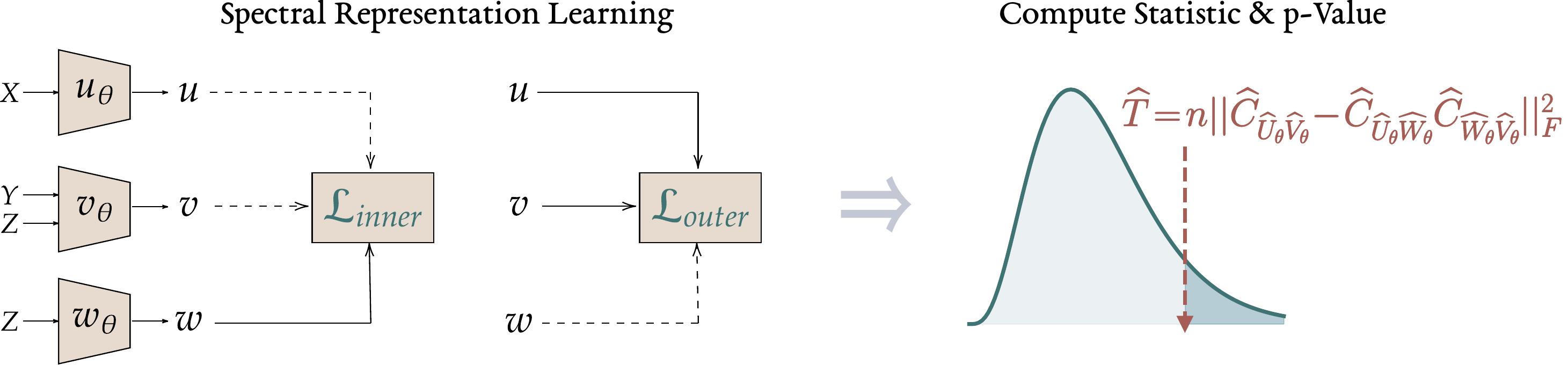}
    \caption{\textbf{SpectralCIT}'s testing pipeline. First the features $\unnestv,\vnnestv,\wnnestv$ are learned using \Cref{alg:learning_pcov} with whitening over the training set. Then, the test statistic $\That$ is computing using \Cref{eq:test_statistic} over the test set. Finally, a decision is made using the $1-\alpha$ quantile of the chi-squared distribution with $d^2$ degrees of freedom, where $d$ is the output dimension of the networks.}
    \label{fig:SCIT_diagram}
\end{figure*}

% P3: Restrictive structural assumptions (implicitly); in practice, known distributions, regressions, etc.
% This realization has motivated recent work leveraging structural assumptions—such as smoothness, known conditional distributions, or simple marginals—to design valid and powerful CI tests.
% \citet{shah_hardness_2020} proposed a regression-based approach, whose validity lies on the minimization of regression losses from $Z$ to $X$ and $Y$.
% \citet{neykov_minimax_2021} derived minimax lower and upper bounds for CI testing under smoothness constraints on $P_{X,Y \mid Z = z}$, proposing binned U-statistic tests for one-dimensional $X,Y,Z$ that are rate-optimal but require calibration with unknown constants, complicating their practical use. \citet{katsevich_power_2022} addressed CI testing under the \emph{model-X} assumption, where $P(X | Z)$ is known, enabling powerful conditional randomization tests but limiting applicability to settings with precise knowledge of $P(X | Z)$. \citet{kim_local_2022} provided theoretical foundations for local permutation tests, offering an explicit calibration method with double-binning strategies to improve robustness, yet incurring significant computational cost and relying on smoothness assumptions.
% Yet there is no free lunch: all CI tests, implicitly or explicitly, impose structural assumptions on $\mXYZ$ under both the null and alternative, and lose validity or power when these assumptions fail.
This realization has motivated a shift from designing universally valid tests to methods tailored for particular settings.
Kernel-based tests such as \textbf{KCIT} \citep{zhang_kernel-based_2012} and \textbf{RCIT} \cite{Strobl2018} rely on regressions from $Z$ to $X$ and $Y$, remaining valid as long as these regressions can be learned sufficiently well \cite{shah_hardness_2020,he2025on,pogodin2025practicalkerneltestsconditional}.
Model-x tests such as \textbf{GCIT} \cite{bellot2019gcit} or \textbf{DGCIT} \citep{shi2021doublegenerative} assume access to $P(X|Z)$, remaining valid when this distribution or a reliable approximation is available \cite{katsevich_power_2022,zhang_doubly_2024}.
Local permutation tests such as \textbf{NNLSCIT} \citep{li2023k} generate empirical p-values by permuting samples with clustered/binned $Z$ values, but incur substantial computational cost and depend on smoothness assumptions \cite{kim_local_2022}.
We defer a broader discussion on existing methods to \Cref{sec:related_work}.

Among these approaches, classical kernel-based tests stand out for modeling conditional dependence through the partial covariance operator, whose properties implicitly encode a wide range of structural assumptions on the joint distribution of $(X, Y, Z)$, including smoothness, sparsity, low-rank, and latent variable models.
Despite the generality of this operator-based framework, its practical impact has remained limited due to the lack of adaptivity and scalability of kernel methods \cite{pogodin2025practicalkerneltestsconditional,Ramdas2015}.

% In this work, we approach conditional independence testing through the lens of operator learning. Specifically, we represent the conditional dependence between $X$ and $Y$ given $Z$ via a cross-covariance operator whose structure reflects the relationships among the three variables. This formulation offers a flexible and general framework: a wide variety of structural assumptions on the joint distribution of $(X, Y, Z)$—such as smoothness, sparsity, low-rank, or latent variable models—are implicitly captured by the properties of the learned operator. Crucially, rather than tailoring a specific test to each possible structural setting, we perform a unified hypothesis test based on analyzing the singular values of this operator. This enables us to design a test that adapts to diverse structural scenarios without requiring prior knowledge of the true underlying assumptions, thus providing a practical and broadly applicable solution to conditional independence testing.

% P5: Our contributions
{\bf From kernels to representation learning.}
Recently, methods based on learning leading spectral features of statistical operators have shown promise across nonparametric inference tasks, including causal effect estimation \cite{sun25speciv,wang2022spectralrepresentationlearningconditional,meunier2025outcomeawarespectralfeaturelearning,meunier2025demystifyingspectralfeaturelearning}, reinforcement learning \cite{hu2024primaldualspectralrepresentationoffpolicy}, and learning dynamical systems \cite{turri2025self}, while remaining simple and scalable due to their connection to contrastive learning.
Motivated by the above examples and classical kernel CI tests, in this paper we explore whether spectral representation learning can address the limitations of kernel methods for scalable and valid CI testing.
In particular, we provide a contrastive learning algorithm to learn the leading spectral features of the partial covariance operator, using them to construct a simple test statistic reminiscent of the Hilbert-Schmidt Independence Criterion \cite{gretton2005hsic}.
We then analyze the behavior of the test statistic under both null and alternative, proving that it converges asymptotically to a chi-squared distribution under the null, and achieves power under the alternative.
Our approach is illustrated in \Cref{fig:SCIT_diagram}.
% We also report the results of preliminary numerical experiments.

{\bf Contributions.} In summary, our main contributions are:\\
    {\bf i.} We introduce a simple, scalable algorithm that learns the leading spectral features of the partial covariance operator, overcoming the adaptivity and scalability bottlenecks of classical kernel CI tests.\\
     {\bf ii.} We provide a comprehensive theoretical analysis of CI testing using the learned representations in combination with a simple test statistic. This includes type I error and power guarantees along with the characterization of the null distribution as asymptotically chi-squared.\\
     % {\bf iii.} We report experiments on real-world and synthetic data which suggest that our approach offers a principled and statistically grounded path toward scalable CI testing, bridging kernel-based theory with modern representation learning. \AF{TODO: limitations of evaluations}
     {\bf iii.} We validate our theory on challenging real and synthetic data, including a novel nonsmooth, high-dimensional variant of the post-nonlinear model and multi-modal breast cancer data \citep{TCGA}, showing that our approach offers a principled and statistically grounded path toward scalable CI testing, bridging kernel-based theory with modern representation learning.
%\end{itemize}

{\bf Paper organization.}
In \cref{sec:background}, we introduce notation and review statistical hypothesis testing, covariance operators, and spectral representation learning.
\cref{sec:representation_learning} presents an algorithm for learning partial covariance operators.
In \Cref{sec:ci_testing}, we present our operator framework for CI testing alongside our main theoretical results.
\cref{sec:experiments} reports numerical experiments.
\cref{sec:related_work} reviews related work.
All proofs are deferred to the appendix.

\section{Background}\label{sec:background}

{\bf Data spaces.} $X$, $Y$, $Z$ take values in measurable spaces $(\X, \mathcal{F}_\X)$, $(\Y, \mathcal{F}_\Y)$, $(\Z, \mathcal{F}_\Z)$ with marginals $\mX, \mY, \mZ$ and joint distributions $\mXY, P_{X,Z}, P_{Y,Z}, \mXYZ$.
In particular, $X,Y,Z$ might be continuous, discrete, or mixed random vectors of arbitrary dimensions $d_X, d_Y, d_Z$.
Throughout the paper, we take $\Ddot{Y} = (Y,Z)$ and assume $P_{XYZ}\ll P_{X}\times P_{YZ}$, $P_{XZ}\ll P_{X}\times P_{Z}$ and $P_{YZ}\ll  P_{Y}\times P_{Z}$, where $P \ll Q$ denotes $P$ is absolutely continuous w.r.t. $Q$.

%$ P_{X,Z}, P_{\ddot{Y},Z}, P_{X,\ddot{Y}}$ are absolutely continuous with respect to $P_{X} \times P_{Z}, P_{\ddot{Y}} \times P_{Z}, \mX \times \mYdot$.

{\bf Statistical hypothesis testing.} We are given a sample $\mathcal{D}_n = \{(X_i, Y_i, Z_i)\}_{i=1}^n$ and seek to decide between a null hypothesis $\mathcal{H}_0$ and an alternative hypothesis $\mathcal{H}_1$. This decision is based on a \emph{test statistic} $\That = T(\mathcal{D}_n)$, with large values of $\That$ typically providing evidence against $\mathcal{H}_0$. The test rejects $\mathcal{H}_0$ whenever $\That \geq c_\alpha$ where the \emph{critical value} $c_\alpha$ is chosen so that the \emph{Type I error}—the probability of incorrectly rejecting $\mathcal{H}_0$ when it is true—is controlled at a prescribed significance level $\alpha \in (0,1)$:
$
%\begin{equation}
\mathbb{P}_{\mathcal{H}_0}(\That \geq c_\alpha) = \alpha.
%\end{equation}
$
% Alternatively, one may compute the \emph{p-value}, which is the smallest significance level $\alpha$ at which $\mathcal{H}_0$ would be rejected given the observed value of $\That$. The p-value provides a measure of the strength of evidence against the null hypothesis.

Under the alternative, the \emph{power} of the test is the probability of correctly rejecting $\mathcal{H}_0$:
$
%\begin{equation}
\text{Power} = \mathbb{P}_{\mathcal{H}_1}(\That \geq c_\alpha).
%\end{equation}
$
The goal is to design tests that achieve high power while controlling the Type I error at level $\alpha$. In nonparametric settings, this trade-off is particularly challenging due to the broad class of possible alternatives. Our objective is thus to develop conditional independence tests that rigorously control Type I error and retain competitive power in a wide range of structural scenarios. We refer to \Cref{app:backgroundtest} for more details on statistical hypothesis testing.

{\bf Function spaces.} For a random variable $A$, $L^2(A)$ denotes the space of square-integrable functions ($\E[f(A)^2] < \infty$).

{\bf Operators on Hilbert spaces.} Let $\mathcal{F,G}$ be Hilbert spaces. For a bounded operator $\mathsf{T}: \mathcal{G}\to\mathcal{F}$, we denote $\norm{\mathsf{T}}$ its operator norm, $\norm{\mathsf{T}}_{\mathrm{HS}}$ its Hilbert-Schmidt norm, and $\mathsf{T}^*$ its adjoint. Given a second operator $\mathsf{S}: \mathcal{G}\to\mathcal{F}$, we denote $\inner{\mathsf{T}, \mathsf{S}}_{\mathrm{HS}} = \trace(\mathsf{TS}^*)$ the Hilbert-Schmidt inner product. For matrices, the Hilbert-Schmidt norm coincides with Frobenious norm, which we denote $\norm{M}_{\mathrm{F}}$ for a matrix $M$. For $f \in L^2(A), g,h \in L^2(B)$, the rank-one operator $f\otimes g$ is defined as $[f\otimes g](h) = \inner{g,h}f$, generalizing $fg^\top$ for vectors $f,g$.

{\bf Covariance operators.} For random variables $A,B$ such that $P_{AB}$ is absolutely continuous with respect to $P_{A}\times P_{B}$, we define the cross-covariance operator as the unique bounded linear operator satisfying\footnote{This is a direct consequence of the Riesz representation theorem \citep[Theorem II.4]{Reed1981}.}
\begin{equation}\label{eq:cov_op}
    \inner{f, \covs{A}{B} g} = \E\big[f(A)g(B)\big] - \E\big[f(A)\big]\E\big[g(B)\big],
\end{equation}
for every $f \in L^2(A), g \in L^2(B)$. Covariance operators generalize (centered) covariance matrices $C_{AB}=\E (A-\E A)(B-\E B)^\top$ and characterize independence ($A\independent B$ iff $\norm{\mathsf{\Sigma}_{AB}}_{\mathrm{HS}} = 0$). Given a third random variable $C$, and assuming $\mathsf{\Sigma}_{AB},\mathsf{\Sigma}_{AC},\mathsf{\Sigma}_{CB}$ are well-defined, we define the partial cross-covariance operator as\footnote{This definition differs from RKHS-based formulations involving $\mathsf{\Sigma}_{CC}^{-1}$ (e.g. \cite{Strobl2018}). Working directly on $L^2$ spaces, $\mathsf{\Sigma}_{CC}$ acts as the identity on the subspace of centered functions, eliminating the need for an explicit inverse.}
\begin{equation}\label{eq:pcov_op}
    \mathsf{\Sigma}_{AB\cdot C} = \covs{A}{B} - \covs{A}{C}\covs{C}{B}.
\end{equation}
Partial covariance operators generalize partial covariances and characterize conditional independence ($A \independent B | C$ iff $\norm{\mathsf{\Sigma}_{A(B,C)\cdot C}}_{\mathrm{HS}} = 0$; \cite{Strobl2018}).

{\bf Singular value decomposition (SVD).} A compact operator between Hilbert spaces $\mathsf{T}\!:\! \mathcal{G} \to \mathcal{F}$ can be written as $\mathsf{T} = \sum_{i=1}^\infty \sigma_i \phi_i \otimes \psi_i$, where $\sigma_1 \geq \sigma_2 \geq \dots \geq 0, \sigma_i \to 0$, are scalars (singular values) and $\phi_i, \psi_j$ are orthonormal functions (singular functions) in $\mathcal{F,G}$. Furthermore, by  Eckart-Young-Mirsky’s theorem \cite{Eckart1936}, the best rank-$d$ approximation of $\mathsf{T}$ with respect to any unitarily invariant norm %$\norm{\cdot}_{\text{HS}}$ 
is given by $[\![\mathsf{T}]\!]_d = \sum_{i=1}^d \sigma_i \phi_i \otimes \psi_i = \mathsf{\Phi \Sigma \Psi}^\star$, where $\mathsf{\Phi} \!=\! \left[ \phi_1 \vert \cdots \vert \phi_d \right]:\R^d \to \mathcal{F}$ and $\mathsf{\Psi} \!=\! \left[ \psi_1 \vert \cdots \vert \psi_d \right]:\R^d \to \mathcal{G}$ act via $\mathsf{\Phi}\alpha = \sum_{i=1}^d \alpha_i \phi_i$ and $\mathsf{\Psi}\beta = \sum_{i=1}^d \beta_i \psi_i$, and $\mathrm{\Sigma} = \text{diag}(\sigma_1, \dots, \sigma_d)$.

% \vladi{Careful, centering needs to be added for the claim to be true, another option is to define Sigma differently without centering}

{\bf Spectral representation learning.} To learn the rank-$d$ truncated SVD of a compact covariance operator $\covs{A}{B}$, \citet{kostic2024neuralconditional} have shown that it suffices to minimize the following regularized loss for $\gamma > 0$
\begin{equation*}
\begin{aligned}
\widehat{\mathcal{L}}_{\gamma}(\theta) &= -\frac{2}{m}\sum_{i=1}^m\inner{\overline{u}_\theta(A_i), M_\theta \overline{v}_\theta(B_i)}^2 \\
&\quad + \frac{1}{m(m-1)} \sum_{i\neq j}^m\inner{\overline{u}_\theta(A_i), M_\theta \overline{v}_\theta(B_j)}  + \gamma\, \widehat{\Omega}(\theta),
\end{aligned}
\end{equation*}
where $\{(A_i,B_i)\}_{i=1}^m$ is a batch sampled i.i.d. from $P_{AB}$, $\overline{u}_\theta,\overline{v}_\theta$ are neural nets whose outputs' have been empirically centered ($\overline{u}_\theta(A_i) = u_\theta(A_i) - \frac{1}{m}\sum_{j=1}^m u_\theta(A_j)$), $M_\theta$ is a matrix, and $\widehat{\Omega}(\cdot)$ is orthonormality regularization %$\widehat{\Omega}(\theta) = \norm{\widehat{C}_{U_\theta U_\theta} - I_d}_F^2 + \norm{\widehat{C}_{V_\theta V_\theta} - I_d}_F^2.$
\begin{equation*}
\widehat{\Omega}(\theta) = \norm{\widehat{C}_{U_\theta U_\theta} - I_d}_F^2 + \norm{\widehat{C}_{V_\theta V_\theta} - I_d}_F^2.
\end{equation*}
The connection with the SVD of $\mathsf{\Sigma}_{AB}$ follows from the Eckart–Young–Mirsky theorem (cf.~\Cref{thm:eckart_young_mirsky_loss}).
In particular, minimizing
\begin{equation*}
\mathcal{L}_\gamma(\theta) :=\norm{\mathsf{\Sigma}_{AB} - \mathsf{U}_\theta\mathsf{M}_\theta\mathsf{V}_\theta^*}_\mathrm{HS}^2 - \norm{\mathsf{\Sigma}_{AB}}_\mathrm{HS}^2 + \gamma\,\Omega(\theta)
\end{equation*}
yields the rank-$d$ truncated SVD of $\mathsf{\Sigma}_{AB}$: $\widehat{u}_{\theta,i}$ and $\widehat{v}_{\theta,i}$, $i \in [d]$, correspond to the left and right singular functions of $\mathsf{\Sigma}_{AB}$, while $\mathsf{M}_\theta$ is the diagonal matrix containing the associated singular values.
The empirical loss $\widehat{\mathcal{L}}_{0}(\theta)$ is simply a U-statistic estimator of $\mathcal{L}_0(\theta)$ \cite{kostic2024neuralconditional,sun25speciv}.

% and comes from a U-statistic estimator of $\norm{\mathsf{\Sigma}_{AB} - \mathsf{U}_\theta\mathsf{\Sigma}_\theta\mathsf{V}_\theta^\star}_\mathrm{HS}^2 - \norm{\mathsf{U}_\theta\mathsf{\Sigma}_\theta\mathsf{V}_\theta^\star}_\mathrm{HS}^2$, whose minimizer is the rank-$d$ truncated SVD of $\mathsf{\Sigma}_{AB}$ }

% For more details, see \Cref{sec:related_work}.
Prior work on spectral representation learning focuses on estimation and uncertainty quantification, rather than on independence or conditional independence testing.
Extending this framework to CI testing is nontrivial: the partial covariance operator $\mathsf{\Sigma}_{A(B,C)\cdot C}$ involves residualization with respect to $C$, which is not directly observable from data and must be handled implicitly.
This fundamentally changes the representation learning problem and leads to the bilevel formulation introduced in \Cref{sec:representation_learning}.
As a result, both the representation learning stage and the subsequent statistical analysis in \Cref{sec:ci_testing} differ substantially from existing literature.
\section{Learning Partial Covariance Operators}
\label{sec:representation_learning}

In this section, we describe a algorithm for learning the leading spectral features of the partial covariance operator $\DCE$ associated with a random triple $(X,Y,Z)$, using an i.i.d. sample $\traindata = \{(X_i, Y_i, Z_i)\}_{i=1}^m$ from their joint distribution $\mXYZ$. Following \citet{kostic2024neuralconditional}, we make the following mild regularity assumption for the SVD of $\DCE$ to be well-defined.
% , noting it holds for a large class of (discrete and continuous) distributions,\footnote{In particular, it holds when $dP_{XYZ}/ dP_{X}dP_{YZ}\in L^2(X)\otimes L^2(YZ)$ and $dP_{XZ}/ dP_{X}dP_{Z}\in L^2(X)\otimes L^2(Z)$. } including those that do not have density w.r.t. the Lebesgue measure.
%which is well-defined bounded operator whenever $P_{XYZ}\ll P_{X}\otimes P_{YZ}$, $P_{XZ}\ll P_{X}\otimes P_{Z}$ and $P_{YZ}\ll  P_{YZ}\otimes P_{Z}$.
% More details, together with the full derivation of the losses and \Cref{alg:learning_pcov} are provided in \Cref{sec:app_learning_pcov}.
%
\begin{assumption}\label{ass:dce_compact}
$\DCE: L^2(\ddot{Y}) \to L^2(X)$ is a compact operator.
\end{assumption}
This assumption holds for a large class of discrete and continuous distributions, including those without densities with respect to the Lebesgue measure.
A sufficient condition is given by the Radon–Nikodym derivatives $\kappa_{X,\ddot{Y}} \!=\! \frac{dP_{X,\ddot{Y}}}{d(P_{X}\times P_{\ddot{Y}})}$ and     $\kappa_{X,Z} \!=\! \frac{dP_{X,Z}}{d(P_{X}\times P_{Z})}$ being square-integrable: $\E_{P_{X,\ddot{Y}}}[\kappa_{X,\ddot{Y}}(X,\ddot{Y})^2], \E_{P_{X,Z}}[\kappa_{X,Z}(X,Z)^2] < \infty$.

{\bf Variational formulation of SVD.}
Given $d\!\in\!\mathbb{N}$, our aim is to learn the best rank-$d$ approximation of $\DCE$, that is,
\begin{equation}\label{eq:pco_svd}
[\![\DCE]\!]_d = \sum_{i=1}^d \sigma_i u_i\otimes v_i,    
\end{equation}
where $\sigma_i\!\geq\!0$, $u_i\!\in\!\LiiX$, and $v_i\!\in\!\LiiYdot$ are the leading singular values and the corresponding left and right singular functions of $\DCE$, respectively.

% \daniel{There is no definition of $\tilde{M}$, I think it would be better to add a definition of it to become clearer, since we have a definition of $U$ and $V$ but not of $M$.}
As in \Cref{sec:background}, we characterize \eqref{eq:pco_svd} via a variational formulation.
Let $U \!=\! \left[u_1(X),\dots,u_d(X)\right]^\top$, $V \!=\! [v_1(\ddot{Y}),\dots,v_d(\ddot{Y})]^\top$, and $\mathsf{M} = \mathrm{diag}(\sigma_1,\dots,\sigma_d)$, and denote by $\mathsf{U} \!=\! \left[u_1\,\vert\cdots\vert u_{d}\right]\!:\! \R^d \!\to\! L^2(X)$ and $\mathsf{V} \!=\! \left[v_1\,\vert\cdots\vert v_{d}\right]\!:\!\R^d \!\to\! L^2(\ddot{Y})$ the associated feature operators. The truncated SVD of $\DCE$ can then be expressed as
% Following the same variational principle used to learn the covariance operator in \Cref{sec:background}, we express \Cref{eq:pco_svd} as
%
% \begin{eqnarray}\label{eq:pco_svd_variational}
% (\mathsf{U}\!,M\!,\mathsf{V})\!\in\!\argmin_{\tilde{\mathsf{U}}\!,\tilde{M}\!,\tilde{\mathsf{V}}}\, \bgbox{eqcolor1}{{\rm tr}\big[C_{\tilde{U}\tilde{U}} \tilde{M}C_{\tilde{V}\tilde{V}} \tilde{M}^\top\!\big]} \\\nonumber
% \bgbox{eqcolor2}{\!-\! 2 \trace\big[C_{\tilde{U}\tilde{V}}\tilde{M}^\top \big]}\\\nonumber
%  + \bgbox{eqcolor3}{2{\rm tr}\big[\tilde{M}^\top\tilde{\mathsf{U}}^*\mathsf{\Sigma}_{XZ}\mathsf{\Sigma}_{Z\ddot{Y}}\tilde{\mathsf{V}}\big],}
% \end{eqnarray}
\begin{equation}
\label{eq:pco_svd_variational}
\begin{aligned}
    (\mathsf{U}, \mathsf{M}, \mathsf{V}) \in \argmin_{\tilde{\mathsf{U}}, \tilde{\mathsf{M}}, \tilde{\mathsf{V}}} &\bgbox{eqcolor1}{\trace\!\left[C_{\tilde{U}\tilde{U}} \tilde{\mathsf{M}}C_{\tilde{V}\tilde{V}} \tilde{\mathsf{M}}^\top \right]} \\
    \bgbox{eqcolor2}{-2 &\trace\!\left[C_{\tilde{U}\tilde{V}}\tilde{\mathsf{M}}^\top \right]} \\
    \bgbox{eqcolor3}{+2 &\trace\!\big[ \tilde{\mathsf{M}}^\top\tilde{\mathsf{U}}^*\mathsf{\Sigma}_{XZ}\mathsf{\Sigma}_{Z\ddot{Y}}\tilde{\mathsf{V}}\big],}
\end{aligned}
\end{equation}
where the covariance matrices are centered.
% where the random vectors $U$ and $V$ are given by $U \!=\! \left[u_1(X), \dots, u_d(X) \right]^\top$ and $V \!=\! \left[v_1(\ddot{Y)}, \dots, v_d(\ddot{Y}) \right]^\top$. 
% In the following, $U$ and $V$ are the vector-valued random variables, while  
% $\mathsf{U} \!=\! \left[u_1\,\vert\cdots\vert u_d\right]\colon\! \mathbb{R}^d\!\to\!\LiiX$, $\mathsf{V}\! =\! \left[v_1\,\vert\cdots\vert v_d\right]\colon\! \mathbb{R}^d\!\to\!\LiiYdot$ are the corresponding operators. Furthermore we let ${M}={\rm diag}(\sigma_1,\ldots,\sigma_d)$. With this notation at hand, we follow the approach for the covariance operator outlined in Section \ref{sec:background} 
% and express \eqref{eq:pco_svd} as 
% However, since the composition of covariance operators cannot be estimated, a difficulty arises in computing the last term.
% Yet, in contrast to the unconditional case, where the corresponding variational principle immediately leads to a loss that admits an unbiased U-estimator, the last term involving the composition of covariance operators cannot be estimated
In contrast to the unconditional setting, the final term in \eqref{eq:pco_svd_variational} involves a composition of covariance operators, $\mathsf{\Sigma}_{XZ}\mathsf{\Sigma}_{Z\ddot{Y}}$, and cannot be directly estimated from data.
\emph{This constitutes the main technical challenge in learning $\DCE$.}
% We color-code the terms in the equations to clarify their evolution; full derivations appear in \Cref{sec:app_learning_pcov}.

% {\bf Bi-level contrastive formulation.}
{\bf Low-rank auxiliary problem.}
To overcome this, we use the cyclic property of the trace and observe that the operator $\mathsf{\Sigma}_{Z\ddot{Y}}\tilde{\mathsf{V}}\tilde{\mathsf{M}}^\top\tilde{\mathsf{U}}^*\mathsf{\Sigma}_{XZ}\colon\LiiZ\!\to\!\LiiZ$ is of rank at most $d$. Therefore, its symmetrization is of rank at most $2d$, and hence admits a singular value decomposition of the form $\mathsf{W} \mathsf{N} \mathsf{W}^*$, where $\mathsf{W} \!=\! \left[w_1\,\vert\cdots\vert w_{2d}\right]\colon \mathbb{R}^{2d}\!\to\!\LiiZ$ for an orthonormal system $w_i\in\LiiZ$ and $\mathsf{N}$ is a diagonal matrix with nonnegative entries. %\daniel{Maybe explain more the connection between this N and the N in the bilevel loss function. This N is diagonal non-negative while the one in the bilevel is not.}.
As the trace is linear and invariant under transposition, applying the same variational principle as in \Cref{sec:background} gives
\begin{equation}
\label{eq:trace_outer_via_inner_opt}
\bgbox{eqcolor3}{2 \trace\!\left[\tilde{\mathsf{M}}^\top\tilde{\mathsf{U}}^*\mathsf{\Sigma}_{XZ}\mathsf{\Sigma}_{Z\ddot{Y}}\tilde{\mathsf{V}}\right] = \trace\left[\mathsf{W}\mathsf{N}\mathsf{W}^*\right],}
\end{equation}
%
% Thus, using the same variational principle, we have that the last term becomes $2{\rm tr}\big[\tilde{M}^*\tilde{\mathsf{U}}^*\mathsf{\Sigma}_{XZ}\mathsf{\Sigma}_{Z\ddot{Y}}\tilde{\mathsf{V}}\big] = {\rm tr}[\mathsf{W}N\mathsf{W}^*]$, where
%
where $(\mathsf{W},\mathsf{N})$ solves the inner optimization problem
%
% \begin{eqnarray}\label{eq:pco_svd_variational_inner}
% & (\mathsf{W},N) \in \argmin_{\tilde{\mathsf{W}},\tilde{N}}\, \bgbox{eqcolor4}{{\rm tr} \big[ \tilde{N} C_{\tilde{W}\tilde{W}}\tilde{N}^*C_{\tilde{W}\tilde{W}}\big]} ~~~~~~~~~~~\\\nonumber
% & ~~~~~~~~~~~~~~~~~~~~~ +\bgbox{eqcolor5}{-{\rm tr}\big[(\tilde{N}^\top\!\!+\! \tilde{N})C_{\tilde{W}\tilde{U}}\tilde{M}C_{\tilde{V}\tilde{W}}\big],}
% \end{eqnarray}
\begin{equation}
\label{eq:pco_svd_variational_inner}
\begin{aligned}
    (\mathsf{W},\mathsf{N}) \in \argmin_{\tilde{\mathsf{W}},\tilde{\mathsf{N}}} &\bgbox{eqcolor4}{ \trace\!\left[ \tilde{\mathsf{N}} C_{\tilde{W}\tilde{W}}\tilde{\mathsf{N}}^\top C_{\tilde{W}\tilde{W}}\right]} \\
    \bgbox{eqcolor5}{- &\trace\!\left[(\tilde{\mathsf{N}}^\top\!\!+\! \tilde{\mathsf{N}})C_{\tilde{W}\tilde{U}}\tilde{\mathsf{M}}C_{\tilde{V}\tilde{W}}\right],}
\end{aligned}
\end{equation}
enabling us to compute the last term in \eqref{eq:pco_svd_variational} via \eqref{eq:trace_outer_via_inner_opt}.
Solving the bi-level optimization\footnote{For a refresher on bi-level optimization, see \citep{franceschi2025} and references therein.} problem obtained by combining Eqs. \eqref{eq:pco_svd_variational} and \eqref{eq:pco_svd_variational_inner} yields the representation
\begin{equation}\label{eq:trunc_svd}
    [\![\DCE]\!]_d= \mathsf{U} [C_{UV} - C_{UW}C_{WV}] \mathsf{V}^*.
\end{equation}
%
% where ${\mathsf U,V}$ and ${\mathsf W}$ solve the bi-level optimization problem \eqref{eq:pco_svd_variational}-\eqref{eq:pco_svd_variational_inner}. %\eqref{eq:pco_svd_variational}-\eqref{pco_svd_variational_inner}.
This expression highlights that the conditional dependence structure between $X$ and $Y$ given $Z$ is captured by the matrix $C_{UV} - C_{UW}C_{WV}$, provided that appropriate representations $(U,V,W) \!=\! (u(X), v(\ddot{Y}), w(Z))$ are used.
A derivation of the above bi-level formulation is given in \Cref{sec:app_learning_pcov}.

% {\bf Bi-level contrastive formulation.}
{\bf Spectral representation learning.}
Using the bi-level variational formulation of the truncated SVD of the partial cross-covariance operator in \eqref{eq:trunc_svd}, we extend the spectral contrastive algorithm of \citet{kostic2024neuralconditional} to learn the spectral features $u, v, w$.
We parametrize these features with neural networks $u_\theta\!:\! \spX \to \R^d, v_\theta\!:\! \Y\times\Z\to\R^d, w_\theta\!:\! \Z\to\R^{2d}$, and minimize empirical losses derived from Eqs. \eqref{eq:pco_svd_variational} and \eqref{eq:pco_svd_variational_inner} using U-statistics \citep{hoeffding1948b}.
The full-batch losses are shown below.

\begin{equation*}\label{eq:empirical_losses1}
\begin{aligned}
% \Lossout(\theta) &= \frac{1}{m(m-1)}\sum_{i\neq j}^m \inner{\overline{u}_\theta(X_i), M_\theta \overline{v}_\theta(\ddot{Y}_j)}^2 \\
% &\;\; - \frac{2}{m-1}\sum_{i=1}^m \inner{\overline{u}_\theta(X_i), M_\theta \overline{v}_\theta(\ddot{Y}_i)} \\
% &\;\; + \frac{2}{m(m-1)}\sum_{i\neq j}^m \inner{\overline{u}_\theta(X_i), \overline{v}_\theta(\ddot{Y}_j)}\inner{\overline{w}_\theta(Z_i), M_\theta \overline{w}_\theta(Z_j)}, \\
% \Lossin(\theta) &= \frac{1}{m(m-1)}\sum_{i\neq j}^m \inner{\overline{w}_\theta(Z_i), N_\theta \overline{w}_\theta(Z_j)}^2 \\
% &\;\; - \frac{2}{m(m-1)}\sum_{i\neq j}^m \inner{\overline{u}_\theta(X_i), N_\theta \overline{v}_\theta(\ddot{Y}_j)}\inner{\overline{w}_\theta(Z_i), M_\theta \overline{w}_\theta(Z_j)}.
\Lossout &= \bgbox{eqcolor1}{\frac{1}{m(m-1)}\sum_{i\neq j}^m \inner{\overline{u}_i, M \overline{v}_j}^2} \bgbox{eqcolor2}{- \frac{2}{m}\sum_{i=1}^m \inner{\overline{u}_i, M \overline{v}_i}} \\
&\;\; +\bgbox{eqcolor3}{ \frac{2}{m(m-1)}\sum_{i\neq j}^m \inner{\overline{u}_i, M\overline{v}_j}\inner{\overline{w}_i, \overline{w}_j}.}
% \Lossin(\theta) &= \frac{1}{m(m-1)}\sum_{i\neq j}^m \inner{w_i, N_\theta w_i}^2 \\
% &\;\; - \frac{2}{m(m-1)}\sum_{i\neq j}^m \inner{u_i, N_\theta v_j}\inner{w_i, M_\theta w_i}.
\end{aligned}
\end{equation*}
\begin{equation*}\label{eq:empirical_losses2}
\begin{aligned}
\Lossin &= \bgbox{eqcolor4}{\frac{1}{m(m-1)}\sum_{i\neq j}^m \inner{\overline{w}_i, N \overline{w}_j}^2} \\
&\;\; \bgbox{eqcolor5}{- \frac{2}{m(m-1)}\sum_{i\neq j}^m \inner{\overline{u}_i, M\overline{v}_j}\inner{\overline{w}_i,N \overline{w}_j},}
\end{aligned}
\end{equation*}
where we've written $\overline{u}_i = \overline{u}_\theta(X_i)$, $\overline{v}_j = \overline{v}_\theta(\ddot{Y}_j)$, and $\overline{w}_k = \overline{w}_\theta(Z_k)$, $M = M_\theta$, $N = (N_\theta + N_\theta^\top)/2$, where $\ddot{Y} = (Y,Z)$ and $\overline{\cdot}$ denotes centering.
The corresponding representation learning algorithm is shown in \Cref{alg:learning_pcov}.
% The derivation of the empirical losses and \Cref{alg:learning_pcov} is provided in \Cref{sec:app_learning_pcov}.
%
% The full derivation of the above losses as well as of \Cref{alg:learning_pcov} is given in \Cref{sec:app_learning_pcov}.
% For detailed derivation of the above contrastive loss for the bilevel optimization formulation of representation learning, see \cref{sec:app_learning_pcov}.

%Note $N$ above is required to be symmetric. This can be achieved by setting $N_\theta = (N_\theta + N_\theta^\top) / 2$.

{\bf Orthonormality regularization.} We also employ orthonormality (whitening) regularization with strength $\gamma > 0$, e.g., %the simplest example of which is
\begin{equation*}
\begin{aligned}
% \widehat{\Omega}_{\mathrm{out}(\theta)} &= \norm{\widehat{C}_{U_\theta U_\theta} - I_d}_F^2 + \norm{\widehat{C}_{V_\theta V_\theta} - I_d}_F^2 \quad\text{and}\quad 
% \Regin(\theta) &= \norm{\widehat{C}_{W_\theta W_\theta} - I_d}_F^2.
\Regout(\theta) &= \norm{\widehat{C}_{U_\theta U_\theta} - I_d}_{\mathrm{F}}^2 + \norm{\widehat{C}_{V_\theta V_\theta} - I_d}_{\mathrm{F}}^2, \\ 
\Regin(\theta) &= \norm{\widehat{C}_{W_\theta W_\theta} - I_{2d}}_{\mathrm{F}}^2.
\end{aligned}
\end{equation*}
We use orthonormality regularization to ensure that empirical covariance matrices remain well-conditioned and can be safely inverted during whitening.
% keep empirical covariance matrices well-conditioned, ensuring that they remain invertible and numerically stable during whitening.
% We apply orthonormality regularization to keep the condition numbers of the empirical covariance matrices controlled, ensuring that they remain invertible during the whitening post-processing step without introducing numerical instabilities.
% but one could also use unbiased U-statistics-based regularizers \citep{ordonezapraez2025equivariantrepresentationlearningsymmetryaware} or exponential regularizers as in \citep{kostic2024learninginvariant}.
 % We denote the neural networks coming from this stage as $\widetilde{u}_\theta,\widetilde{v}_\theta,\widetilde{w}_\theta$.
%

% \newcommand{\ALGCOMMENT}[1]{\STATE \textbf{\textcolor{teal}{\# #1}}}
\newcommand{\ALGCOMMENT}[1]{\STATE \textcolor{eqcolor4}{\textbf{\#} \textit{\textbf{#1}}}}
\begin{algorithm}[t]
\caption{Bi-level spectral representation learning}
\label{alg:learning_pcov}
\begin{algorithmic}
\STATE
\STATE \textbf{Input:} train set $\traindata$; regularization strength $\gamma > 0$

% \STATE
\ALGCOMMENT{Learn representations}
\FOR{for $t=1, \dots, n_{\texttt{steps}}$}
    \ALGCOMMENT{Step inner model}
    \FOR{$s=1, \dots, n_{\texttt{steps\_inner}}$}
        \STATE Sample mini-batch $\mathbf{B}$ from $\traindata$% $\gets (\mathbf{x},\mathbf{y},\mathbf{z})$
        \STATE $g_{\mathrm{in}} \gets \nabla_\theta [\widehat{\mathcal{L}}_{\mathrm{in}}(\theta, \mathbf{B}) + \gamma\,\Regin(\theta, \mathbf{B})]$
        \STATE Update $w_\theta$ with $g_{\mathrm{in}}$
    \ENDFOR
    % \STATE
    \STATE \ALGCOMMENT{Step outer model}
    \STATE Sample mini-batch $\mathbf{B}'$ from $\traindata$ %$\gets (\mathbf{x},\mathbf{y},\mathbf{z})$
    \STATE $g_{\mathrm{out}} \gets \nabla_\theta [\widehat{\mathcal{L}}_{\mathrm{out}}(\theta, \mathbf{B}') + \gamma\,\Regout(\theta, \mathbf{B}')]$
    \STATE Update $u_\theta, v_\theta$ with $g_{\mathrm{out}}$
\ENDFOR
% \ALGCOMMENT{Output: $\widetilde{u}_\theta,\widetilde{v}_\theta,\widetilde{w}_\theta$}
% \STATE
\STATE \ALGCOMMENT{Whiten representations}
\STATE $\unnestv, \vnnestv, \wnnestv  \gets \widehat{C}_{\widetilde{U}_\theta\widetilde{U}_\theta}^{-1/2} \widetilde{u}_\theta, \widehat{C}_{\widetilde{V}_\theta\widetilde{V}_\theta}^{-1/2} \widetilde{v}_\theta, \widehat{C}_{\widetilde{W}_\theta\widetilde{W}_\theta}^{-1/2} \widetilde{w}_\theta$
% \STATE
\STATE \textbf{Output:} estimated leading spectral features $\unnestv,\vnnestv,\wnnestv$
\end{algorithmic}
\end{algorithm}

% \begin{algorithm}[t]
% \caption{SCIT}
% \label{alg:scit}
% \begin{algorithmic}
% \STATE
% \STATE \textbf{Input:} $\{(x_i, y_i, z_i)\}_{i=1}^n$; regularization strength $\gamma > 0$

% \STATE
% \STATE \textbf{\textcolor{teal}{\texttt{\# Learning representations}}}
% \FOR{for $t=1, \dots, n_{\texttt{steps}}$}
%     \STATE \textbf{\textcolor{teal}{\texttt{\# Training inner model}}}
%     \FOR{$s=1, \dots, n_{\texttt{steps\_inner}}$}
%         \STATE Sample mini-batch $\mathbf{B} \gets (\mathbf{x},\mathbf{y},\mathbf{z})$
%         \STATE $g_{\mathrm{in}} \gets \nabla_\theta [\widehat{\mathcal{L}}_{\mathrm{in}}(\theta, \mathbf{B}) + \gamma\,\Regin(\theta, \mathbf{B})]$
%         \STATE Update $w_\theta$ with $g_{\mathrm{in}}$
%     \ENDFOR
%     % \STATE
%     \STATE \textbf{\textcolor{teal}{\texttt{\# Training outer model}}}
%     \STATE Sample mini-batch $\mathbf{B}' \gets (\mathbf{x},\mathbf{y},\mathbf{z})$
%     \STATE $g_{\mathrm{out}} \gets \nabla_\theta [\widehat{\mathcal{L}}_{\mathrm{out}}(\theta, \mathbf{B}') + \gamma\,\Regout(\theta, \mathbf{B}')]$
%     \STATE Update $u_\theta, v_\theta$ with $g_{\mathrm{out}}$
% \ENDFOR
% \STATE
% \STATE \textbf{\textcolor{teal}{\texttt{\# Computing statistic and p-value}}}
% \STATE Compute test statistic $\widehat{T} \gets n \, \norm{\widehat{C}_{\widehat{U}_\theta\widehat{V}_\theta} - \widehat{C}_{\widehat{U}_\theta\widehat{W}_\theta}\widehat{C}_{\widehat{W}_\theta\widehat{V}_\theta}}_{F}^2$
% \STATE \textbf{Output:} p-value obtained by locating $\widehat{T}$ on $\chi^2(d^2)$
% \end{algorithmic}
% \end{algorithm}

{\bf Whitening post-processing.}
Let $\widetilde{u}_\theta,\widetilde{v}_\theta,\widetilde{w}_\theta$ denote the features learned by \Cref{alg:learning_pcov}.
Since orthonormality is enforced only at the batch level and competes with other objectives, we apply an additional post-processing step to ensure that $\{\unnest{i}\}_{i=1}^d$, $\{\vnnest{j}\}_{j=1}^d$, and $\{\wnnest{k}\}_{k=1}^{2d}$ are empirically orthonormal.
Specifically, we whiten $\widetilde{u}_\theta(X)$, $\widetilde{v}_\theta(\ddot{Y})$, and $\widetilde{w}_\theta(Z)$ via $\widehat{u}_\theta(X) = \sqrtinvcovestUnnest\widetilde{u}_\theta(X)$ (and analogously for $\widehat{v}_\theta$ and $\widehat{w}_\theta$), where covariance matrices are estimated over the entire training set.
This transformation preserves the learned subspaces (e.g., $\range(\Unnest) \!=\! \range(\Unnwhite)$) while improving the geometry of the basis functions, yielding empirically orthonormal representations.

\section{CI Testing with Spectral Representations}\label{sec:ci_testing}

We begin this section by recasting nonparametric CI testing in terms of the partial covariance operator:
\begin{equation}
\label{eq:H0vsH1separation}
\mathcal{H}_0\!: \norm{\pcovs{X}{\ddot{Y}}}_{\mathrm{HS}}^2 \!=\! 0 \; \text{vs.} \; \mathcal{H}_{1,n,d}\!: \norm{\pcovs{X}{\ddot{Y}}}^2_{\mathrm{HS}} \!\geq\! \epsilon_n.
\end{equation}
Under $\mathcal{H}_{0}$, we consider all distributions $\mXYZ$ for which $\DCE$ is well-defined and vanishes.
Under $\mathcal{H}_{1,n,d}$, we consider local alternatives consisting of distributions $\mXYZ$ whose partial covariance operator is well-defined and has Hilbert-Schmidt norm at least $\epsilon_n$, where the separation threshold $\epsilon_n$ may decay with the test sample size $n$ at a rate that will be described in \Cref{thm:power}.
The dependence on $d$ comes from the requirement that $\sigma_{d}>\sigma_{d+1}\geq 0$ and $u_i(X)$, $v_i(\ddot{Y}), w_j(Z)$ be sub-Gaussian for $i \in [d], j \in [2d]$.

In light of the no-free-lunch theorem of \citet{shah_hardness_2020}, no test can distinguish $\mathcal{H}_0$ from $\mathcal{H}_{1,n,d}$ uniformly fast.
Rather than imposing restrictive structural conditions--such as Lipschitzness of $P_{X\mid Z=z}$ and $P_{Y\mid Z=z}$ w.r.t. $z$ as in \citep{neykov_minimax_2021,kim_local_2022}--we adopt a more practical perspective, framing test validity and power in terms of the quality of the learned representations $\widehat{u}_\theta,\widehat{v}_\theta,\widehat{w}_\theta$ obtained via \Cref{alg:learning_pcov}.
Formally, we measure this via:
\begin{equation}
\label{eq:gap_optim}
\begin{aligned}
&\vspace{-0.2cm}\mathcal{E}_{m}^{\mathrm{val}} = \max \Big\lbrace \norm{C_{\widehat{U}_\theta\widehat{U}_\theta} - I_d}, \norm{C_{\widehat{V}_\theta\widehat{V}_\theta} - I_d},\notag\\
&\hspace{5cm} \norm{C_{\widehat{W}_\theta\widehat{W}_\theta} - I_{2d}} \Big\rbrace,\\
& \mathcal{E}_{m}^{\mathrm{pow}} = \norm{\,[\![\mathsf{\Sigma}_{X\ddot{Y}\cdot Z}]\!]_d - \mathsf{U}_\theta \mathsf{M}_\theta \mathsf{V}_\theta^*}.
\end{aligned}
\end{equation}
%
%
% \begin{equation}
% \label{eq:gap_optim}
% \begin{aligned}
% \mathcal{E}_{m} = \max \Big\lbrace & \norm{\,[\![\mathsf{\Sigma}_{X\ddot{Y}\cdot Z}]\!]_d - \mathsf{U}_\theta \mathsf{M}_\theta \mathsf{V}_\theta^*}, \norm{C_{\widehat{U}_\theta\widehat{U}_\theta} - I_d} \\
% &  \norm{C_{\widehat{V}_\theta\widehat{V}_\theta} - I_d},  \norm{C_{\widehat{W}_\theta\widehat{W}_\theta} - I_{2d}} \Big\rbrace.
%   % \mathcal{E}_{\theta} := \max\{\norm{\mathrm{Cov}(u_{\theta}(x_1)) - I_d}, \norm{\mathrm{Cov}(v_{\theta}(\ddy_1)) - I_d}, \\
%   % \norm{\mathrm{Cov}(w_{\theta}(z)) - I_d}  \}.
% \end{aligned}
% \end{equation}
%
%
%
As shown in \Cref{sec:app_stat_proofs}, under the null the validity of our test is governed by the central limit theorem applied to the empirical partial covariance $\widehat{C}_{\widehat{U}_\theta\widehat{V}_\theta} - \widehat{C}_{\widehat{U}_\theta\widehat{W}_\theta}\widehat{C}_{\widehat{W}_\theta\widehat{V}_\theta}$ computed from the learned features. In particular, validity holds as long as $\mathcal{E}_{m}^{\mathrm{val}}$ is sufficiently small. Compared to regression-based tests \citep{shah_hardness_2020,he2025on}, our requirement is less restrictive, as it does not rely on estimating conditional expectations at a prescribed rate. This is consistent with the classical distinction between testing and estimation \citep{ingster1993asymptotically}: more precisely, rather than solving a full regression problem, our approach only requires capturing a low-dimensional spectral subspace of $\DCE$. Under the alternative, power is governed by $\mathcal{E}_{m}^{\mathrm{pow}}$, which measures how well the learned representations capture $[\![\DCE]\!]_d$. Thus, $\mathcal{E}_{m}^{\mathrm{val}}$ controls calibration under the null, while $\mathcal{E}_{m}^{\mathrm{pow}}$ controls signal retention under the alternative.

% As shown in \Cref{sec:app_stat_proofs}, under the null the validity of our test is controlled by the Central Limit Theorem rate for the latent partial covariance $C_{\widehat{U}_\theta\widehat{V}_\theta} - C_{\widehat{U}_\theta\widehat{W}_\theta}C_{\widehat{W}_\theta\widehat{V}_\theta}$ which depends only on $\mathcal{E}_{m}$ and the first few moments of the learned features, which can be controlled by using bounded activations such as \texttt{Tanh}.
% This requirement is simpler than the uniform conditional mean embedding rates needed in regression-based tests \citep{shah_hardness_2020,he2025on}.
% Under the alternative, power depends both on the singular value decay of $\DCE$ and the convergence rate $\mathcal{E}_{m}$, a challenge that is familiar to practitioners of representation learning.

We now introduce our proposed CI test, \textbf{SpectralCIT}.

% \VK{Namely, recalling that operator $\DCE$ reveals the structure of the joint distribution $P_{XYZ}$ relevant to conditional independence, natural conditions that define the null and alternative hypothesis are via the leading spectral components of $\DCE$  given in \eqref{eq:trunc_svd}. express. Namely, we model the collection of joint distributions that define the alternative hypotheses so that the corresponding operatro $\DCE$ is such that $\sigma_{d}>\sigma_{d+1}\geq 0$, for all $i\in[d]$ random variables $u_i(X)$, $v_i(\ddot{Y})$ and $w_i(Z)$ are sub-Gaussian. On the other hand, since under the  null hypotheses $\sigma_i=0$ for all $i\in\mathbb{N}$, we only need mild condition that there exist $U$, $V$ and $W$ satisfying \eqref{eq:trunc_svd}, such that random variables $u_i(X)$, $v_i(\ddot{Y})$ and $w_i(Z)$, $i\in[d]$, have bounded third moments.}

{\bf Spectral conditional independence test.} Let $N=m+n$ and let $\{(X_i, Y_i, Z_i)\}_{i=1}^{N}$ be an i.i.d. sample from $\mXYZ$, split into train and test sets,
\begin{equation*}
     \traindata \!=\! \{(X_i, Y_i, Z_i)\}_{i\!=\!1}^m, \testdata \!=\! \{(X_i, Y_i, Z_i)\}_{i\!=\!m\!+\!1}^{N}.
\end{equation*}
\textbf{SpectralCIT} works as follows.
First, \Cref{alg:learning_pcov} is applied to $\traindata$ to learn the feature maps $\unnestv$, $\vnnestv$, and $\wnnestv$.
Using these representations and the test set $\testdata$, we compute the test statistic 
\begin{equation}
\label{eq:test_statistic}
\begin{aligned}
\That &=n\, 
\norm{\covestUVnnest - \covestUWnnest\covestWVnnest}_F^2.\\%\quad \text{where}\quad \covestUVnnest = \frac{1}{n}\sum_{i=m+1}^{m+n} \unnestv(x_i) \otimes \vnnestv(\ddy_i),\notag\\
% &\covestUWnnest =\frac{1}{n}\sum_{i=m+1}^{m+n} \unnestv(x_i) \otimes \wnnestv(z_i) \\
% &\covestWVnnest = \frac{1}{n}\sum_{i=m+1}^{m+n} \wnnestv(z_i) \otimes \vnnestv(\ddy_i).
\end{aligned}
\end{equation}

{\bf Validity of testing with spectral features.}
The statistical guarantees are derived under the following mild regularity assumptions, which can be ensured by choosing bounded activation functions (e.g., \texttt{Tanh}).
% {\bf Validity of testing with spectral features.}

\begin{assumption}
\label{ass:subgaussian}
   Let $\norm{\widehat{u}_{\theta}(X)}$, $\norm{\widehat{v}_{\theta}(\ddot{Y})}$ and $\norm{\widehat{w}_{\theta}(Z)}$ be $K$--sub-Gaussian random variables.  
\end{assumption}

% \begin{theorem}[Validity]
% \label{thm:typeIerror}
%     Let Assumption \ref{ass:subgaussian} be satisfied. Assume in addition that $\mathcal{E}_m \rightarrow 0$ as $m\rightarrow \infty$. Then under the null hypothesis,
%     $
%     \That= n \, \norm{\covestUVnnest - \covestUWnnest \covestWVnnest}_{F}^2
%     $
%     converges in distribution to a chi-square distribution with $d^2$ degrees of freedom as $m,n\rightarrow \infty$.
% \end{theorem}

\begin{theorem}[Validity]
\label{thm:typeIerror}
    Let Assumption \ref{ass:subgaussian} be satisfied. Assume in addition that $\mathcal{E}_m^{\mathrm{val}} \rightarrow 0$ as $m\rightarrow \infty$. Then under the null hypothesis,
    $
    \That= n \, \norm{\covestUVnnest - \covestUWnnest \covestWVnnest}_{F}^2
    $
    converges in distribution to a chi-square distribution with $d^2$ degrees of freedom as $m,n\rightarrow \infty$.
\end{theorem}

\Cref{thm:typeIerror} shows that $\That \stackrel{d}{\to} \chi^2(d^2)$ as $m,n\to\infty$ as long as $\mathcal{E}_m^{\mathrm{val}} \to 0$ as $m\to\infty$. Intuitively, after whitening the learned features are approximately orthonormal, so $\widehat{C}_{\widehat{U}_\theta\widehat{V}_\theta} - \widehat{C}_{\widehat{U}_\theta\widehat{W}_\theta}\widehat{C}_{\widehat{W}_\theta\widehat{V}_\theta}$ has approximately identity covariance structure under the null, making $\That$ approximately a sum of $d^2$ squared independent standard Gaussians (see \Cref{app:prooftypeIerror}).
Therefore, for a significance level $\alpha \in (0,1)$, the method rejects the null hypothesis whenever $\That \geq c_\alpha$, where $c_\alpha = q_{1-\alpha}(\chi^2(d^2))$ denotes the $1-\alpha$ quantile of the $\chi^2(d^2)$ distribution.

% \Cref{thm:typeIerror} shows that $\That \stackrel{d}{\to} \chi^2(d^2)$ as $m,n\to\infty$ as long as the representation learning step is consistent ($\mathcal{E}_m \to 0$ as $m\to\infty$).
% % Thus, given a significance level $\alpha \in (0,1)$, we can define the rejection region of the test as $\{\That \geq c_{\alpha}\}$ where the critical value $c_{\alpha} = q_{1-\alpha}(\chi^2(d^2))$ is the $1-\alpha$ quantile of the chi-squared distribution with $d^2$ degrees of freedom.
% Therefore, for a significance level $\alpha \in (0,1)$, the method rejects the null hypothesis whenever $\That \geq c_\alpha$, where
% $c_\alpha = q_{1-\alpha}(\chi^2(d^2))$ denotes the $1-\alpha$ quantile of the $\chi^2(d^2)$ distribution.

{\bf Power of testing with spectral features.} 
\Cref{thm:power} characterizes the minimum signal strength $\epsilon_n$ required for our method to reliably detect conditional dependence under the alternative hypothesis.

\begin{theorem}[Power]
    \label{thm:power}
    Let Assumption \ref{ass:subgaussian} be satisfied.
    Let $\delta \in (0, 1)$ and take $d$ large enough such that $\sum_{j=1}^{d} \sigma_{j}^2 \geq  \norm{ \pcovs{X}{\ddot{Y}} }^{2}_{\mathrm{HS}} /2 $.
    Assume that $\epsilon_n^2 \geq c\left(d\,(\mathcal{E}_m^{\mathrm{pow}})^2 + \frac{d^2 + d\log(\delta^{-1})}{n}\right)$ for some large enough numerical constant $c>0$. Then for $\mathcal{H}_{1,n,d}$ defined in \Cref{eq:H0vsH1separation}, we have $
\mathbb{P}_{\mathcal{H}_{1,n,d}} \left( \That  > c_\alpha \right) \geq 1-\delta.
$
\end{theorem}

\section{Experiments}
\label{sec:experiments}

\begin{figure*}[t]
    \centering
    \begin{minipage}[b]{0.48\textwidth}
        \centering
        \includegraphics[width=\textwidth]{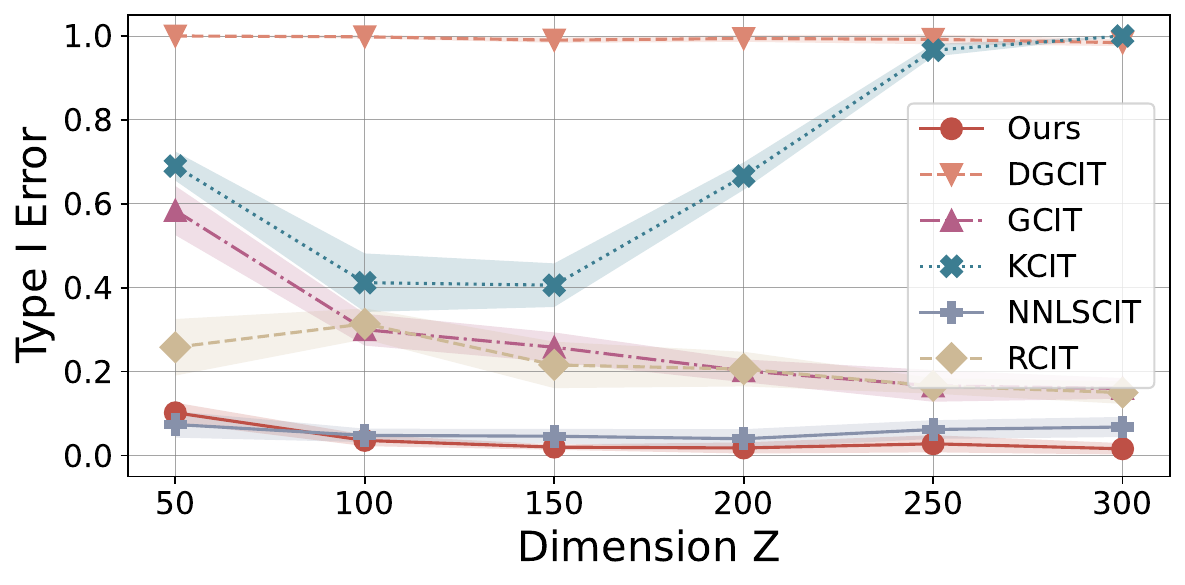}
        % \caption{Description of the left image.}
        % \label{fig:type_1_yang_benchmark}
    \end{minipage}
    \hfill
    \begin{minipage}[b]{0.48\textwidth}
        \centering
        \includegraphics[width=\textwidth]{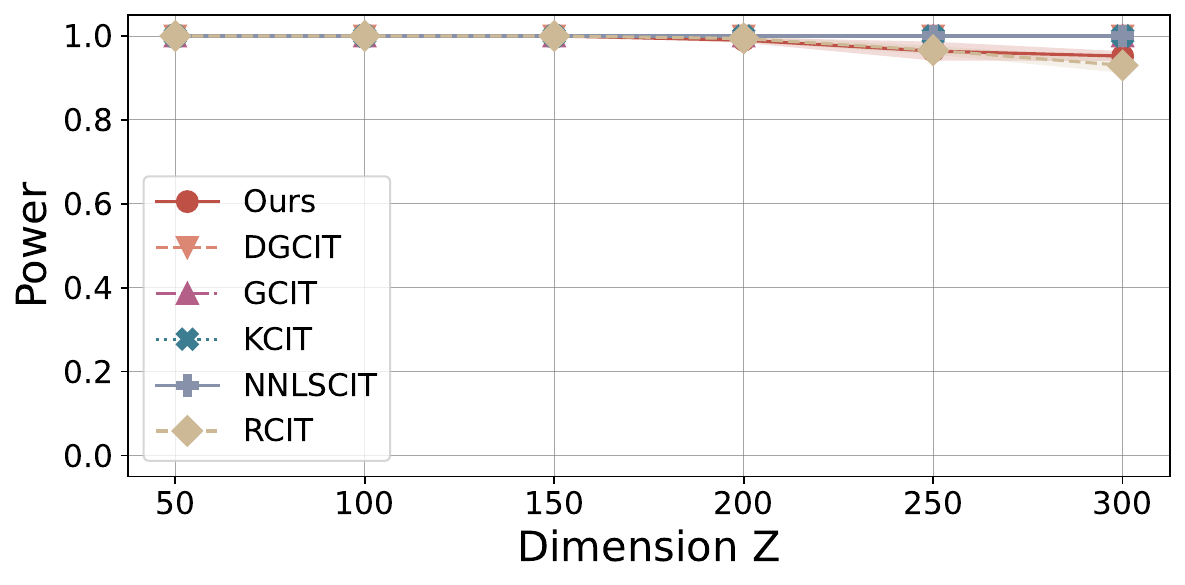}
        % \caption{Description of the right image.}
        % \label{fig:power_yang_benchmark}
    \end{minipage}
    \caption{Type I error and power of our method (\textbf{SpectralCIT}) compared to state-of-the-art conditional independence tests across varying dimensionality of the conditioning set $Z$.}
    \label{fig:yang_benchmark}
\end{figure*}

In this section, we report results from numerical experiments on challenging synthetic and real-world datasets. In \Cref{sec:synthetic_data}, we benchmark a range of existing CI tests against our proposed method under the post-nonlinear model \citep{yang2025,scetbon2022anasymptotic,zhang_kernel-based_2012,li2023k}. In \Cref{sec:bca_data}, we apply our approach to further investigate the role of visual (histological) features in breast tumor progression \citep{Li2024}.

% Ideas for next experiments:
% \begin{itemize}
%     % \item Other synthetic models for CI testing, as in \citep{berrett_conditional_2020, ai_testing_2024}.
%     \item Illustrate theory: conditiional independent images and chi2 distribution?
%     % \item Synthetic causal discovery with random DAGs, as in \citep{zhang_kernel-based_2012}.
%     % \item Real-world causal discovery in genetics datasets \citep{brouillard__2024}.
%     \item Feature selection in TCGA datasets from Daniel Tiezzi.
%     \item PNL such that kernel methods don't work, e.g., by considering mixed distributions with long tails (Cauchy) and spikes (e.g., sharp Gaussian mixtured). Also, we can use nonsmooth functions $F$ and $G$ such as step functions.
%   \item \citet{neykov_minimax_2021} provided experimental results; we could compete against it.
% \end{itemize}

\subsection{Synthetic Data}
\label{sec:synthetic_data}
%
% GCIT
% \begin{align*}
%     \mathcal{H}_0: X &= f(A_X Z + \E[\norm{A_X Z/d_X}_1] \varepsilon_X), \\
%     Y &= g(A_Y Z + \E[\norm{A_X Z/d_Y}_1]\varepsilon_Y), \\
%     \mathcal{H}_1: X &\sim \mathcal{N}(0, I_{d_X}), \\
%     Y &= g(A_Y Z + 2A_{XY} X),
% \end{align*}
%
% where the entries of $A_{(\cdot)}$ are generated uniformly at random from $[0,1)$ and then $A_{(\cdot)}$ is column-normalized with respect to the $\ell_1$ norm, the entries of $Z$ are generated according to $\text{Laplace}(0, 1)$, $\varepsilon_X \sim \mathcal{N}(0,I_{d_X})$, and $\varepsilon_Y \sim \mathcal{N}(0,I_{d_Y})$. We set the sample size to $N = 1500$ and vary the dimensionality of $Z$, $d_Z$, from $100$ to $500$. Each setting is repeated 100 times. We report the results in \Cref{fig:SCIT_power,fig:SCIT_type_1}, using a significance level of $\alpha = 0.05$.

{\bf Benchmark on post-nonlinear model.} %of \citet{bellot2019gcit}.}
% KERNEL BASELINES: SplitKCIT, KCIT, CIRCE
% BASELINES: KCI, GCIT, DGCIT, NNLSCIT, CCIT, CDCIT, Scetbon \\ GCM, HSCIC, FCI
% dZ: (10, 50, 100, 200, ..., 500)
We first evaluate the type I error and power of \textbf{SpectralCIT} against established baselines: \textbf{KCIT} \cite{zhang_kernel-based_2012}, \textbf{RCIT} \citep{Strobl2018}, \textbf{GCIT} \citep{bellot2019gcit}, \textbf{DGCIT} \citep{shi2021doublegenerative}, and \textbf{NNLSCIT} \citep{li2023k} using the post-nonlinear model of \citet{yang2025}:
% Yang2025
\begin{align*}
    \mathcal{H}_0: X &= f(\bar{Z} +    \varepsilon_X/4), \,
          Y = g(\bar{Z} +  \varepsilon_Y/4), \\
\mathcal{H}_1:  X &= f(\bar{Z} + \varepsilon_X/4) + \varepsilon/2, \,
         Y = g(\bar{Z} +  \varepsilon_Y/4) +  \varepsilon/2,
\end{align*}
% \begin{align*}
%     \mathcal{H}_0: X &= f(\bar{Z} +    \varepsilon_X/4), \\
%           Y &= g(\bar{Z} +  \varepsilon_Y/4), \\
% \mathcal{H}_1:  X &= f(\bar{Z} + \varepsilon_X/4) + \varepsilon/2, \\
%          Y &= g(\bar{Z} +  \varepsilon_Y/4) +  \varepsilon/2,
% \end{align*}
where $\bar{Z} = \frac{1}{d_Z}\sum_{i=1}^{d_Z} Z_i$, and $Z_i, \varepsilon_X, \varepsilon_Y, \varepsilon$ are i.i.d. standard Gaussian for $i \in [d_Z]$.
The nonlinearities are $f(w) = w^3$ and $g(w) = \tanh(w)$.
We fix the sample size $N = 1000$ and vary the conditioning dimension $d_Z \in [50, 300]$.
Each setting is repeated $100$ times.
Hyperparameter selection is described in \Cref{sec:experimental_details}.

Results in \Cref{fig:yang_benchmark} at significance level $\alpha \!=\! 0.05$ show that \textbf{KCIT}, \textbf{RCIT}, and \textbf{GCIT} achieved high power, but failed to control for type I errors.
\textbf{DGCIT} performed substantially worse, completely losing type I error control.
In contrast, \textbf{NNLSCIT} and \textbf{SpectralCIT} consistently maintained robust type I error control while achieving high power across all dimensions, consistent with the findings in \citet{li2023k}.
% \AF{Tune GCIT and DGCIT for type I error control and more variability on power and discuss results. If time allows, tune more RCIT and re-tune our method. If even more time allows, add Cauchy/Laplace noise and vary nonlinearities $f$ and $g$, but this probably would take a month running with a proper hyperparameter tuning protocol.}

{\bf On the role of structural assumptions.}
A key advantage of our approach is its ability to adapt to diverse structural scenarios.
To illustrate this, we evaluate the type I error of our method (\textbf{SpectralCIT}) against \textbf{NNLSCIT} using a nonsmooth, high-dimensional data model defined as:% follows:
%
% \begin{align*}
%     \mathcal{H}_0: X &= f(Z + \varepsilon_X), \,
%           Y = g(Z + \varepsilon_Y), \\
%     \mathcal{H}_1: X &= f(Z + \varepsilon_X) + 4\varepsilon, \,
%           Y = g(Z + \varepsilon_Y) + 4\varepsilon,
% \end{align*}
\begin{equation*}
    X = f(Z/2 + \varepsilon_X), \quad Y = g(Z/2 + \varepsilon_Y),
\end{equation*}
where $Z$, $\varepsilon_X$, and $\varepsilon_Y$ are sampled independently from $N(0, I_{100})$.
We define the nonlinearities $f(w) = h_2(w)$ and $g(w)=h_3(w)$ to be highly oscillatory near the origin:
\begin{align*}
h_k(w) &= 
\begin{cases} 
w^k & \text{for } |w| \ge 1, \\
\cos\left(\frac{2\pi}{w}\right) & \text{for } 0 < |w| < 1, \\
1 & \text{for } w = 0.
\end{cases} 
\end{align*}
%\\
% g(w) &= 
% \begin{cases} 
% w^3 & \text{for } |w| \ge 1, \\
% \cos\left(\frac{2\pi}{w}\right) & \text{for } 0 < |w| < 1, \\
% 1 & \text{for } w = 0.
% \end{cases}
% \end{align*}
%
We set the dimensionality of $Z$ to $100$ and vary the sample size from $500$ to $1500$.
Each setting is repeated $100$ times.
Hyperparameters are the same as in the previous experiment.

\begin{figure}[t]
    \centering
    \includegraphics[width=\columnwidth]{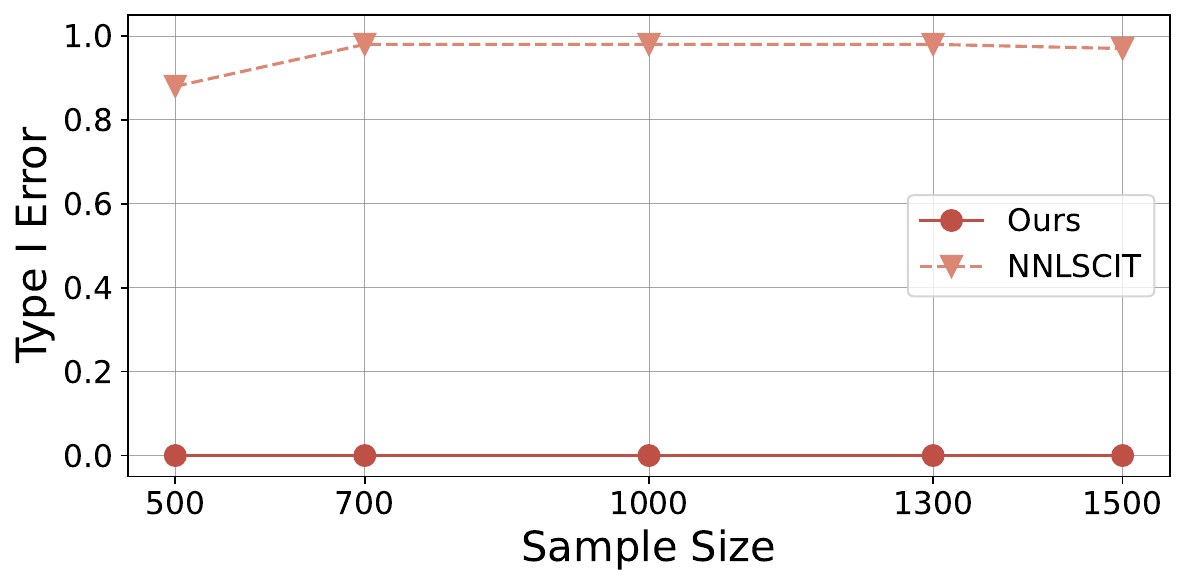}
    \caption{Type 1 error of our method (\textbf{SpectralCIT}) compared to \textbf{NNLSCIT} across varying sample sizes.}
    \label{fig:type_1_nonsmooth_nnlscit}
\end{figure}

Results are reported in \Cref{fig:type_1_nonsmooth_nnlscit} with a significance level of $\alpha = 0.05$. Notably, the highly oscillatory behavior of $f$ and $g$ near $w=0$ creates a lack of smoothness on the conditional distributions that violates the core assumptions of \textbf{NNLSCIT} \citep{kim_local_2022}, leading to a complete breakdown of its type I error control. In contrast, our operator-based approach remains robust.
%
% \AF{As discussed previously, existing CI tests rely on restrictive structural assumptions. Picking NNSLCIT as an example as it performed well in the previous experiment, we note that it relies on smoothness assumptions of the dcondtional distributrion $p(y|z)$ and on the existence of clustered points in $Z$. In the following, we generate data violating these assumptions
% %
% \begin{align*}
%     \mathcal{H}_0: X &= \bar{Z} + \varepsilon_X, \\
%           Y &= \bar{Z} + \varepsilon_Y, \\
%     \mathcal{H}_1: X &= \bar{Z} + \varepsilon_X, \\
%           Y &= X + \varepsilon_Y,
% \end{align*}
% %
% where TODO.}

{\bf Overcoming limitations of kernel methods.}
Finally, we evaluate the scalability and adaptivity of \textbf{SpectralCIT} against established kernel-based CI tests, including \textbf{KCIT} \cite{zhang_kernel-based_2012}, \textbf{RCIT} \citep{Strobl2018}, \textbf{LPCIT} \citep{scetbon2022anasymptotic}, and \textbf{GCM} \citep{shah_hardness_2020}.
Data is generated according to \citep{shi2021doublegenerative}:
\begin{align*}
X &= \sin(a_X^\top Z + \varepsilon_X/2),\,
Y = \cos(a_Y^\top Z + bX + \varepsilon_Y/2).
\end{align*}
Here, $Z$, $\varepsilon_X$, and $\varepsilon_Y$ are i.i.d. standard Gaussian and the entries of $a_{(\cdot)}$ are sampled uniformly from $[0, 1)$ and then normalized to unit $\ell_1$ norm.
The parameter $b$ controls the degree of conditional dependence: $b=0$ corresponds to the null hypothesis $\mathcal{H}_0$, while $b \neq 0$ implies $\mathcal{H}_1$.
We set the sample size to $N = 1000$ and vary the dimensionality of $Z$, $d_Z$, from $50$ to $300$.
Each setting is repeated $100$ times. Hyperparameters are selected following \Cref{sec:experimental_details}.

Results at significance level $\alpha = 0.05$ and $b \in \{0, 2\}$ are shown in \Cref{fig:dgcit_kernels}. %\Cref{fig:type_1_dgcit_kernels,fig:power_dgcit_kernels}.
Consistent with \citet{Ramdas2015}, we observe a general loss of power for kernel-based CI tests. 
While \textbf{KCIT} achieved higher power than \textbf{SpectralCIT} for $d_Z \in \{250, 300\}$, it failed to maintain type I error control, in line with \citep{shi2021doublegenerative}.
In contrast, \textbf{GCM} preserved type I error validity but exhibited very low power, again consistent with \citep{shi2021doublegenerative}.
The method of \citet{scetbon2022anasymptotic} was excluded due to excessive runtime.\footnote{Cf. \citep[Fig. 7]{yang2025}.}
%the 30-minute computational budget per run.\footnote{Cf. \citep[Fig. 7]{yang2025}.}

\begin{figure}[t]
    \centering
    \includegraphics[width=\columnwidth]{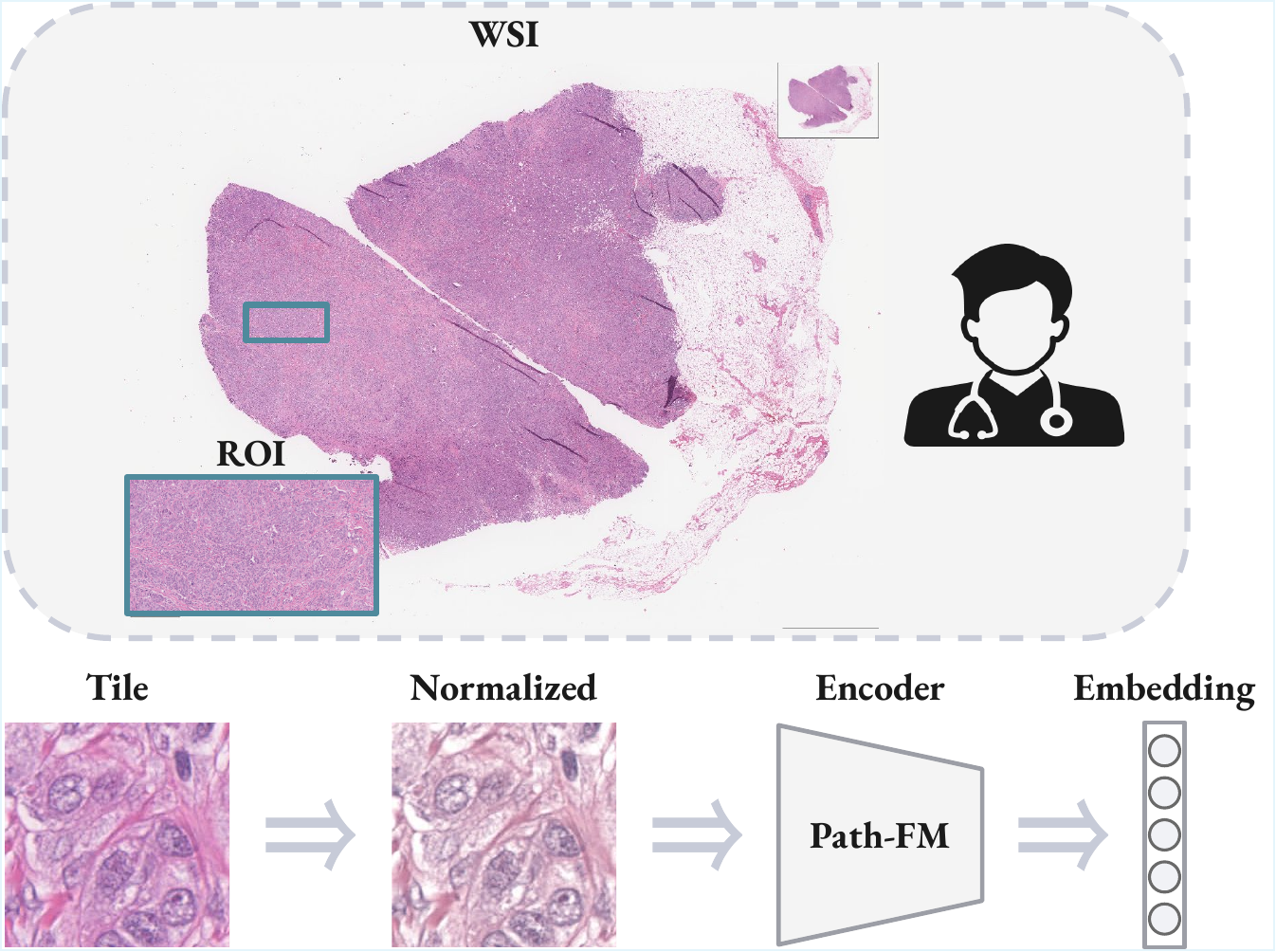}
    \caption{Data processing pipeline for breast cancer experiment.}
    \label{fig:wsi_tcga}
    \vspace{-2mm}
\end{figure}

\begin{figure*}[t]
    \centering
    \begin{minipage}[b]{0.48\textwidth}
        \centering
        \includegraphics[width=\textwidth]{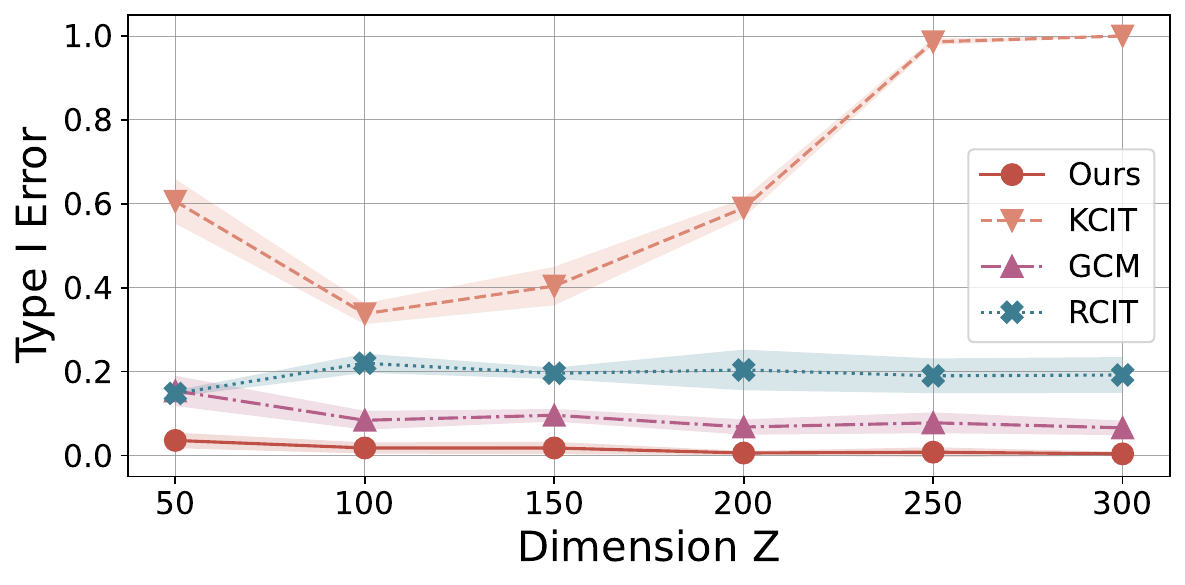}
        % \caption{Description of the left image.}
        % \label{fig:type_1_yang_benchmark}
    \end{minipage}
    \hfill
    \begin{minipage}[b]{0.48\textwidth}
        \centering
        \includegraphics[width=\textwidth]{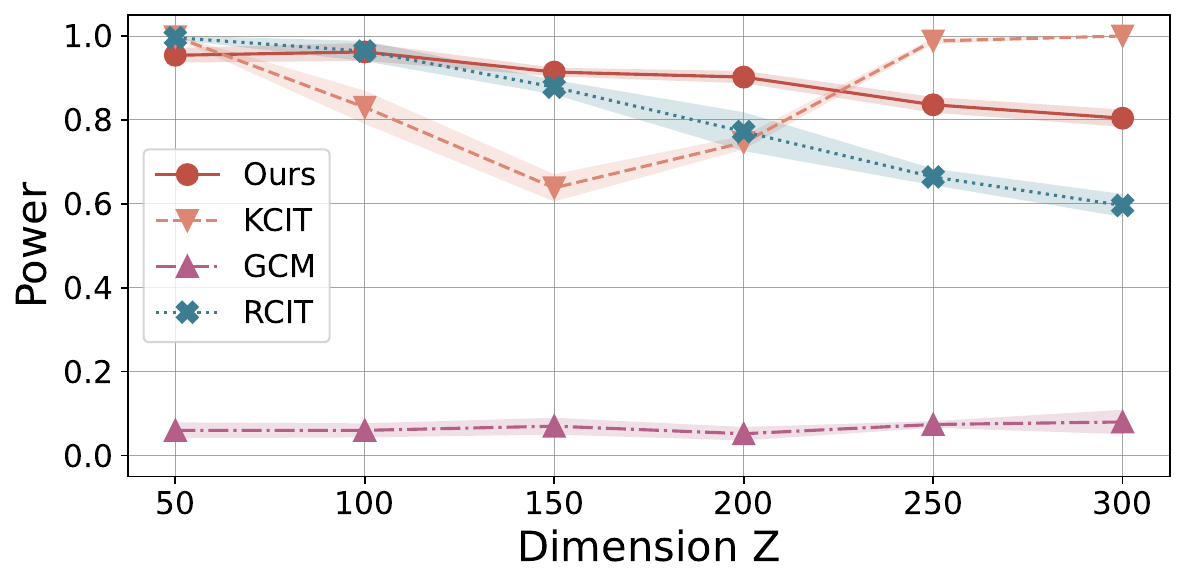}
        % \caption{Description of the right image.}
        % \label{fig:power_yang_benchmark}
    \end{minipage}
    \caption{Type I error and power of our method (\textbf{SpectralCIT}) compared to state-of-the-art kernel-based conditional independence tests across varying dimensionality of the conditioning set $Z$.} %\textcolor{blue}{Dashed red line denotes $\alpha=0.05$.}}
    \label{fig:dgcit_kernels}
\end{figure*}
% \begin{figure}
% \centering
% \includegraphics[width=\columnwidth]{plots/experiments/type_1_dgcit_kernels.pdf}
% % \refstepcounter{figure}
% \caption{Type I error of SpectralCIT (Ours) compared to state-of-the-art kernel-based conditional independence tests across the varying dimensionality of the conditioning set $Z$.}
% \label{fig:type_1_dgcit_kernels}
% \end{figure}
% \begin{figure}
% \centering
% \includegraphics[width=\columnwidth]{plots/experiments/power_dgcit_kernels.pdf}
% % \refstepcounter{figure}
% \caption{Power of SpectralCIT and competing kernel-based CI tests as a function of the dimensionality of the conditioning set $Z$.}
% \label{fig:power_dgcit_kernels}
% \end{figure}

\subsection{Breast Cancer Data}
\label{sec:bca_data}
% Predicting outcomes in cancer patients is crucial for clinical decision-making and precision medicine.
% Gene expression profiling can identify high-risk patients \citep{10.1093/jnci/djj329, doi:10.1056/NEJMoa1804710}, but it remains costly and time-consuming, especially in underserved communities \citep{10.1158/1538-7445.SABCS18-P5-15-04}.
% Histological images, routinely examined by pathologists to inform prognosis and treatment, inherently reflect underlying molecular phenotypes \cite{CRICK1970,PMID:35465400} and may already capture much of the genome's prognostic information.
% This raises a key question: how much additional value do molecular profiles provide beyond image-based representations?
% Recent advances in computational pathology suggest that machine learning can extract prognostic signals beyond human perception \citep{Farahmand2022,WANG2023112953}.
% Motivated by these advances, we combine foundation models for pathology with our method to formally assess whether molecular profiles provide prognostic information beyond that captured by histological images.
Predicting outcomes in cancer patients is crucial for clinical decision-making and precision medicine.
Gene expression profiling can identify high-risk patients \citep{10.1093/jnci/djj329, doi:10.1056/NEJMoa1804710}, but it remains costly and time-consuming, especially in underserved communities \citep{10.1158/1538-7445.SABCS18-P5-15-04}.
Histological images, routinely examined by pathologists to inform prognosis and treatment, inherently reflect underlying molecular phenotypes \cite{PMID:35465400}.
Recent advances in computational pathology suggest that machine learning can extract prognostic signals beyond human perception \citep{Farahmand2022,WANG2023112953}, raising the possibility that sufficiently rich image representations may capture much of the prognostic information encoded in molecular profiles. This leads to a key question: how much additional value do molecular profiles provide beyond image-based representations? Clinically, genomic assays such as MammaPrint \citep{10.1093/jnci/djj329} and  OncotypeDx \citep{doi:10.1056/NEJMoa1804710} provide prognostic information beyond standard pathological assessment, but it remains unclear whether modern foundation-model representations of histology already capture most of this signal. Motivated by this question, we combine pathology foundation models with our conditional independence testing framework to formally assess whether molecular profiles provide prognostic information beyond that captured by histological images.

{\bf Dataset.}
We selected all invasive carcinoma whole slide images (WSIs) from The Cancer Genome Atlas Breast Invasive Carcinoma (TCGA-BRCA) dataset \citep{TCGA,Ciriello2015} with available survival data indicating either death within three years of diagnosis ($Y=1$) or survival beyond five years of follow-up ($Y=0$), resulting in 135 WSIs.
For each WSI, a representative malignant tumor region (target area) was annotated by a medical doctor using QuPath \citep{qupath}.
Target areas were then divided into $244\times244$ tiles, H\&E stain-normalized following \citep{macenko}, and processed using the Path Foundation model \citep{lai2023domainspecificoptimizationdiverseevaluation} to extract latent representations encoding local nuclear morphology and micro-environmental context.
Up to 10 target-area patches per WSI were selected based on Euclidean distance in the latent space, yielding a dataset of $N=1341$ triplets $(X,Y,Z)$, where $X \in \R^3$ represents Her2, Luminal, and Basal metagene scores from \citep{Tiezzi2025}, $Y \in \{0, 1\}$ is the survival outcome, and $Z \in \R^{384}$ represents the high-dimensional image features.
The data processing pipeline is illustrated in \Cref{fig:wsi_tcga}.

{\bf Conditional independence testing.}
% We first examined linear associations via partial correlations between each coordinate $X$ and $Y$ given $Z$, with results indicating negligible linear dependence ($-0.006$, $-0.088$, $0.07$).
% In line with this observation, a logistic regression model incorporating both $X$ and $Z$ achieved slightly lower predictive accuracy than a model using $Z$ alone ($0.87$ vs. $0.86$).
% We then applied \textbf{KCIT}, \textbf{NNLSCIT}, and \textbf{SpectralCIT} to assess conditional independence beyond linear relationships.
% \textbf{KCIT} and \textbf{NNLSCIT} failed to reject the null hypothesis at significance level $\alpha\!=\!0.05$, whereas \textbf{SpectralCIT} strongly rejected it (see \Cref{tab:pvals_cit}).
% Training an XGBoost model \citep{Chen2016} confirmed residual predictive information in $X$: accuracy rose from $0.91$ with $Z$ alone to $0.95$ with $X$ and $Z$.
% This reveals non-linear associations between metagene scores and survival, consistent with recent prognostic models that incorporate multi-modal data \citep{Tolaney2024}, and highlights \textbf{SpectralCIT}’s ability to detect complex dependencies in high-dimensional biological data.
We first examined linear associations via partial correlations between each coordinate $X$ and $Y$ given $Z$, with results indicating negligible linear dependence ($-0.006$, $-0.088$, $0.07$).
In line with this observation, a logistic regression model incorporating both $X$ and $Z$ achieved slightly lower predictive accuracy than a model using $Z$ alone ($0.86$ vs. $0.87$).
We then applied \textbf{KCIT}, \textbf{NNLSCIT}, and \textbf{SpectralCIT} to assess conditional independence beyond linear relationships.
\textbf{KCIT} and \textbf{NNLSCIT} failed to reject the null hypothesis at significance level $\alpha\!=\!0.05$, whereas \textbf{SpectralCIT} strongly rejected it (see \Cref{tab:pvals_cit}).
Training an XGBoost model \citep{Chen2016} confirmed residual predictive information in $X$: accuracy rose from $0.91$ with $Z$ alone to $0.95$ with $X$ and $Z$.
This indicates non-linear associations between metagene scores and survival not captured by latent representations, consistent with recent prognostic models that incorporate multi-modal data \citep{Tolaney2024}, and highlights SpectralCIT’s ability to detect complex dependencies in high-dimensional biological data.

\begin{table}[ht]
\centering
\caption{P-values returned by the conditional independence tests.}
\label{tab:pvals_cit}
\begin{tabular}{@{}ccc@{}}
\toprule
 \textbf{SpectralCIT} & \textbf{KCIT} & \textbf{NNLSCIT} \\
\midrule
$<10^{-3}$ & $6.8\times10^{-2}$ & $3.8 \times 10^{-1}$ \\
\bottomrule
\end{tabular}
\end{table}
% & \textbf{RCIT} & $<10^{-3}$

% \begin{table}[ht]
% \centering
% \caption{P-values of different conditional independence tests}
% \label{tab:pvals_cit}
% \begin{tabular}{lcc}
% \toprule
% \textbf{SpectralCIT} & \textbf{KCIT} & \textbf{NNLSCIT} \\
% \midrule
% 0.000 & 0.068 & 0.83 \\
% \bottomrule
% \end{tabular}
% \end{table}

% \begin{table}[ht]
% \centering
% \caption{Predictive Accuracy Comparison}
% \label{tab:regressions_tcga}
% \begin{tabular}{lcc}
% \toprule
% Model & $Z$ & $(X, Z)$ \\
% \midrule
% Logistic Regression & 0.84 & 0.83 \\
% XGBoost & 0.88 & 0.97 \\
% \bottomrule
% \end{tabular}
% \end{table}

\section{Related Work}
\label{sec:related_work}

We provide a brief overview of existing CI tests and spectral representation learning, and refer the reader to \citep{li_nonparametric_2020} for a comprehensive discussion.
We note that the CI testing literature is highly method-centric, with a lack of standardized benchmarks and well-established hyperparameter tuning protocols \citep{positioncausalml}.

{\bf Spectral representation learning.}
Expansions in orthogonal bases such as Fourier, wavelets, and splines have a long history in nonparametric statistics \cite{Efromovich1999}, but perform poorly in high dimensions and lack adaptivity to the data distribution.
This limitation motivated the use of bases derived from the spectral decomposition of linear integral operators defined by symmetric positive semi-definite kernels \citep{izbickithesis}.
More recently, research has shifted toward learned spectral bases, or features, which have been applied across diverse settings including dynamical systems \cite{turri2025self}, reinforcement learning \cite{hu2024primaldualspectralrepresentationoffpolicy}, causal inference \cite{wang2022spectralrepresentationlearningconditional,sun25speciv,meunier2025outcomeawarespectralfeaturelearning,meunier2025demystifyingspectralfeaturelearning}, and geometric deep learning \citep{ordonezapraez2025equivariantrepresentationlearningsymmetryaware}.
% Many of these methods resemble contrastive approaches in self-supervised learning \cite{haochen_2021_provable}. In cases where the objective is to learn linear integral operators, the problem often reduces to estimating their associated kernel \cite{kostic2024neuralconditional,ordonezapraez2025equivariantrepresentationlearningsymmetryaware}
%; for the conditional expectation operator, this kernel can be expressed as a density ratio, for which a variety of estimation strategies have been developed \cite{Sugiyama2012}.

{\bf Kernel tests of CI.}
Kernel-based measures of conditional dependence were first introduced in \citep{fukumizu2007kernelmeasures}, building on the partial cross-covariance operator from \citep{fukumizu2004dimensionalityreduction}.
The first kernel-based CI test, \textbf{KCIT} \citep{zhang_kernel-based_2012}, uses the Hilbert–Schmidt norm of this operator and provides a characterization of the asymptotic null distribution, but without power guarantees.
To improve scalability, \citet{Strobl2018} introduced \textbf{RCIT} and \textbf{RCoT}, which approximate \textbf{KCIT} using random Fourier features.
More recently, kernel partial correlation \citep{huang_kernel_2022}, generalized covariance measure (\textbf{GCM}) \citep{shah_hardness_2020}, and tests based on analytical kernel mean embeddings (\textbf{LPCIT}) \citep{scetbon2022anasymptotic} have been developed.
A causal representation learning method leveraging a kernel-based statistic has been proposed in \citep{pogodin2023efficient}.
Practical strategies for kernel selection have been discussed \citep{wang2025practical}, though at notable computational cost.
Despite their simplicity and effectiveness, kernel-based tests rely on kernel mean embeddings \citep{Muandet_2017}, resulting in slow convergence rates \citep{pogodin2025practicalkerneltestsconditional} and loss of power as the dimensionality increases \citep{Ramdas2015}.

{\bf Model-x, Local permutation, and other tests of CI.}
Model-X tests rely on knowledge of the conditional distribution $P(X | Z)$ and were first introduced in \cite{Candes2018}, showing that valid p-values are obtained when $P(X | Z)$ is known exactly by sampling from it.
\citet{berrett_conditional_2020} extended this framework by designing a permutation test that requires only approximate knowledge of the conditional distribution.
Motivated by advances in generative modeling with neural networks, \citet{bellot2019gcit} proposed learning $P(X | Z)$ and performing a permutation test (\textbf{GCIT}), with theoretical guarantees based on the GAN loss.
However, \citet{zhang_doubly_2024} showed that these bounds can be loose even in simple linear settings and proposed a doubly robust alternative, while \citet{shi2021doublegenerative} offered another double generative approach (\textbf{DGCIT}).
More recently, \citet{yang2025} introduced a diffusion-based variant.
Local permutation tests \cite{runge_conditional_2018,kim_local_2022,bellot2019gcit,fukumizu2007kernelmeasures,sen2017model,li2023k} exploit the factorization $P_{X|Y=y,Z=z} = P_{X|Z=z}$ by permuting $X$ within binned or clustered $Z$ values, producing approximately valid samples from $P_{XYZ}$ under the null under smoothness assumptions and sufficiently populated clusters \citep{kim_local_2022}.

\section{Conclusions}

% We have developed a statistically rigorous and computationally efficient framework for nonparametric conditional independence testing, grounded in operator learning and representation-based estimation of the partial cross-covariance operator. Our method delivers strong performance in both synthetic and real-world settings, with accurate Type I error control and competitive power, even in challenging high-dimensional or continuous-conditioning scenarios. Looking ahead, we aim to extend this framework to causal discovery tasks, where conditional independence forms the backbone of many structure learning algorithms. In particular, adapting operator-based CI tests to handle interventions, feedback loops, or time-varying systems poses a rich set of theoretical and practical challenges. Moreover, scaling to complex structured data—such as graphs, sequences, or distributions—will require novel representation strategies that preserve conditional relationships in non-Euclidean domains.

% \AF{We revisited conditional independence testing via partial covariance operators and learned spectral features, addressing scalability and adaptivity limits of kernel-based tests. Our representation learning approach allows efficient GPU optimization while maintaining theoretical guarantees: the statistic converges to chi-squared under the null and is consistent under the alternative. Experiments on synthetic and real data demonstrate the practicality of our framework, pointing to a scalable CI testing approach that unites kernel theory with modern representation learning.}

In this paper, we revisited conditional independence testing with partial covariance operators through the lens of learned spectral features, addressing the scalability and adaptivity limitations of kernel-based tests.
We proposed a scalable bi-level contrastive algorithm to learn leading spectral features of partial covariance operators, enabling simple linear test statistics while maintaining validity and power.
Under the null, we showed that validity depends on central limit theorem rates for the test statistic's matrix, which can be easily controlled using bounded activations such as \texttt{Tanh}.
Under alternative, we established that power relies on the quality of the learned representations--an equivalent but more machine learning-aligned perspective.
Finally, we evaluated the proposed test on high-dimensional synthetic and real-world datasets, including breast cancer data from \citep{TCGA}, showing that our approach offers a practical and statistically grounded path toward scalable CI testing, bridging kernel-based theory with modern representation learning.

{\bf Limitations.} %
Our work shares the following limitations with the CI testing and spectral representation learning literatures. First, there is no standard protocol for hyperparameter selection in conditional independence testing. As described in \Cref{sec:experimental_details}, we therefore either use established values or apply a common tuning protocol with equal computational budgets across methods, varying only hyperparameters exposed in the original implementations.
Second, unlike the unconditional setting where permutation methods provide exact calibration \cite{Berrett2019nonparametric}, existing tests require additional assumptions: model-X methods depend on knowledge of $P_{X \mid Z}$, local permutation methods rely on smoothness, and regression-based approaches require sufficiently fast regression estimation. Our method is no exception. Although our assumptions for validity are somewhat weaker, requiring only moment control of learned features, we still observe conservative calibration in practice, which can reduce power.
Finally, our guarantees depend on the quality of the learned representations through the $\mathcal{E}_m$ terms in \Cref{sec:ci_testing}. We make this dependence explicit rather than tying the analysis to a particular architecture or optimization scheme. Obtaining explicit non-asymptotic rates for learned spectral representations remains an important open problem beyond our specific method \citep{meunier2025demystifyingspectralfeaturelearning,kostic2024neuralconditional}.

% {\bf LLM usage.}
% LLMs have been used to improve the clarity of the main text and parts of the appendix.
% GitHub Copilot has been used within VS Code for code completion.
% All generated content has been carefully reviewed and verified by the authors.

\section*{Acknowledgments}
This work was supported by the EU Project ELIAS (grant No. 101120237), and by the European Union – NextGenerationEU and the Italian National Recovery and Resilience Plan through the Ministry of University and Research (MUR), under Project PE0000013 CUP J53C22003010006.
KL acknowledges support from the French National Research Agency for the DECATTLON project (ANR-24-CE40-3341).

\section*{Impact Statement}
This paper presents work whose goal is to advance the field of machine learning. There are many potential societal consequences of our work, none of which we feel must be specifically highlighted here.

% {\bf Acknowledgments.} This work was supported by the EU Project ELIAS (grant No. 101120237), and by the European Union – NextGenerationEU and the Italian National Recovery and Resilience Plan through the Ministry of University and Research (MUR), under Project PE0000013 CUP J53C22003010006. KL acknowledges support from the French National Research Agency for the DECATTLON project (ANR-24-CE40-3341). We thank the anonymous reviewers for their insightful and valuable feedback.

\bibliography{references}
\bibliographystyle{icml2026}

%%%%%%%%%%%%%%%%%%%%%%%%%%%%%%%%%%%%%%%%%%%%%%%%%%%%%%%%%%%%%%%%%%%%%%%%%%%%%%%
%%%%%%%%%%%%%%%%%%%%%%%%%%%%%%%%%%%%%%%%%%%%%%%%%%%%%%%%%%%%%%%%%%%%%%%%%%%%%%%
% APPENDIX
%%%%%%%%%%%%%%%%%%%%%%%%%%%%%%%%%%%%%%%%%%%%%%%%%%%%%%%%%%%%%%%%%%%%%%%%%%%%%%%
%%%%%%%%%%%%%%%%%%%%%%%%%%%%%%%%%%%%%%%%%%%%%%%%%%%%%%%%%%%%%%%%%%%%%%%%%%%%%%%
\newpage
\appendix
\onecolumn

\section{Derivation of Losses and Algorithm for Learning the Partial Covariance Operator}
\label{sec:app_learning_pcov}

In this section, we present the full derivation of the losses and the bi-level contrastive algorithm presented in \Cref{sec:representation_learning}.

{\bf From Eckart-Young-Mirsky to a loss.} Let $d \in \N$. Our goal is to learn $\SVDd{\DCE}$ from an i.i.d. sample from $\mXYZ$ using representation learning. To do so, we require a loss whose minimizer is $\SVDd{\DCE}$. We base our loss on the following reformulation of Eckart-Young-Mirsky theorem from \cite{kostic2024neuralconditional}.

\begin{theorem}[Eckart-Young-Mirsky: loss form]\label{thm:eckart_young_mirsky_loss}
Let $\mathsf{A}:\mathcal{G}\to\mathcal{F}$ be a compact operator and $\gamma > 0$, then the solution of 
\begin{align*}
%\mathcal{L}_\gamma(\mathsf{U}, \mathsf{M}, \mathsf{V}) = 
    \min_{\U,\M, \V}\quad &\mathcal{L}(\mathsf{U}, \mathsf{M}, \mathsf{V})  + \gamma\,\Omega(\mathsf{U}, \mathsf{V}), \\
    \text{where} \quad &\mathcal{L}(\mathsf{U}, \mathsf{M}, \mathsf{V}) =\norm{\mathsf{A} - \mathsf{U}\mathsf{M}\mathsf{V}^*}_{\mathrm{HS}}^2  - \norm{\mathsf{A}}_{\mathrm{HS}}^2\\
    &\Omega(\mathsf{U}, \mathsf{V}) = \norm{\mathsf{U}^*\mathsf{U} - \mathsf{Id}}_{\mathrm{HS}}^2 + \norm{\mathsf{V}^*\mathsf{V} - \mathsf{Id}}_{\mathrm{HS}}^2,
\end{align*}
and $\mathsf{U}: \mathbb{R}^d \!\to\! \mathcal{F}, \mathsf{M}: \mathbb{R}^d \!\to\! \mathbb{R}^d, \mathsf{V}: \mathbb{R}^d \!\to\! \mathcal{G}$ are linear operators, is given by the rank-$d$ truncated SVD of $\mathsf{A}$.
\end{theorem}

Applied to our problem, \Cref{thm:eckart_young_mirsky_loss} yields the following optimization problem:
\begin{equation}
\label{eq:opt_hs}
\begin{aligned}
    \min_{\U,\M, \V} \quad & \mathcal{L}(\U,\M,\V) + \gamma\,\Omega(\U,\V) \\
    \text{where} \quad & \mathcal{L}(\U,\M,\V) = \norm{\DCE - \U\M\V^*}_\mathrm{HS}^2 - \norm{\DCE}_\mathrm{HS}^2\\
    &\Omega(\U,\V) = \norm{\U^*\U - \mathsf{Id}}_\mathrm{HS}^2 + \norm{\V^*\V - \mathsf{Id}}_\mathrm{HS}^2.
\end{aligned}
\end{equation}
As before, we take $\U\!:\! \R^d\to \LiiX$, $\V\!:\! \R^d\to\LiiYdot$, and $\M\!:\!\R^d\to\R^d$ to be linear maps spanning (at most) $d$-dimensional subspaces of $\LiiX$, $\LiiYdot$, and $\R^d$. We further assume that the functions spanned by $\U$ and $\V$ are centered, i.e., $\range(\U), \range(\V) \subset \linspan(1)^\perp$. This reduces the search space for $\U$ and $\V$, effectively simplifying the problem, and is justified by $1 \in \kernel(\DCE) \cap \kernel(\DCE^*)$ and the left and right singular functions of $\DCE$ being contained in  $\kernel(\DCE^*)^\perp$ and  $\kernel(\DCE)^\perp$, respectively.

By subtracting the constant $\norm{\DCE}_\mathrm{HS}^2 < \infty$, we can expand the loss function as follows
\begin{equation}
\begin{aligned}
    \norm{\DCE - \U\M\V^*}_\mathrm{HS}^2 &- \norm{\DCE}_\mathrm{HS}^2\\
    % \norm{\DCE - \U\M\V^*}_\mathrm{HS}^2 - \texttt{constant} &= \\
    % &= \norm{\U\M\V^*}_\mathrm{HS}^2 - 2\inner{\DCE, \U\M\V^*}_\mathrm{HS} \\
    &= \norm{\U\M\V^*}_\mathrm{HS}^2 - 2\inner{\covsrddot{X}{Y}, \U\M\V^*}_\mathrm{HS} + 2\inner{\covs{X}{Z}\covsrddot{Z}{Y}, \U\M\V^*}_\mathrm{HS} \\
    % &= \norm{\U\M\V^*}_\mathrm{HS}^2 - 2\inner{\CElddot{Y}{X}, \U\M\V^*}_\mathrm{HS} + 2\inner{\CE{Z}{X}\CElddot{Y}{Z}, \U\M\V^*}_\mathrm{HS} \\
    % &= \mathcal{L}_\texttt{l2\_contrastive}(\U,\M,\V) + 2\trace\,(\CE{Z}{X}\CElddot{Y}{Z}\V\M^*\U^*)
    % &= \trace(C_{UU}MC_{VV}M^*) - 2\trace(C_{UV}M^*) + 2\trace(\CE{Z}{X}\CElddot{Y}{Z}\V\M^*\U^*),
    % &= \trace(C_{UU}MC_{VV}M^*) -2 \trace(C_{U\tilde{V}}M^*),
    &= \trace(C_{UU}MC_{VV}M^T) - 2\trace(C_{UV}M^T) + 2\trace(\covs{X}{Z}\covsrddot{Z}{Y}\V\M^*\U^*),
\end{aligned}
\end{equation}
where $C_{UU}$ and $C_{VV}$ are the covariance matrices of $U$ and $V$, $C_{UV}$ is the cross-covariance matrix of $U = [u_1(X) \dotsc u_d(X)]^\top$ and $V = [v_1(\ddot{Y}) \dotsc v_d(\ddot{Y})]^\top$, and $M$ is the matrix of $\M$ w.r.t. the canonical basis of $\R^d$. To deal with the last term $\overline{\F} = \covsrddot{Z}{Y}\V\M^*\U^*\covs{X}{Z}$, we first symmetrize it $\F = (\overline{\F} + \overline{\F}^*)/2$, and then rely on the fact that $\trace(\F) = \trace(\W\W^*\F)$, for any linear map $\W\!:\!\R^{2d}\to\LiiZ$ satisfying $\W^*\W = I_{2d}$ and $\range(\F) \subseteq \range(\W)$, i.e., any orthogonal projector $\W\W^*$ onto $\range(\F)$. Again, since the constant function $\mathbf{1}\in\LiiZ$ satisfies $\mathbf{1} \in \kernel(\F)$, we search for centered functions by requiring $\range(\W) \subset \linspan(\mathbf{1})^\perp$. With this, the last term reduces to $2\trace\!\big(C_{UW}C_{WV}M^T\big)$, where $C_{UW}$ and $C_{WV}$ are the cross-covariance matrices of $U,W$ and $W,V$. The regularization terms can be expanded similarly
\begin{equation}
    \norm{\U^*\U - \mathsf{Id}}_\mathrm{HS}^2 = \norm{C_{UU} - I_d}_\mathrm{F}^2, \qquad \norm{\V^*\V - \mathsf{Id}}_\mathrm{HS}^2 = \norm{C_{VV} - I_d}_\mathrm{F}^2.
\end{equation}
{\bf Bi-level formulation.} Since evaluating $\trace(\F)$ requires learning an orthogonal projector $\W\W^*$ such that $\range(\W) \supset \range(\F)$, we adopt the following bi-level optimization problem moving forward
\begin{equation}\label{eq:bilevel_formulation}
\begin{aligned}
    \min_{U,M,V} \quad & \mathcal{L}_{\mathrm{out}}(U,M,V,W_\mathrm{in}) + \gamma\, \Omega_{\mathrm{out}}(U,V) \\
    \text{s.t.} \quad & (W_\mathrm{in}, N_\mathrm{in}) \in \argmin_{W\!,\,N} \mathcal{L}_{\mathrm{in}}(W,N) + \gamma\, \Omega_{\mathrm{in}}(W)\\
    \text{where} \quad & \mathcal{L}_{\mathrm{out}}(U,M,V,W_\mathrm{in}) = \trace(C_{UU}MC_{VV}M^T) - 2\trace(C_{UV}M^T) + 2\trace\!\big(C_{UW}C_{WV}M^T\big)\\
    & \mathcal{L}_{\mathrm{in}}(W,N) = \trace(C_{WW}NC_{WW}N^T) - \trace(C_{UW}(N+N^T)C_{WV}M^T) ,
\end{aligned}
\end{equation}
% \begin{align*}
%     \min_{U,M,V} \quad & \mathcal{L}_\mathrm{out}(U,M,V) \equiv \trace\!\big(C_{UU}MC_{VV}M^T) - 2 \trace\!\big(C_{UV}M^T\big) + 2\trace\!\big(C_{UW}C_{WV}M^T\big)\\
%     \text{s.t.} \quad & C_{WW} = I \\
%     & (W, N) \in \argmin_{W'\!,\,N'} \mathcal{L}_\mathrm{in}(W',N') := \trace(C_{W'W'}N'C_{W'W'}N'^T) - \trace(C_{UW'}(N'+N'^T)C_{W'V}M^T),
% \end{align*}
%
where $C_{WW}$ is the covariance matrix of $W = [w_1(Z) \dotsc w_{2d}(Z)]^\top$ and $\Nop\!:\!\R^{2d}\to\R^{2d}$ is a linear map. The inner loss is the same regularized Hilbert-Schmidt loss as in \eqref{eq:opt_hs}, but applied to learning the finite-rank operator $\F$ with $\W\Nop\W^*$.
% \vladi{Instead of bilevel we can pick at random say 2m RFF and compute that trace, we pay inversion but we don't need gradients through that inverse.. This might be also interesting because in the test we can use this randomness to get p-value maybe?}

{\bf Empirical resolution.} In practice, we can represent the finite-rank operators $\U$, $\V$, $\W$, $\M$, and $\Nop$ using neural networks $u_\theta\!:\!\X \to \R^d$, $v_\theta\!:\!\Y \times \Z \to \R^d$, and $w_\theta\!:\!\Z \to \R^{2d}$, along with learnable weight matrices $M_\theta \in \R^{d \times d}, N_\theta \in \R^{2d\times 2d}$. These parameterizations define subspaces of $\LiiX$, $\LiiYdot$, and $\LiiZ$, and model the action of the operators as linear maps between them. Moreover, the bi-level optimization problem \eqref{eq:bilevel_formulation} corresponds to simultaneously minimizing two regularized contrastive losses, which is amenable to traditional gradient-based optimization and features manifold connections to contrastive learning \citep{haochen_2021_provable}
% \vladi{Need to put centering, introduce $\bar{u}_\theta(X_i)$ and same for v and w}
%
% \begin{equation*}
% \begin{aligned}
% \Lossout(\theta) &= \frac{1}{m(m-1)}\sum_{i\neq j}^m \inner{\overline{u}_\theta(X_i), M_\theta \overline{v}_\theta(\ddot{Y}_j)}^2 \\
% &\;\; - \frac{2}{m-1}\sum_{i=1}^m \inner{\overline{u}_\theta(X_i), M_\theta \overline{v}_\theta(\ddot{Y}_i)} \\
% &\;\; + \frac{2}{m(m-1)}\sum_{i\neq j}^m \inner{\overline{u}_\theta(X_i), \overline{v}_\theta(\ddot{Y}_j)}\inner{\overline{w}_\theta(Z_i), M_\theta \overline{w}_\theta(Z_j)}, \\
% \Lossin(\theta) &= \frac{1}{m(m-1)}\sum_{i\neq j}^m \inner{\overline{w}_\theta(Z_i), N_\theta \overline{w}_\theta(Z_j)}^2 \\
% &\;\; - \frac{2}{m(m-1)}\sum_{i\neq j}^m \inner{\overline{u}_\theta(X_i), N_\theta \overline{v}_\theta(\ddot{Y}_j)}\inner{\overline{w}_\theta(Z_i), M_\theta \overline{w}_\theta(Z_j)}, \\
% \end{aligned}
% \end{equation*}
% {\color{red}
% [OPTION 1]: 
\begin{equation*}
\begin{aligned}
\Lossout(\theta) &= \frac{1}{m(m-1)}\sum_{i\neq j}^m \inner{\overline{u}_\theta(X_i), M_\theta \overline{v}_\theta(\ddot{Y}_j)}^2 \\
&\;\; - \frac{2}{m}\sum_{i=1}^m \inner{\overline{u}_\theta(X_i), M_\theta \overline{v}_\theta(\ddot{Y}_i)} \\
&\;\; + \frac{2}{m(m-1)}\sum_{i\neq j}^m \inner{\overline{u}_\theta(X_i), M_\theta\overline{v}_\theta(\ddot{Y}_j)}\inner{\overline{w}_\theta(Z_i), \overline{w}_\theta(Z_j)}, \\
\Lossin(\theta) &= \frac{1}{m(m-1)}\sum_{i\neq j}^m \inner{\overline{w}_\theta(Z_i), N_\theta \overline{w}_\theta(Z_j)}^2 \\
&\;\; - \frac{2}{m(m-1)}\sum_{i\neq j}^m \inner{\overline{u}_\theta(X_i), M_\theta \overline{v}_\theta(\ddot{Y}_j)}\inner{\overline{w}_\theta(Z_i), N_\theta \overline{w}_\theta(Z_j)}, \\
\end{aligned}
\end{equation*}
% [OPTION 2 (just change in upper level)]: 
% \begin{equation*}
% \begin{aligned}
% \Lossout(\theta) &= \frac{1}{m(m-1)}\sum_{i\neq j}^m \inner{\overline{u}_\theta(X_i), M_\theta \overline{v}_\theta(\ddot{Y}_j)}^2 \\
% &\;\; - \frac{2}{m-1}\sum_{i=1}^m \inner{\overline{u}_\theta(X_i), M_\theta \overline{v}_\theta(\ddot{Y}_i)} \\
% &\;\; + \frac{1}{(m-1)}\sum_{i\neq j}^m \inner{\overline{w}_\theta(Z_i), N_\theta\overline{w}_\theta(Z_j)}, \\
% \end{aligned}
% \end{equation*}
% }
%
where $\overline{\cdot}$ denotes centering and $N_\theta$ is taken to be symmetric. Furthermore, orthonormality regularizers such as the following are used with strength $\gamma > 0$
\begin{equation*}
\begin{aligned}
\Regout(\theta) &= \norm{\widehat{C}_{U_\theta U_\theta} - I_d}_{\mathrm{F}}^2 + \norm{\widehat{C}_{V_\theta V_\theta} - I_d}_{\mathrm{F}}^2, \\ 
\Regin(\theta) &= \norm{\widehat{C}_{W_\theta W_\theta} - I_{2d}}_{\mathrm{F}}^2.
\end{aligned}
\end{equation*}

\section{Proofs of Statistical Results}
\label{sec:app_stat_proofs}

\subsection{Background on statistical hypothesis testing}
\label{app:backgroundtest}
\paragraph{Statistical hypothesis testing.}

We briefly recall the fundamental concepts of statistical hypothesis testing, as they will be used throughout the paper.

Given a sample $\mathcal{D}_n = \{(X_i, Y_i, Z_i)\}_{i=1}^n$ drawn i.i.d. from a joint distribution $P_{XYZ}$, a \emph{statistical test} is a decision rule that determines whether to reject a null hypothesis $H_0$ in favor of an alternative hypothesis $H_1$. This decision is typically based on a \emph{test statistic} $T_n = T(\mathcal{D}_n)$, a real-valued function designed to capture evidence against $H_0$.

The test proceeds by defining a \emph{rejection region} $\mathcal{R}_n \subset \mathbb{R}$ such that:
\begin{equation}
\text{Reject } H_0 \quad \Longleftrightarrow \quad T_n \in \mathcal{R}_n.
\end{equation}
Commonly, $\mathcal{R}_n$ takes the form of a threshold rule $\mathcal{R}_n = \{t \in \mathbb{R} : t \geq t_\alpha\}$, where $t_\alpha$ is the \emph{critical value} chosen to control the \emph{Type I error rate} at a desired significance level $\alpha \in (0,1)$. Specifically, the Type I error corresponds to the probability of incorrectly rejecting $H_0$ when it is true:
\begin{equation}
\alpha = \mathbb{P}_{H_0}(T_n \in \mathcal{R}_n).
\end{equation}

Conversely, the \emph{Type II error} is the probability of failing to reject $H_0$ when the alternative hypothesis holds:
\begin{equation}
\beta(P_{XYZ}) = \mathbb{P}_{H_1}(T_n \notin \mathcal{R}_n),
\end{equation}
where we emphasize that $\beta$ depends on the specific alternative distribution $P_{XYZ} \in H_1$. The \emph{power} of the test is defined as the probability of correctly rejecting $H_0$ under $H_1$:
\begin{equation}
\text{Power} = 1 - \beta(P_{XYZ}).
\end{equation}

The objective in hypothesis testing is to design test statistics and rejection regions that simultaneously control the Type I error at level $\alpha$ and achieve high power against relevant alternatives. In nonparametric settings, where minimal assumptions are made on $P_{XYZ}$, this is particularly challenging.

A fundamental notion in the asymptotic analysis of hypothesis tests is that of \emph{contiguity} \citep{van2000asymptotic}. A sequence of distributions $\{Q_n\}$ is said to be contiguous with respect to the sequence of distributions $\{P_n\}$ if, for every sequence of events $\{A_n\}$, we have:
\begin{equation}
P_n(A_n) \to 0 \quad \Longrightarrow \quad Q_n(A_n) \to 0,
\end{equation}
as $n \to \infty$. In the context of statistical testing, contiguity implies that the alternative becomes asymptotically indistinguishable from the null: no test statistic can achieve both vanishing Type I and Type II error rates. %In the context of conditional independence testing, contiguity arises when the alternative hypothesis allows dependence signals to be arbitrarily weak, making it impossible to consistently detect deviations from $H_0$ without further structural assumptions on $P_{XYZ}$.

\subsection{Background on sub-Gaussian random variables}
We recall the definition of a sub-Gaussian random variable and some of its useful properties. 

Let $\psi_2(x)= e^{x^2}-1$, $x\geq 0$. We define the $\psi_2$-Orlicz norm of a random variable $\eta$ as 
$$
\norm{\eta}_{\psi_2}:= \inf\left\lbrace C>0\,:\, \mathbb{E}\left[\psi_2\left(\frac{|\eta|}{C}\right)\right]\leq 1 \right\rbrace.
$$

%We recall now the definition of a sub-Gaussian random vector. 
\begin{definition}[Sub-Gaussian random vector]
\label{def:subgaussian}
A centered random vector $X\in \mathbb{R}^{\underline{d}}$ , with probability distribution denoted $\mu_X$, will be called sub-Gaussian iff, for all $u\in \mathbb{R}^{\underline{d}}$,
$$
\norm{\langle X,u \rangle}_{\psi_2} \lesssim \norm{\langle X,u \rangle}_{L_2(\mu_{X})}.
$$
\end{definition}

%\textcolor{red}{Add sub-Gaussian definition with Orlicz norm $\norm{\cdot}_{\psi_2}$.}

\begin{lemma}[(Sub-Gaussian random variable) Lemma 5.5. in \cite{vershynin2011introduction}]
\label{lem:sub-gaussian-def}
Let $Z$ be a random variable. Then, the following assertions are equivalent with parameters $K_i>0$ differing from each other by at most an absolute constant factor.
\begin{enumerate}
    \item Tails:  
      $\mathbb{P} \{ |Z| > t \} \le \exp(1-t^2/K_1^2)$ for all $t \ge 0$;
    \item Moments:
      $(\mathbb{E} |Z|^p)^{1/p} \le K_2 \sqrt{p}$ for all $p \ge 1$;
    \item Super-exponential moment:
      $\mathbb{E} \exp(Z^2/K_3^2) \le 2$.
  \end{enumerate}
A random variable $Z$ satisfying any of the above assertions is called a sub-Gaussian random variable. We will denote by $K_3$ the sub-Gaussian norm.
 %  Moreover, if $\EE X = 0$ then properties 1--3 are also equivalent 
 %  to the following one: 
 %  \begin{enumerate} \setcounter{enumi}{3}
 %    \item[(d)] Moment generating function:
 %      $\EE \exp(tX) \le \exp(t^2 K_4^2)$ for all $t \in \R$.
 % \end{enumerate}
\end{lemma}
Consequently, a sub-Gaussian random variable satisfies the following equivalence of moments property. There exists an absolute constant $c>0$ such that for any $m\geq 2$, 
\begin{equation}
\label{eq:momentequiv-subGaussian}
    \big(\mathbb{E}|Z|^m \big)^{1/m} \leq c K_3\sqrt{m} \big(\mathbb{E}|Z|^2 \big)^{1/2}.
\end{equation}

Let $\eta_1,\dots,\eta_n, \eta$ be i.i.d $K$-sub-Gaussian random vectors in $\mathbf{R}^d$. Then there exists an absolute constant $C > 0$ such that for all $\delta \in (0,1)$,
\[
\mathbb{P}\!\left( \big\| \frac{1}{n}\sum_{I=1}^n \eta_i - \mathbb{E}[\eta] \big\| 
\leq C K \left( \sqrt{\frac{d}{n}} + \sqrt{\frac{\log(1/\delta)}{n}} \right) \right) 
\geq 1 - \delta.
\]
where $\|\cdot\|$ is the Euclidean norm.

\subsection{Background on Lindeberg-Feller Multivariate Central Limit Theorem}

% \paragraph{Statement of the Problem}

% Let \( u_{\theta}(x_i), v_{\theta}(\tilde{y}_i) \in \mathbb{R}^d \) be independent, centered, \(K\)-sub-Gaussian random vectors. Define the sequence of scalar weights \( [Q]_{i,i} \in (0,1) \), with the property that
% \[
% \sup_{i} |[Q]_{i,i} - 1| \to 0 \quad \text{as } n \to \infty.
% \]
% We study the convergence in distribution of the random matrix:
% \[
% Z_n := \frac{1}{\sqrt{n}} \sum_{i=1}^n [Q]_{i,i} \, u_{\theta}(x_i) \otimes v_{\theta}(\tilde{y}_i) \in \mathbb{R}^{d \times d}.
% \]

%\paragraph{Multivariate Central Limit Theorem and Lindeberg Condition}

We use the multivariate CLT for independent (not necessarily identically distributed) vector-valued random variables:

\begin{theorem}[Multivariate Central Limit Theorem]\label{thm:lindenberg_feller_clt}
Let \( \{W_i\}_{i=1}^n \subset \mathbb{R}^m \) be independent, mean-zero random vectors. Assume that
\[
\frac{1}{n} \sum_{i=1}^n \mathbb{E}[W_i W_i^\top] \to \Sigma \quad \text{and} \quad \text{for all } \varepsilon > 0, \quad \frac{1}{n} \sum_{i=1}^n \mathbb{E} \left[ \norm{W_i}^2 \cdot \mathbbm{1}_{\{ \norm{W_i} > \varepsilon \sqrt{n} \}} \right] \to 0.
\]
Then
\[
\frac{1}{\sqrt{n}} \sum_{i=1}^n W_i \xrightarrow{d} \mathcal{N}(0, \Sigma).
\]
\end{theorem}

\subsection{Technical results}
\label{app:technicalresults}

We prove now several technical results that will be used in several parts of the proof of our main results.
For the sake of brevity and to avoid cumbersome notation, we denote the learned features $\widehat{u}_\theta$, $\widehat{v}_\theta$, and $\widehat{w}_\theta$ obtained after the representation learning step by $u_\theta$, $v_\theta$, and $w_\theta$, respectively, and keep them fixed thereafter. In the testing phase, we compute the test statistic with new data $(\{x_i,y_i,z_i)\}_{i=1}^{n}$ not used during the representation learning phase.

\paragraph{Spectral structure of $\widehat{P}_{\theta}$}

We define of the $n\times d$ matrices $\widehat{U}_{\theta}, \widehat{V}_{\theta}$ and the $n\times (2d)$ matrix $\widehat{W}_{\theta}$ as follows:
$$
\widehat{U}_{\theta} = \frac{1}{\sqrt{n}}\left[ u_{\theta}(x_1)  \vert \cdots \vert u_{\theta}(x_n) \right]^\top,\quad \widehat{V}_{\theta} = \frac{1}{\sqrt{n}}\left[ v_{\theta}(\ddy_1)  \vert \cdots \vert v_{\theta}(\ddy_n) \right]^\top,\quad \widehat{W}_{\theta} = \frac{1}{\sqrt{n}}\left[ w_{\theta}(z_1)  \vert \cdots \vert w_{\theta}(z_n) \right]^\top.%\in \mathbb{R}^{n\times d}
$$
Set also 
$$
\widehat{P}_{\theta} = I_n - \widehat{W}_{\theta}\widehat{W}_{\theta}^\top.
$$

We assume from now on that $n> 2d$. Note first that the nonzero eigenvalues of $\widehat{W}_{\theta} \widehat{W}_{\theta}^\top$ are the same as those of $\widehat{W}_{\theta}^\top \widehat{W}_{\theta}$. We set $\mathrm{Cov}(w_\theta) = \mathbb{E}[\widehat{W}_{\theta}^\top \widehat{W}_{\theta} ] $. Next we have
\begin{align*}
    \norm{\widehat{W}_{\theta}^\top \widehat{W}_{\theta} - 
    I_d } &\leq \norm{ \widehat{W}_{\theta}^\top\widehat{W}_{\theta} - \mathbb{E}[\widehat{W}_{\theta}^\top \widehat{W}_{\theta} ]}+ \norm{ \mathbb{E}[\widehat{W}_{\theta}^\top \widehat{W}_{\theta} ] - I_d}\\
&= \underbrace{\norm{ \widehat{W}_{\theta}^\top\widehat{W}_{\theta} - \mathbb{E}[\widehat{W}_{\theta}^\top \widehat{W}_{\theta} ]}}_{(I)} + \underbrace{\norm{ \mathrm{Cov}(w_\theta) - I_d}}_{(II)}.
\end{align*}
Recall the optimization gap condition:
$$
(II) \leq \mathcal{E}_m.
$$
We study now $(I)$. Note that
$$
\widehat{W}_{\theta}^\top\widehat{W}_{\theta} - \mathbb{E}[\widehat{W}_{\theta}^\top \widehat{W}_{\theta} ] =  \frac{1}{n}\sum_{i\in [n]} w_{\theta}(z_i) \otimes w_{\theta}(z_i) - \mathrm{Cov}(w_{\theta}(z)).
$$

The effective rank of a symmetric positive semi-definite matrix $A$ is
$$
\mathbf{r}(A) =\frac{\mathrm{tr}(A)}{\norm{A}}.
$$
Define the rate
$$
\psi_n(\delta): = c_K \norm{\mathrm{Cov}(w_{\theta}(z))} \left(\sqrt{ \frac{\mathbf{r}(\mathrm{Cov}(w_{\theta}(z)))}{n} }
\bigvee \frac{\mathbf{r}(\mathrm{Cov}(w_{\theta}(z)))}{n} 
\bigvee \sqrt{\frac{\log \delta^{-1}}{n}} 
\bigvee \frac{\log \delta^{-1}}{n}
\right),
$$
where $c_K>0$ is a numerical constant that can depend only $K$.

Using well-known concentration results for covariance operators like \citet[Corollary 2]{koltloucovariancebernoulli}, we get for any $\delta \in (0,1)$, w.p.a.l. $1-\delta$
\begin{align}
\label{eq:CovW_op_norm}
    \norm{\widehat{W}_{\theta}^{\top}\widehat{W}_{\theta} - \mathrm{Cov}(w_{\theta}(z))} \leq \psi_n(\delta).
\end{align}
Hence we get for any $\delta \in (0,1)$, w.p.a.l. $1-\delta$
$$
\norm{\widehat{W}_{\theta}^\top \widehat{W}_{\theta} - I_d} \leq \mathcal{E}_m + \psi_n(\delta).
$$

Hence we get on the same probability event:
\begin{equation}
\label{eq:eigenvalueWTW}
   \max_{j\in [2d]} |\widehat{\lambda}_j - 1 | \leq  \mathcal{E}_m +  \psi_n(\delta).
\end{equation}

Recall that the nonzero eigenvalues of $\widehat{W}_{\theta}\widehat{W}_{\theta}^{\top}$ are the same as the nonzero eigenvalues of $\widehat{W}_{\theta}^{\top}\widehat{W}_{\theta}$ (assuming the worst case that $\widehat{W}_{\theta}^{\top}$ is full rank). Hence the eigenvalues of $\widehat{W}_{\theta}\widehat{W}_{\theta}^{\top}$ are $0$ with multiplicity $n-2d$ and the above $\widehat{\lambda}_j$'s. Consequently $\widehat{P}_\theta$ admits spectral decomposition with eigenvalue $1$ of multiplicity $n-2d$ and the remaining $d$ eigenvalues $\mu_{j}$ satisfying:
\begin{equation}
\label{eq:eigenvaluesmujsmall}
    |\mu_{j}( \widehat{P}_{\theta})| \leq \psi_n(\delta) + \mathcal{E}_m, 
\quad \forall j\in [n-2d+1,n],
\end{equation}
and $\widehat{P}_{\theta}$ admits an orthonormal family $\{a_j\}_{j\in [n]}$ in $\mathbb{R}^n$ of eigenvectors such that 
\begin{equation}
\label{eq:Ptheta_def}
    \widehat{P}_{\theta} = \sum_{j=1}^{n-2d} a_j \otimes a_j + \sum_{j= n-2d+1}^{n} \mu_j \, a_j \otimes a_j.
\end{equation}

We discuss now the rate $\psi_n(\delta)$. We always have $\mathbf{r}(\mathrm{Cov}(w_{\theta}(z))) \leq 2d$. Furthermore, under the optimization gap condition, we have 
$
\norm{\mathrm{Cov}(w_{\theta}(z))} \leq 1+ \mathcal{E}_m$. Hence we have have the following upper bound for any $\delta\in (0,1)$ such that  $n \geq d \vee \log \delta^{-1}$:
\begin{align}
\label{eq:boundpsindelta}
    \psi_n(\delta) \leq c_K (1+ \mathcal{E}_m) \left( \sqrt{ \frac{2d}{n} }
\bigvee \sqrt{\frac{\log \delta^{-1}}{n}} 
%\bigvee \frac{\log \delta^{-1}}{n}
\right).
\end{align}

\paragraph{Asymptotic covariance structure under the null hypothesis.}

Using the tower property of conditional expectation, the conditional independence of $x_i$ and $\ddy_i$ given $z_i$ under the null hypothesis, and the fact that $\{x_{i},\ddy_i,z_i\}_{i=1}^{n}$ is an i.i.d. sample set, we have for any $i \in [n]$:
\begin{align}
\label{eq:covariance_single_term}
&\mathbb{E}\left[ \mathrm{vec}( u_{\theta}(x_{i}) \otimes v_{\theta}(\ddy_{i})) \, \mathrm{vec}( u_{\theta}(x_{i}) \otimes v_{\theta}(\ddy_{i}))^\top \right] \notag\\
&\hspace{2cm} = \mathbb{E}\left[ (u_{\theta}(x_{i}) \otimes v_{\theta}(\ddy_{i})) (u_{\theta}(x_{i}) \otimes v_{\theta}(\ddy_{i}))^\top \right] \notag\\
&\hspace{2cm} = \mathbb{E}\left[ u_{\theta}(x_{i}) u_{\theta}(x_{i})^\top \right] \otimes \mathbb{E}\left[ v_{\theta}(\ddy_{i}) v_{\theta}(\ddy_{i})^\top \right] \notag\\
&\hspace{2cm} = \mathrm{Cov}(u_{\theta}(x)) \otimes \mathrm{Cov}(v_{\theta}(\ddy)) \in \mathbb{R}^{d^2 \times d^2},
\end{align}
where we used the identity $(a \otimes b)(a \otimes b)^\top = (aa^\top) \otimes (bb^\top)$ in the third line and the independence of $u_{\theta}(x_i)$ and $v_{\theta}(\ddy_i)$ (conditional on $z_i$).

We now study the spectral structure of $\mathrm{Cov}(u_{\theta}(x)) \otimes \mathrm{Cov}(v_{\theta}(\ddy))$. Denote by $\{\lambda_j(u_\theta)\}_{j\in [d]}$ and $\{\lambda_j(v_\theta)\}_{j\in [d]}$ the eigenvalues of $\mathrm{Cov}(u_{\theta}(x))$ and $\mathrm{Cov}(v_{\theta}(\ddy))$ respectively. By properties of Kronecker products, $\mathrm{Cov}(u_{\theta}(x)) \otimes \mathrm{Cov}(v_{\theta}(\ddy))$ admits eigenvalues $\{\lambda_j(u_\theta) \lambda_k(v_\theta)\}_{j,k \in [d]}$.

Now we prove that $\mathrm{Cov}(u_{\theta}(x)) \otimes \mathrm{Cov}(v_{\theta}(\ddy))$ is approximately equal to $I_{d^2}$. To this end, we introduce the true representation basis $u,v$:
$$
u(\cdot) = [u_1(\cdot),\ldots, u_d(\cdot)]^\top, \quad v(\cdot) = [v_1(\cdot),\ldots, v_d(\cdot)]^\top,
$$
and we recall that $u$ and $v$ are orthonormal families w.r.t. $L^2_{\mu_x}$ and $L^2_{\mu_{\ddy}}$, meaning that 
$$
\mathrm{Cov}(u(x)) \otimes \mathrm{Cov}(v(\ddy)) = I_d \otimes I_d = I_{d^2}.
$$

Next we have 
\begin{align*}
&\mathrm{Cov}(u_{\theta}(x)) \otimes \mathrm{Cov}(v_{\theta}(\ddy)) - \mathrm{Cov}(u(x)) \otimes \mathrm{Cov}(v(\ddy))\\
&\hspace{0.5cm} = \big(\mathrm{Cov}(u_{\theta}(x)) - \mathrm{Cov}(u(x))\big) \otimes \mathrm{Cov}(v_{\theta}(\ddy)) \\
&\hspace{1.5cm} + \mathrm{Cov}(u(x)) \otimes \big(\mathrm{Cov}(v_{\theta}(\ddy)) - \mathrm{Cov}(v(\ddy))\big).
\end{align*}
Taking the operator norm and using that $\norm{A\otimes B}_{op} = \norm{A}_{op}\norm{B}_{op}$, we get
\begin{align*}
&\norm{\mathrm{Cov}(u_{\theta}(x)) \otimes \mathrm{Cov}(v_{\theta}(\ddy)) - \mathrm{Cov}(u(x)) \otimes \mathrm{Cov}(v(\ddy))}_{op}\\
&\hspace{1cm}\leq \norm{\mathrm{Cov}(u_{\theta}(x)) - \mathrm{Cov}(u(x))}_{op}\norm{\mathrm{Cov}(v_{\theta}(\ddy))}_{op} \\
&\hspace{2cm}+ \norm{\mathrm{Cov}(u(x))}_{op}\norm{\mathrm{Cov}(v_{\theta}(\ddy)) - \mathrm{Cov}(v(\ddy))}_{op}\\
&\hspace{1cm}=\norm{\mathrm{Cov}(u_{\theta}(x)) - I_d}_{op}\norm{\mathrm{Cov}(v_{\theta}(\ddy))}_{op} + \norm{I_d}_{op}\norm{\mathrm{Cov}(v_{\theta}(\ddy)) - I_d}_{op}.
\end{align*}
We recall that by definition of the optimization gap 
$$
\mathcal{E}_m\geq \max \left\lbrace \norm{\mathrm{Cov}(u_{\theta}(x)) - I_d}_{op}, \norm{\mathrm{Cov}(v_{\theta}(\ddy)) - I_d}_{op} \right\rbrace.
$$
Combining the last two displays gives
\begin{align*}
\norm{\mathrm{Cov}(u_{\theta}(x)) \otimes \mathrm{Cov}(v_{\theta}(\ddy)) - I_{d^2}}_{op} \leq (2+ \mathcal{E}_m) \mathcal{E}_m. 
\end{align*}
By standard perturbation bounds for eigenvalues, we deduce:
\begin{align}
\label{eq:eigenvalue_bound}
\max_{j,k\in [d]} \left\lbrace |\lambda_j(u_{\theta}) \lambda_k(v_\theta) - 1| \right\rbrace \leq (2+ \mathcal{E}_m) \mathcal{E}_m.
\end{align}

Similarly, using that $\mathrm{tr}(A \otimes B) = \mathrm{tr}(A) \mathrm{tr}(B)$, we have:
\begin{align}
\label{eq:trace_bound}
\mathrm{tr}\left(\mathrm{Cov}(u_{\theta}(x)) \otimes \mathrm{Cov}(v_{\theta}(\ddy))\right) = \mathrm{tr}\left(\mathrm{Cov}(u_{\theta}(x))\right) \mathrm{tr}\left(\mathrm{Cov}(v_{\theta}(\ddy))\right) \leq (1+\mathcal{E}_m)^2d^2.
\end{align}

\paragraph{Non-asymptotic bound in probability.}

\begin{lemma}
\label{lem:bound_proba_convergence}
Let $\Gamma  = \sum_{j=1}^{r} \lambda_j a_j\otimes a_j\in \mathbb{R}^{n\times n}$ where $\{a_j\}_{j\in r}$ is an orthonormal family of $\mathbb{R}^{n}$. We recall that the $\{x_i,\ddy_i,z_i\}_{i\in [n]}$ are i.i.d. copies of $(X,\ddot{Y},Z)$ and that $x \independent \ddy\, \vert\, z$ under the null hypothesis
.
% For any $\epsilon> 0$, we have 
% %     \begin{align}
% %     \mathbb{P}\left(  n\,\norm{\widehat{U}_{\theta}^\top \Gamma \widehat{V}_{\theta}}_F \geq \epsilon \right) 
% %     \leq  (1+\mathcal{E}_m)^2 \frac{r\, d^2}{n\,\epsilon^2} \norm{\Gamma}^2 .
% % \end{align}
% \begin{align}
% \label{eq:final_prob_bound_alternative}
% \mathbb{P}\left( n\,\norm{\widehat{U}_{\theta}^\top \Gamma \widehat{V}_{\theta}}_F^2 \geq \epsilon \right) &\lesssim \frac{1}{\epsilon} \left[ \frac{K^2 d r^2\|\Gamma\|^2}{n} (1 + \mathcal{E}_m) + K^4 r\|\Gamma\|^2 \frac{d^2}{n} + K^4 r\|\Gamma\|^2 \frac{d^{3/2}\sqrt{\log(n)}}{n} \right] + \frac{1}{n}.
% \end{align}
% Assume now in addition that  $X \independent \ddot{Y}\, \vert\, Z$. 
Then, for any $\epsilon> 0$, we have
   \begin{align}
\label{eq:final_prob_bound_cond_indep}
\mathbb{P}\left( n\,\norm{\widehat{U}_{\theta}^\top \Gamma \widehat{V}_{\theta}}_F^2 \geq \epsilon \right) &\lesssim \frac{1}{\epsilon}\left[ K^4 r\|\Gamma\|^2  \left( \frac{d^2}{n} + \frac{\log^2(n)}{n} \right)\right] + \frac{1}{n}.
\end{align}
\end{lemma}

\begin{proof}
We start with some preliminary observations exploiting sub-Gaussianity of the mappings $u_\theta(x)$ and $v_{\theta}(\ddy)$.

\noindent
\textbf{Properties of mean and empirical means of $u_\theta(x)$ and $v_{\theta}(\ddy)$.}~ In view of Assumption \ref{ass:subgaussian} we have, by sub-Gaussian concentration, there exists an absolute constant $C>0$ such that, w.p.a.l. $1-\delta$, 
\begin{align}
  \bigg\Vert  \mathbb{E}[u_{\theta}(x) ] - \frac{1}{n}\sum_{i=1}^n u_{\theta}(x_{i}) \bigg\Vert \bigvee \bigg\Vert  \mathbb{E}[v_{\theta}(\ddy) ] - \frac{1}{n}\sum_{i=1}^n v_{\theta}(\ddy_{i}) \bigg\Vert  \leq C K \left( \sqrt{\frac{d}{n}} +  \sqrt{\frac{\log(\delta^{-1})}{n}} \right).
\end{align}

Recall that the functions $u_{\theta}$ and $v_{\theta}$ are empirically centered. Hence we get for any $\delta \in (0,1)$, w.p.a.l. $1-\delta$
\begin{equation}
\label{eq:Expectation_utheta}
     \bigg\Vert  \mathbb{E}[u_{\theta}(x) ] \bigg\Vert \bigvee \bigg\Vert  \mathbb{E}[v_{\theta}(\ddy) ] \bigg\Vert  \leq C K \left( \sqrt{\frac{d}{n}} +  \sqrt{\frac{\log(\delta^{-1})}{n}} \right).
\end{equation}

Let $A := u_\theta(x) -  \mathbb{E}[u_{\theta}(x) ]\in \mathbb{R}^d$. Note that $A$ is $2K$-sub-Gaussian, i.e.,
\[
\sup_{v \in \mathbb{S}^{d-1}} \big\| \langle A, v \rangle \big\|_{\psi_2} \leq 2K.
\]
By the equivalence of moment property of sub-Gaussian distribution, there exists an absolute constant $C > 0$ such that, for any $m\geq 3$, we have 
\[
\mathbb{E} \|A\|^m \leq C^{m} (K \sqrt{d})^m  m^{m/2} \left(\mathbb{E} \|A\|^2 \right)^{m/2} .
\]

Now consider the uncentered moment $\mathbb{E} \|u_\theta(x)\|^m = \mathbb{E} \|\mathbb{E} [u_\theta(x)] + A\|^m$. By the triangle inequality and the convexity of $t \mapsto t^m$ on $[0,\infty)$,
\[
\|\mathbb{E} [u_\theta(x)]  + A\|^m \leq \big( \|\mathbb{E} [u_\theta(x)] \| + \|A\| \big)^m \leq 2^{m-1} \big( \|\mathbb{E} [u_\theta(x)] \|^m + \|A\|^m \big).
\]
Taking expectations yields
\begin{equation}
    \label{eq:equiv_moment_map}
\mathbb{E} \|u_\theta(x)\|^m \leq 2^{m-1} \Big( \|\mathbb{E} [u_\theta(x)] \|^m +  C^{m} (K \sqrt{d})^m  m^{m/2} \left(\mathbb{E} \|A\|^2 \right)^{m/2}  \Big).
\end{equation}

\noindent
\textbf{Representation of the rescaled Frobenius norm.}~ Note first that  $[\Gamma]_{i,k} = \sum_{j=1}^r \lambda_j a_{i,j} a_{k,j}$. We have
\begin{align*}
   \sqrt{n}  \widehat{U}_{\theta}^\top \Gamma \widehat{V}_{\theta} &= \frac{1}{\sqrt{n}}\sum_{j=1}^{r}\sum_{i_1,i_2=1}^{n} \lambda_j a_{i_1,j} a_{i_2,j}   u_{\theta}(x_{i_1}) \otimes v_{\theta}(\ddy_{i_2})
   %\\&
    = \frac{1}{\sqrt{n}}\sum_{i_1,i_2=1}^{n} [\Gamma]_{i_1,i_2}   u_{\theta}(x_{i_1}) \otimes v_{\theta}(\ddy_{i_2})
    % \\
    % &= \frac{1}{\sqrt{n}}\sum_{i=1}^{n} [\Gamma]_{i,i}   u_{\theta}(x_{i}) \otimes v_{\theta}(\ddy_{i}) + \frac{1}{\sqrt{n}}\sum_{i_1,i_2=1, i_1\neq i_2 }^{n} [\Gamma]_{i_1,i_2}   u_{\theta}(x_{i_1}) \otimes v_{\theta}(\ddy_{i_2})
    .
\end{align*}    

The Frobenius norm of $\sqrt{n}  \widehat{U}_{\theta}^\top \Gamma \widehat{V}_{\theta} $ is
\begin{align}
\label{eq:Frobenius_decomposition}
    n\,\norm{\widehat{U}_{\theta}^\top \Gamma \widehat{V}_{\theta}}_F^2  &= \frac{1}{n} \sum_{i_1,i_2,i_3,i_4 = 1}^{n} [\Gamma]_{i_1,i_2}  [\Gamma]_{i_3,i_4}  \langle u_{\theta}(x_{i_1})  , u_{\theta}(x_{i_3}) \rangle \, \langle  v_{\theta}(\ddy_{i_2}), v_{\theta}(\ddy_{i_4}) \rangle
\end{align}

\noindent
\textbf{Analysis of the rescaled Frobenius norm under the null hypotheis.}~ 

We exploit now the conditional independence of $x$ and $\ddy$ given $z$. We take the expectation in \eqref{eq:Frobenius_decomposition}.
% \begin{align}
% \label{eq:Frobenius_decomposition}
%     n\,\norm{\widehat{U}_{\theta}^\top \Gamma \widehat{V}_{\theta}}_F^2  &= \frac{1}{n} \sum_{i_1,i_2,i_3,i_4 = 1}^{n} [\Gamma]_{i_1,i_2}  [\Gamma]_{i_3,i_4}  
%     \langle u_{\theta}(x_{i_1})  , u_{\theta}(x_{i_3}) \rangle \, \langle  v_{\theta}(\ddy_{i_2}), v_{\theta}(\ddy_{i_4}) \rangle.
% \end{align}
Exploiting the tower property of conditional expectation, we get
\begin{align*}
    &\mathbb{E}\left[\langle u_{\theta}(x_{i_1})  , u_{\theta}(x_{i_3}) \rangle \, \langle  v_{\theta}(\ddy_{i_2}), v_{\theta}(\ddy_{i_4}) \rangle \right] \\
    &\hspace{1cm}=\mathbb{E}\left[   \mathbb{E}\left[\langle u_{\theta}(x_{i_1})  , u_{\theta}(x_{i_3}) \rangle \, \langle  v_{\theta}(\ddy_{i_2}), v_{\theta}(\ddy_{i_4}) \rangle \mid z_{i_1},z_{i_2},z_{i_3},z_{i_4} \right] \right]\\
    &\hspace{1cm}=  \mathbb{E}\left[\langle u_{\theta}(x_{i_1})  , u_{\theta}(x_{i_3}) \rangle\right]  \, \mathbb{E}\left[  \langle  v_{\theta}(\ddy_{i_2}), v_{\theta}(\ddy_{i_4}) \rangle  \right]\\
    &\hspace{1cm}=
    \begin{cases}
        \langle \mathbb{E}\left[u_{\theta}(x_{i_1}) \right]  , \mathbb{E}\left[u_{\theta}(x_{i_3}) \right] \rangle  \,  \langle \mathbb{E}\left[ v_{\theta}(\ddy_{i_2}) \right], \mathbb{E}\left[v_{\theta}(\ddy_{i_4})  \right]\rangle  &\text{if $i_1 \neq i_3$ and $i_2 \neq i_4$}\\
        \mathbb{E}\left[\| u_{\theta}(x_{i_1})\|^2\right]  \, \mathbb{E}\left[  \|  v_{\theta}(\ddy_{i_2})\|^2  \right]  &\text{if $i_1 =i_3$ and $i_2 = i_4$}
    \end{cases}.
\end{align*}
Hence
\begin{align*}
    &|\mathbb{E}\left[\langle u_{\theta}(x_{i_1})  , u_{\theta}(x_{i_3}) \rangle \, \langle  v_{\theta}(\ddy_{i_2}), v_{\theta}(\ddy_{i_4}) \rangle \right] |\\
    &\hspace{1cm}\leq 
    \begin{cases}
        \| \mathbb{E}\left[u_{\theta}(x_{i_1}) \right] \|^2  \,  \| \mathbb{E}\left[ v_{\theta}(\ddy_{i_2})\right] \|^2  &\text{if $i_1 \neq i_3$ and $i_2 \neq i_4$}\\
        \mathbb{E}\left[\| u_{\theta}(x_{i_1})\|^2\right]  \, \mathbb{E}\left[  \|  v_{\theta}(\ddy_{i_2})\|^2  \right]  &\text{if $i_1 =i_3$ and $i_2 = i_4$}
    \end{cases}
\end{align*}
Recall that Assumption \ref{ass:subgaussian} implies $\mathbb{E}[\|u_{\theta}(x)\|^2] \leq CK^2 d$ and $\mathbb{E}[\|v_{\theta}(y)\|^2] \leq CK^2 d$. We also have in view of \eqref{eq:Expectation_utheta}, for any $\delta \in (0,1)$, w.p.a.l. $1-\delta$, for any indices satisfying $i_1 \neq i_3$ and $i_2 \neq i_4$,
\begin{equation}
\label{eq:bound_offdiag_cond_indep}
     |\mathbb{E}\left[\langle u_{\theta}(x_{i_1})  , u_{\theta}(x_{i_3}) \rangle \, \langle  v_{\theta}(\ddy_{i_2}), v_{\theta}(\ddy_{i_4}) \rangle \right] | \leq  C^4 K^4 \left( \sqrt{\frac{d}{n}} +  \sqrt{\frac{\log(\delta^{-1})}{n}} \right)^4,
\end{equation}

Taking expectations in \eqref{eq:Frobenius_decomposition} and decomposing the sum into diagonal and off-diagonal terms, we obtain
\begin{align*}
    \mathbb{E}\left[n\,\norm{\widehat{U}_{\theta}^\top \Gamma \widehat{V}_{\theta}}_F^2\right]  &= \frac{1}{n} \sum_{\substack{i_1,i_2=1}}^{n} [\Gamma]_{i_1,i_2}  [\Gamma]_{i_1,i_2} \mathbb{E}\left[\| u_{\theta}(x_{i_1})\|^2\right]  \mathbb{E}\left[  \|  v_{\theta}(\ddy_{i_2})\|^2  \right] \\
    &\quad + \frac{1}{n} \sum_{\substack{i_1,i_2,i_3,i_4=1 \\ (i_1,i_2) \neq (i_3,i_4)}}^{n} [\Gamma]_{i_1,i_2}  [\Gamma]_{i_3,i_4} \mathbb{E}\left[\langle u_{\theta}(x_{i_1})  , u_{\theta}(x_{i_3}) \rangle \, \langle  v_{\theta}(\ddy_{i_2}), v_{\theta}(\ddy_{i_4}) \rangle \right].
\end{align*}
For the diagonal terms, we have
\begin{align*}
    \left|\frac{1}{n} \sum_{i_1,i_2=1}^{n} [\Gamma]_{i_1,i_2}^2 \mathbb{E}\left[\| u_{\theta}(x)\|^2\right]  \mathbb{E}\left[  \|  v_{\theta}(y)\|^2  \right]\right| &\leq \frac{CK^4 d^2}{n} \sum_{i_1,i_2=1}^{n} [\Gamma]_{i_1,i_2}^2 = CK^4 d^2 \frac{\|\Gamma\|_F^2}{n} \leq CK^4 d^2 \frac{r\|\Gamma\|^2}{n}.
\end{align*}
For the off-diagonal terms, using \eqref{eq:bound_offdiag_cond_indep} with $\delta = 1/n$, w.p.a.l. $1-1/n$,
\begin{align*}
    &\left|\frac{1}{n} \sum_{\substack{i_1,i_2,i_3,i_4=1 \\ (i_1,i_2) \neq (i_3,i_4)}}^{n} [\Gamma]_{i_1,i_2}  [\Gamma]_{i_3,i_4} \mathbb{E}\left[\langle u_{\theta}(x_{i_1})  , u_{\theta}(x_{i_3}) \rangle \, \langle  v_{\theta}(\ddy_{i_2}), v_{\theta}(\ddy_{i_4}) \rangle \right]\right| \\
    &\hspace{2cm}\leq \frac{C^4 K^4}{n} \left( \sqrt{\frac{d}{n}} +  \sqrt{\frac{\log(n)}{n}} \right)^4 \sum_{\substack{i_1,i_2,i_3,i_4=1}}^{n} |[\Gamma]_{i_1,i_2}|  |[\Gamma]_{i_3,i_4}| \\
    &\hspace{2cm}\leq \frac{C^4 K^4}{n} \left( \sqrt{\frac{d}{n}} +  \sqrt{\frac{\log(n)}{n}} \right)^4 \left(\sum_{i_1,i_2=1}^{n} |[\Gamma]_{i_1,i_2}|\right)^2 \\
    &\hspace{2cm}\leq \frac{C^4 K^4}{n} \left( \sqrt{\frac{d}{n}} +  \sqrt{\frac{\log(n)}{n}} \right)^4 n^2\|\Gamma\|_F^2 \\
    &\hspace{2cm}\leq C^4 K^4 nr\|\Gamma\|^2 \left( \frac{d^2}{n^2} + \frac{\log^2(n)}{n^2} \right).
\end{align*}
Combining both terms, we obtain w.p.a.l. $1-1/n$,
\begin{align}
\label{eq:bound_expectation_cond_indep}
    \left|\mathbb{E}\left[n\,\norm{\widehat{U}_{\theta}^\top \Gamma \widehat{V}_{\theta}}_F^2\right]\right| &\lesssim  K^4 r\|\Gamma\|^2 \left( \frac{d^2}{n} + \frac{\log^2(n)}{n} \right).
\end{align}

\medskip
\textbf{Final step.}~ Let $B_n$ denote the event
\begin{align*}
B_n := \left\{ \bigg\Vert  \mathbb{E}[u_{\theta}(x) ] \bigg\Vert \bigvee \bigg\Vert  \mathbb{E}[v_{\theta}(\ddy) ] \bigg\Vert  \leq C K \left( \sqrt{\frac{d}{n}} +  \sqrt{\frac{\log(n)}{n}} \right) \right\}.
\end{align*}
By \eqref{eq:Expectation_utheta} with $\delta = 1/n$, we have $\mathbb{P}(B_n^c) \leq 1/n$.

On the event $B_n$, by \eqref{eq:bound_expectation_cond_indep}, we have
\begin{align}
\label{eq:bound_S_on_Bn_cond_indep}
\mathbb{E}[S \mathbf{1}_{B_n}] &\lesssim  K^4 r\|\Gamma\|^2 \left( \frac{d^2}{n} + \frac{\log^2(n)}{n} \right).
\end{align}

For any $\epsilon > 0$, we have
\begin{align*}
\mathbb{P}\left( n\,\norm{\widehat{U}_{\theta}^\top \Gamma \widehat{V}_{\theta}}_F^2 \geq \epsilon \right) &= \mathbb{P}\left( S \geq \epsilon \right) \\
&= \mathbb{P}(S \geq \epsilon, B_n) + \mathbb{P}(S \geq \epsilon, B_n^c) \\
&\leq \mathbb{P}(S \geq \epsilon, B_n) + \mathbb{P}(B_n^c) \\
&\leq \mathbb{P}(S \geq \epsilon, B_n) + \frac{1}{n}.
\end{align*}

By Markov's inequality,
\begin{align*}
\mathbb{P}(S \geq \epsilon, B_n) &= \mathbb{P}(S \mathbf{1}_{B_n} \geq \epsilon) \leq \frac{\mathbb{E}[S \mathbf{1}_{B_n}]}{\epsilon}.
\end{align*}

Combining with \eqref{eq:bound_S_on_Bn_cond_indep}, we obtain
\begin{align}
\mathbb{P}\left( n\,\norm{\widehat{U}_{\theta}^\top \Gamma \widehat{V}_{\theta}}_F^2 \geq \epsilon \right) &\lesssim \frac{K^4 r\|\Gamma\|^2}{\epsilon}  \left( \frac{d^2}{n} + \frac{\log^2(n)}{n} \right) + \frac{1}{n}.
\end{align}

\end{proof}

\subsection{Approximation of $\That$.} 
\label{app:That_approx}
%\TODO{Discuss this! TBC: We prove that $\That = ?????? n\, \norm{?}_F^2 + o_{P}\left( 1 \right)  ????????????$ provided that $n \gg d$.}

We recall that
\begin{align*}
  \That &= n\norm{\widehat{U}_{\theta}^\top \widehat{V}_{\theta} - \widehat{U}_{\theta}^\top \widehat{W}_{\theta} \widehat{W}_{\theta}^\top \widehat{V}_{\theta} }_F^2 =  n \norm{ \widehat{U}_{\theta}^\top [ I_n - \widehat{W}_{\theta} \widehat{W}_{\theta}^\top] \widehat{V}_{\theta} }_F^2 = n\norm{ \widehat{U}_{\theta}^\top \widehat{P}_{\theta} \widehat{V}_{\theta} }_F^2 
  %\mathrm{tr}(\widehat{P}_{\theta} \widehat{V}_{\theta}\widehat{V}_{\theta}^{\top} \widehat{P}_{\theta}\widehat{U}_{\theta}\widehat{U}_{\theta}^\top)
  .
\end{align*}

Set $Q := \sum_{j=1}^{n-2d} a_j \otimes a_j$ and $\Gamma :=\widehat{P}_{\theta} - Q  
%\widehat{P}_{\theta}  
=  \sum_{j= n-2d+1}^{n} \mu_j \, a_j \otimes a_j$. Next we get
\begin{align}
\label{eq:That_nonasymp_approximation}
  \That &= n\norm{\widehat{U}_{\theta}^\top Q\widehat{V}_{\theta}}_{F}^{2} + 2 n\langle  \widehat{U}_{\theta}^\top \Gamma \widehat{V}_{\theta},  \widehat{U}_{\theta}^\top Q \widehat{V}_{\theta}  \rangle_{HS} + n\norm{\widehat{U}_{\theta}^\top \Gamma \widehat{V}_{\theta}}_{F}^{2}
.
\end{align}

The next result guarantees that $\That$ and $n\norm{\widehat{U}_{\theta}^\top Q\widehat{V}_{\theta}}_{F}^{2}$ have the same limiting distribution.

\begin{lemma}
\label{lem:approxTnull}
Set $Q := \sum_{j=1}^{n-2d} a_j \otimes a_j$. Assume the conditional independence assumption $x\perp \ddy\mid z$. Assume that $d = o(\sqrt{n})$ and $\sqrt{d}\,\mathcal{E}_{\theta_m}\rightarrow 0 $ as $m,n\rightarrow \infty$. Assume that $\sqrt{n} \norm{\widehat{U}_{\theta}^\top Q \widehat{V}_{\theta}}_F$ converges in distribution. Then we have:
$$
\That =n\, \norm{\widehat{U}_{\theta}^\top Q \widehat{V}_{\theta}}_F^2 + o_{\mathbb{P}}\left( 1 \right).
$$
\end{lemma}

\begin{proof}

Equation \eqref{eq:final_prob_bound_cond_indep} in Lemma \ref{lem:bound_proba_convergence} combined with \eqref{eq:eigenvaluesmujsmall} guarantee that (with $r=2d$ and $\delta = n^{-1}$)
\begin{align*}
\mathbb{P}\left( n\,\norm{\widehat{U}_{\theta}^\top \Gamma \widehat{V}_{\theta}}_F^2 \geq \epsilon \right) &\lesssim \frac{1}{\epsilon}\left[ K^4 d  \left( \psi_n(n^{-1}) + \mathcal{E}_m \right)^2  \left( \frac{d^2}{n} + \frac{\log^2(n)}{n} \right)\right] + \frac{1}{n}.
%\mathbb{P}\left(  \sqrt{n}\,\norm{\widehat{U}_{\theta}^\top \Gamma \widehat{V}_{\theta}}_F \geq \epsilon \right) &  \leq  (1+\mathcal{E}_m)^2 \frac{d^3}{n\,\epsilon^2} \left( \psi_n(\delta) + \mathcal{E}_m \right)^2 = (1+\mathcal{E}_m)^2 \frac{d^2}{n\,\epsilon^2} \left( \sqrt{d}\,\psi_n(\delta) + \sqrt{d}\,\mathcal{E}_m \right)^2.
\end{align*}

Under the assumption $d=o(\sqrt{n})$, we note that $\sqrt{d}\, \psi_n(n^{-1}) \rightarrow 0$ in view of \eqref{eq:boundpsindelta}. Furthermore assuming that $\sqrt{d}\, \mathcal{E}_{\theta_{m}} \rightarrow 0$ as $m,n \rightarrow \infty$, the previous display guarantees that 
$$
n\,\norm{\widehat{U}_{\theta}^\top \Gamma \widehat{V}_{\theta}}_F^2 \stackrel{\mathbb{P}}{\rightarrow} 0,\quad \text{as $m,n \rightarrow \infty$.}
$$

Next the Cauchy-Schwarz inequality gives
$$
n\vert \langle  \widehat{U}_{\theta}^\top \Gamma \widehat{V}_{\theta},  \widehat{U}_{\theta}^\top Q \widehat{V}_{\theta}  \rangle_{HS} \vert \leq \sqrt{n} \norm{\widehat{U}_{\theta}^\top \Gamma \widehat{V}_{\theta}}_F \, \cdot \,  \sqrt{n} \norm{\widehat{U}_{\theta}^\top Q \widehat{V}_{\theta}}_F  
$$
Assume that $\sqrt{n} \norm{\widehat{U}_{\theta}^\top Q \widehat{V}_{\theta}}_F$ converges in distribution. Set $a_n := \sqrt{n} \norm{\widehat{U}_{\theta}^\top \Gamma \widehat{V}_{\theta}}_F$ and $b_n := \sqrt{n} \norm{\widehat{U}_{\theta}^\top Q \widehat{V}_{\theta}}_F$. By the previous analysis, $a_n \stackrel{\mathbb{P}}{\to} 0$. By Slutsky's theorem, $a_n b_n \stackrel{d}{\to} 0$, and since the limit is constant, $a_n b_n \stackrel{\mathbb{P}}{\to} 0$.

\end{proof}

\subsection{Diagonal Concentration of the Projection Matrix $Q$}

We end this subsection with a concentration result on $Q$. 

Define
\begin{equation}
\label{eq:cauchyresidue}
    \epsilon_{n}(i):=\sum_{j=n-2d+1}^n a_{i,j}^2\quad \forall i \in [n].% \rightarrow 0 \quad \text{as $n\rightarrow \infty$.}%\leq \epsilon,\quad \text{and consequently}\quad [Q]_{i,i} \geq 1-\epsilon.
\end{equation}

\begin{lemma}
\label{lem:epsilon_bound}
Assume that $\mathcal{E}_m + \psi_n(2\delta) < 1$. Then, with probability at least $1-\delta$, we have
\begin{align}
\label{eq:maxepsiloni_whp}
\max_{i\in [n]}\{\epsilon_{n}(i)\} \leq \frac{c^2 K^2}{\sqrt{1-\mathcal{E}_m - \psi_n(2\delta)}} \frac{d^2}{n} \, \log(4nd\delta^{-1}).
\end{align}
\end{lemma}

\begin{proof}
By definition of $\widehat{P}_{\theta}$ and \eqref{eq:Ptheta_def}, we deduce that the SVD of $\widehat{W}_{\theta}$ is
$$
\widehat{W}_{\theta} = \sum_{j=1}^{2d} \sqrt{\widehat{\lambda}_j} \, a_{n-2d+j}\otimes b_j,
$$
where $\{a_{n-2d+j}\}_j$ and $\{b_j\}_j$ are the empirical left and right singular vectors. This yields
$$
\sqrt{\widehat{\lambda}_j }\, a_{n-2d+j} = \widehat{W}_{\theta} b_j = \frac{1}{\sqrt{n}} \left(\begin{array}{c}
     w_\theta(z_1)^\top  \\
     \vdots\\
    w_\theta(z_n)^\top
\end{array}\right)b_j.
$$

By the Cauchy-Schwarz inequality and concentration of $w_\theta(z_i)$, with probability at least $1-\delta/2$, for any $i\in [n]$ and $j\in [2d]$,
$$
\sqrt{\widehat{\lambda}_j }\, |a_{i,n-2d+j}| \leq \frac{|\langle w_\theta(z_i), b_j \rangle |}{\sqrt{n}} \leq \frac{\norm{w_\theta(z_i)}}{\sqrt{n}} \leq c K \sqrt{\frac{d}{n} \, \log(4 n d \delta^{-1})}.
$$

By \eqref{eq:eigenvalueWTW}, we have w.p.a.l $1-\delta/2$ that $\widehat{\lambda}_j \geq 1 - \mathcal{E}_m - \psi_n(2\delta)$ for any $j \in [2d]$. Under the assumption that $\mathcal{E}_m + \psi_n(2\delta) < 1$, we obtain $\sqrt{\widehat{\lambda}_j} \geq \sqrt{1-\mathcal{E}_m - \psi_n(2\delta)}$. 

Therefore, with probability at least $1-\delta$, for any $i\in [n]$ and $j\in [2d]$,
$$
|a_{i,n-2d+j}| \leq \frac{c K}{\sqrt{1-\mathcal{E}_m - \psi_n(2\delta)}} \sqrt{\frac{d}{n} \, \log(4nd\delta^{-1})}.
$$

A union bound over gives, with probability at least $1-\delta$,
$$
\max_{i\in [n],j\in [2d]} |a_{i,n-2d+j}|  \leq \frac{c K}{\sqrt{1-\mathcal{E}_m - \psi_n(2\delta)}} \sqrt{\frac{d}{n} \, \log(4nd\delta^{-1})}.
$$

Recall $\epsilon_{n}(i) = \sum_{j=1}^{2d} a_{i,n-2d+j}^2$. For the previous display, we obtain w.p.a.l. $1-\delta$:
$$
\max_{i\in [n]}\{\epsilon_{n}(i)\} \leq 2 d \cdot \max_{i\in [n], j\in [d]} a_{i,n-2d+j}^2 \leq \frac{c^2 K^2}{1-\mathcal{E}_m - \psi_n(2\delta)} \frac{d^2}{n} \, \log(4nd\delta^{-1}),
$$
which completes the proof.
\end{proof}

\paragraph{Asymptotic behavior.} 
We now discuss the asymptotic implications of Lemma~\ref{lem:epsilon_bound}. Assume that $\mathcal{E}_m + \psi_n(2n^{-1}) \leq 1/2$ and that $d^2 \log(n)/n \to 0$ as $n\to \infty$. Under these conditions, the bound in Lemma \ref{lem:epsilon_bound} implies
$$
\max_{i\in [n]}\{\epsilon_{n}(i)\} = O\left(\frac{d^2 \log(n)}{n}\right) \stackrel{\mathbb{P}}{\to} 0, \quad \text{as } n\to \infty.
$$

Recalling that $[Q]_{i,i} = \sum_{j=1}^{n-2d} a_{i,j}^2 = 1 - \epsilon_{n}(i)$ for any $i\in [n]$, we immediately obtain
\begin{align}
\label{eq:Q_diag_prob_1}
    \min_{i\in [n]}\{[Q]_{i,i}\} = 1 - \max_{i\in [n]}\{\epsilon_{n}(i)\} \stackrel{\mathbb{P}}{\to} 1, \quad \text{as } n\to \infty.
\end{align}

This establishes that the diagonal entries of $Q$ converge to $1$ in probability, indicating that asymptotically, the projection matrix $Q$ becomes increasingly close to the identity. This asymptotic property is crucial for establishing consistency results in the regime where $d^2 = o(n/\log n)$.

\subsection{Proof of Theorem \ref{thm:typeIerror}}
\label{app:prooftypeIerror}
\begin{proof}[Proof of Theorem \ref{thm:typeIerror}]

Throughout the proof, $c>0$ will denote an absolute constant, which may vary from one occurrence to another.

By Lemma \ref{lem:approxTnull}, we have
$$
\That =n\, \norm{\widehat{U}_{\theta}^\top Q \widehat{V}_{\theta}}_F^2 + o_{\mathbb{P}}\left( 1 \right).
$$

\paragraph{Asymptotic distribution of $n \, \norm{\widehat{U}_{\theta}^\top Q\widehat{V}_{\theta}}_{F}^{2}$.}
%\Xi_n

Set $a_j = (a_{i,j})_{i\in n}\in \mathbb{R}^n$. By definition of $Q$, we have $[Q]_{i_1,i_2}=  \sum_{j=1}^{n-2d} a_{i_1,j} a_{i_2,j}$. We consider
\begin{align}
   M_n:= \sqrt{n}\widehat{U}_{\theta}^\top Q \widehat{V}_{\theta} &= \frac{1}{\sqrt{n}}\sum_{j=1}^{n-2d}\sum_{i_1,i_2=1}^{n} a_{i_1,j} a_{i_2,j}   u_{\theta}(x_{i_1}) \otimes v_{\theta}(\ddy_{i_2})\notag\\
    &= \frac{1}{\sqrt{n}}\sum_{i_1,i_2=1}^{n} [Q]_{i_1,i_2}   u_{\theta}(x_{i_1}) \otimes v_{\theta}(\ddy_{i_2})\notag\\
    &= \frac{1}{\sqrt{n}}\sum_{i=1}^{n} [Q]_{i,i}   u_{\theta}(x_{i}) \otimes v_{\theta}(\ddy_{i}) + \frac{1}{\sqrt{n}}\sum_{i_1,i_2=1, i_1\neq i_2 }^{n} [Q]_{i_1,i_2}   u_{\theta}(x_{i_1}) \otimes v_{\theta}(\ddy_{i_2}).
\end{align}

We prove that the second off-diagonal sum is negligible in front of the diagonal sum in the previous display. Indeed computing the variance of the off-diagonal term, using the tower property of conditional expectation and the conditional independence of $x_i$ and $\ddy_i$ given $z_i$ under the null hypothesis, and the fact that $\{x_{i},\ddy_i,z_i\}_{i=1}^{n}$ is an i.i.d. sample set, we get:
\begin{align*}
    &\mathrm{Cov} \left(   \sum_{i_1\neq i_2 } [Q]_{i_1,i_2} u_{\theta}(x_{i_1}) \otimes v_{\theta}(\ddy_{i_2})  \right)\\
    &\hspace{1cm}= \sum_{i_1\neq i_2 }\sum_{i_3\neq i_4 } [Q]_{i_1,i_2}  [Q]_{i_3,i_4}  \mathrm{Cov}\left(  u_{\theta}(x_{i_1}) \otimes v_{\theta}(\ddy_{i_2}), u_{\theta}(x_{i_3}) \otimes v_{\theta}(\ddy_{i_4})\right)\\
    &\hspace{1cm}= \sum_{i_1\neq i_2 }\sum_{i_3\neq i_4 } [Q]_{i_1,i_2}  [Q]_{i_3,i_4}  \Big[\mathbb{E}\left[  u_{\theta}(x_{i_1}) \otimes v_{\theta}(\ddy_{i_2}) \otimes u_{\theta}(x_{i_3}) \otimes v_{\theta}(\ddy_{i_4})\right]\\
    &\hspace{3cm} - \mathbb{E}[u_{\theta}(x_{i_1}) \otimes v_{\theta}(\ddy_{i_2})] \otimes \mathbb{E}[u_{\theta}(x_{i_3}) \otimes v_{\theta}(\ddy_{i_4})]\Big].
\end{align*}

Using the conditional independence and i.i.d. structure, $\mathrm{Cov}\left(  u_{\theta}(x_{i_1}) \otimes v_{\theta}(\ddy_{i_2}), u_{\theta}(x_{i_3}) \otimes v_{\theta}(\ddy_{i_4})\right)$ is non-zero only when $(i_1, i_2) = (i_3, i_4)$, giving:
\begin{align*}
    &\mathrm{Cov} \left(   \sum_{i_1\neq i_2 } [Q]_{i_1,i_2} u_{\theta}(x_{i_1}) \otimes v_{\theta}(\ddy_{i_2})  \right) =  \left( \sum_{i_1\neq i_2 } [Q]_{i_1,i_2}^2  \right)\mathrm{Cov}(u_{\theta}(x)) \otimes \mathrm{Cov}(v_{\theta}(\ddy)) .
\end{align*}

Recall next that that $Q$ is an orthogonal projection of rank $n-2d$, we have $\norm{Q}_F^2 = \mathrm{tr}(Q) = n-2d$. Hence:
\begin{align}
\label{eq:Covdiagtermdominated}
    \mathrm{Cov} \left(   \sum_{i_1\neq i_2 } [Q]_{i_1,i_2} u_{\theta}(x_{i_1}) \otimes v_{\theta}(\ddy_{i_2})  \right) &\preceq \left( \sum_{i_1\neq i_2 } [Q]_{i_1,i_2}^2  \right)  \mathrm{Cov}(u_{\theta}(x)) \otimes \mathrm{Cov}(v_{\theta}(\ddy)) \notag\\
    &= \left( \norm{Q}_F^2 - \sum_{i=1}^n [Q]_{i,i}^2 \right) \mathrm{Cov}(u_{\theta}(x)) \otimes \mathrm{Cov}(v_{\theta}(\ddy))\notag\\
    &\preceq \left( (n-2d) - \min_{i\in [n]}\{Q_{i,i}\}\,\mathrm{tr}(Q) \right) 
    \,\mathrm{Cov}(u_{\theta}(x)) \otimes \mathrm{Cov}(v_{\theta}(\ddy))\notag\\
    &=  \left( 1 - \min_{i\in [n]}\{Q_{i,i}\} \right) \, (n-2d)
    \, \mathrm{Cov}(u_{\theta}(x)) \otimes \mathrm{Cov}(v_{\theta}(\ddy)).
\end{align}

Similarly, for the diagonal term:
\begin{align}
\label{eq:Covdiagtermdominating}
    \mathrm{Cov} \left( \sum_{i=1}^{n} [Q]_{i,i}   u_{\theta}(x_{i}) \otimes v_{\theta}(\ddy_{i})  \right) &= \sum_{i=1}^n [Q]_{i,i}^2 \mathrm{Cov} \left(  u_{\theta}(x_{i}) \otimes v_{\theta}(\ddy_{i}) \right) \notag\\
    &=\left(   \sum_{i=1}^n [Q]_{i,i}^2 \right) \, \mathrm{Cov}(u_{\theta}(x)) \otimes \mathrm{Cov}(v_{\theta}(\ddy))  \notag\\
    &\succeq \min_{i\in [n]}\{Q_{i,i}\}\,\mathrm{tr}(Q)\,  \mathrm{Cov}(u_{\theta}(x)) \otimes \mathrm{Cov}(v_{\theta}(\ddy))  \notag\\
    &=\min_{i\in [n]}\{Q_{i,i}\}\,(n-2d)\,  \mathrm{Cov}(u_{\theta}(x)) \otimes \mathrm{Cov}(v_{\theta}(\ddy)) .
\end{align}

Recall that $[Q]_{i,i}  = \sum_{j=1}^{n-2d} a_{i,j}^2= 1- \epsilon_{n}(i)$ for any $i\in [n]$. We proved in Lemma \ref{lem:epsilon_bound} and in \eqref{eq:Q_diag_prob_1} that %on the event of p.a.l. $1-\delta$ 
that
$$
\min_{i}\{[Q]_{i,i}\}  \stackrel{\mathbb{P}}{\to} 1, \quad \text{as $n\to \infty$.}
$$
Thanks to the previous property, we just proved that the covariance in \eqref{eq:Covdiagtermdominated} is completely dominated and negligible in front of the covariance in \eqref{eq:Covdiagtermdominating}.

To prove convergence in probability, we analyze both the mean and variance of the rescaled off-diagonal term.

\textbf{Variance analysis:} 
The variance of the rescaled term satisfies:
\begin{align*}
    \mathrm{Cov}\left(\frac{1}{\sqrt{n}}\sum_{i_1\neq i_2 } [Q]_{i_1,i_2} u_{\theta}(x_{i_1}) \otimes v_{\theta}(\ddy_{i_2})\right) 
    &= \frac{1}{n}\mathrm{Cov} \left(   \sum_{i_1\neq i_2 } [Q]_{i_1,i_2} u_{\theta}(x_{i_1}) \otimes v_{\theta}(\ddy_{i_2})  \right)\\
    &\preceq \frac{1}{n} \left( 1 - \min_{i\in [n]}\{Q_{i,i}\} \right) (n-2d) \, \mathrm{Cov}(u_{\theta}(x)) \otimes \mathrm{Cov}(v_{\theta}(\ddy))\\
    &= \frac{n-2d}{n} \left( 1 - \min_{i\in [n]}\{Q_{i,i}\} \right) \mathrm{Cov}(u_{\theta}(x)) \otimes \mathrm{Cov}(v_{\theta}(\ddy)).
\end{align*}

Since $\min_{i\in [n]}\{Q_{i,i}\} \stackrel{\mathbb{P}}{\to} 1$ as $n\to \infty$, and $\frac{n-2d}{n} \to 1$ under the regime $d = o(n)$, we have
$$
\frac{n-2d}{n} \left( 1 - \min_{i\in [n]}\{Q_{i,i}\} \right) \stackrel{\mathbb{P}}{\to} 0.
$$

Therefore, the covariance converges to zero in probability.

\textbf{Mean analysis:} 
For the mean, using conditional independence for $i_1 \neq i_2$:
\begin{align*}
    \mathbb{E}\left[\frac{1}{\sqrt{n}}\sum_{i_1\neq i_2 } [Q]_{i_1,i_2} u_{\theta}(x_{i_1}) \otimes v_{\theta}(\ddy_{i_2})\right] 
    &= \frac{1}{\sqrt{n}}\sum_{i_1\neq i_2 } [Q]_{i_1,i_2} \mathbb{E}[u_{\theta}(x)] \otimes \mathbb{E}[v_{\theta}(\ddy)].
\end{align*}

By the empirical centering property \eqref{eq:Expectation_utheta}, we have w.p.a.l. $1-\delta$:
$$
\|\mathbb{E}[u_{\theta}(x)]\| \bigvee \|\mathbb{E}[v_{\theta}(\ddy)]\| \leq C K \left( \sqrt{\frac{d}{n}} +  \sqrt{\frac{\log(\delta^{-1})}{n}} \right).
$$

Using $\|Q\|_1 \leq \|Q\|_F \sqrt{n} = \sqrt{(n-2d)n}$, we obtain:
\begin{align*}
    \left\|\mathbb{E}\left[\frac{1}{\sqrt{n}}\sum_{i_1\neq i_2 } [Q]_{i_1,i_2} u_{\theta}(x_{i_1}) \otimes v_{\theta}(\ddy_{i_2})\right]\right\|
    &\leq \frac{1}{\sqrt{n}} \|Q\|_1 \|\mathbb{E}[u_{\theta}(x)]\| \|\mathbb{E}[v_{\theta}(\ddy)]\|\\
    &\leq \sqrt{n-2d} \cdot C^2 K^2 \left( \sqrt{\frac{2d}{n}} +  \sqrt{\frac{\log(\delta^{-1})}{n}} \right)^2\\
    &= O\left(\sqrt{n-2d} \cdot \frac{2d}{n}\right) = O\left(\frac{d}{\sqrt{n}}\right) \to 0,
\end{align*}
as $n\to \infty$ under the regime $d = o(\sqrt{n})$.

\textbf{Conclusion:} 
Since both the mean and variance of the rescaled off-diagonal term converge to zero, we conclude:
$$
\frac{1}{\sqrt{n}}\sum_{i_1,i_2=1, i_1\neq i_2 }^{n} [Q]_{i_1,i_2}   u_{\theta}(x_{i_1}) \otimes v_{\theta}(\ddy_{i_2}) \stackrel{\mathbb{P}}{\rightarrow} 0,\quad \text{as $n\to \infty$}.
$$

We consider now the diagonal term:
$$
\frac{1}{\sqrt{n}}\sum_{i=1}^{n} [Q]_{i,i}   u_{\theta}(x_{i}) \otimes v_{\theta}(\ddy_{i}).
$$

We apply the Lindeberg-Feller CLT to get the following result. See Appendix \ref{sec:prooflemCLT} for the proof.

\begin{lemma}
\label{lem:CLT}
    Assume that there exists a large enough absolute constant $c>0$ such that $n\geq c \, K^4 \, d^2 \log n$ as $n\to \infty$. Then
    \[
\frac{1}{\sqrt{n}} \sum_{i=1}^n [Q]_{i,i} \, u_{\theta}(x_i) \otimes v_{\theta}(\ddy_i) \xrightarrow{\mathcal{D}} \mathcal{Z} \in \mathbb{R}^{d \times d}, \quad \text{where } \mathcal{Z} \sim \mathcal{N}(0, \Sigma),
\]
with $
\Sigma 
=  \mathrm{Cov}(u_{\theta}(x)) \otimes \mathrm{Cov}(v_{\theta}(\ddy))$.
\end{lemma}

Recall that $\{\lambda_j(u_\theta)\}_{j\in [d]}$ and $\{\lambda_j(v_\theta)\}_{j\in [d]}$ are the eigenvalues of $\mathrm{Cov}(u_{\theta}(x_1))$ and $\mathrm{Cov}(v_{\theta}(\ddy_1))$ respectively.

Slutsky's theorem implies the following convergence in distribution:
\begin{align}\label{eq:CLT-0}
    n \, \norm{\widehat{U}_{\theta}^\top Q\widehat{V}_{\theta}}_{F}^{2} \stackrel{\mathcal{D}}{\rightarrow}   \sum_{i=1}^d \sum_{j=1}^d \lambda_i(u_\theta)\lambda_j(v_\theta) \chi^2_{i,j}(1) =: \xi,
\end{align}
as $n\rightarrow \infty$ where the $\chi^2_{i,j}(1)$ are independent random variables following a chi-square distribution of degree $1$.

Noting that the $\chi^2_{i,j}(1)$ are nonnegative for any $i,j$ and assuming that $(2+ \mathcal{E}_m)\, \mathcal{E}_m<1$, we get for any $t\in \mathbb{R}$ that
\begin{align}
\label{eq:asymptchi2equiv-1}
\mathbb{P}\left( \chi^2(d^2) \leq \frac{t}{1+ (2+ \mathcal{E}_m)\, \mathcal{E}_m} \right) \leq    \mathbb{P}\left( \xi \leq t \right) \leq \mathbb{P}\left( \chi^2(d^2) \leq \frac{t}{1- (2+ \mathcal{E}_m)\, \mathcal{E}_m}\right),
\end{align}
where $\chi^2(d^2)$ is a chi-square distribution with $d^2$ degrees of freedom.

Note that the $\chi^2(d^2)$ distribution admits the following density w.r.t the Lebesgue measure
$$
f(t) = \frac{1}{2^{d^2/2}\Gamma(d^2/2)}t^{d^2/2-1} e^{-t/2}.
$$

Since we assumed that $\epsilon:=(2+ \mathcal{E}_m)\, \mathcal{E}_m<1$, we have
$$
\frac{t}{1-\epsilon} = t +\frac{\epsilon\, t}{1-\epsilon} = t\, (1 + \diamondsuit),\quad \text{with}\quad \diamondsuit = \frac{\epsilon}{1-\epsilon}.
$$
Consequently for any $t>0$
\begin{align*}
\mathbb{P}\left( \chi^2(d^2) \leq t\, (1 + \diamondsuit) \right)  = \mathbb{P}\left( \chi^2(d^2) \leq t \right)  + \int_{t}^{t(1+\diamondsuit)} f(s) ds 
\end{align*}
We proceed similarly for the lower bound. Indeed we first note that
$
\frac{t}{1+\epsilon} \geq t (1-\epsilon)
$. Hence, 
\begin{align*}
\mathbb{P}\left( \chi^2(d^2) \leq t\, (1 - \epsilon) \right)  = \mathbb{P}\left( \chi^2(d^2) \leq t \right)  -\int_{t(1-\epsilon)}^{t} f(s) ds .
\end{align*}

Combining the last two displays with \eqref{eq:asymptchi2equiv-1}, we get for any $t>0$ that
\begin{align}
\label{eq:asymptchi2equiv-2}
\left| \mathbb{P}\left( \xi \leq t \right) - \mathbb{P}\left( \chi^2(d^2) \leq t \right) \right| \leq \max \left\lbrace \int_{t(1-\epsilon)}^{t} f(s) ds , \int_{t}^{t(1+\diamondsuit)} f(s) ds \right\rbrace 
\end{align}

Since $\mathcal{E}_m = \mathcal{E}_{\theta_m} \rightarrow 0$ as $m\rightarrow \infty$ and the density $f$ is uniformly bounded on $\mathbb{R}^{+}$, the right-hand-side in \eqref{eq:asymptchi2equiv-2} goes to $0$ as $m\rightarrow \infty$ for any fixed $t>0$. This concludes the proof.

\end{proof}

\subsection{Proof of Lemma \ref{lem:CLT}}
\label{sec:prooflemCLT}

Define
\[
Z_n := \frac{1}{\sqrt{n}} \sum_{i=1}^n [Q]_{i,i} \, u_{\theta}(x_i) \otimes v_{\theta}(\ddy_i),
\]

\begin{proof}[Proof of Lemma \ref{lem:CLT}]
Define
\[
Z_n := \frac{1}{\sqrt{n}} \sum_{i=1}^n [Q]_{i,i} \, u_{\theta}(x_i) \otimes v_{\theta}(\ddy_i),
\]
and consider the vectorized form:
\[
\mathrm{vec}(Z_n) = \frac{1}{\sqrt{n}} \sum_{i=1}^n [Q]_{i,i} \cdot \mathrm{vec}\left( u_{\theta}(x_i) \otimes v_{\theta}(\ddy_i) \right) \in \mathbb{R}^{d^2}.
\]

Let us denote
\[
w_i := \mathrm{vec}\left( u_{\theta}(x_i) \otimes v_{\theta}(\ddy_i) \right), \quad \text{so that} \quad \mathrm{vec}(Z_n) = \frac{1}{\sqrt{n}} \sum_{i=1}^n [Q]_{i,i} w_i.
\]

We now verify the Lindeberg condition:
\[
\frac{1}{n} \sum_{i=1}^n \mathbb{E} \left[ \norm{[Q]_{i,i} w_i}^2 \cdot \mathbbm{1}_{\{ \norm{[Q]_{i,i} w_i} > \varepsilon \sqrt{n} \}} \right] \to 0 \quad \text{for all } \varepsilon > 0.
\]

Since \( \norm{[Q]_{i,i} w_i} = [Q]_{i,i} \cdot \norm{w_i} \), and \( [Q]_{i,i} \leq 1 \), it suffices to show:
\[
\frac{1}{n} \sum_{i=1}^n \mathbb{E} \left[ \norm{w_i}^2 \cdot \mathbbm{1}_{\{ \norm{w_i} > \varepsilon \sqrt{n} \}} \right] \to 0.
\]

Note that:
\[
\norm{w_i} = \norm{ \mathrm{vec}( u_{\theta}(x_i) \otimes v_{\theta}(\ddy_i) ) } = \norm{ u_{\theta}(x_i) } \cdot \norm{ v_{\theta}(\ddy_i) }.
\]

Let us define:
\[
A_i := \norm{ u_{\theta}(x_i) }, \quad B_i := \norm{ v_{\theta}(\ddy_i) }, \quad \text{so that } \norm{w_i} = A_i B_i.
\]

Since \( u_{\theta}(x_i) \) and \( v_{\theta}(\ddy_i) \) are \(K\)--sub-Gaussian, their norms \( A_i \), \( B_i \) are $K\sqrt{d}$--sub-Gaussian random variables. Therefore, their product \( A_i B_i \) is $2K^2 \, d$--sub-exponential.

As a result, for all \( t > 0 \),
\[
\mathbb{P}( A_i B_i > t ) \leq c e^{-c' \frac{t}{K^2 \, d}},
\]
for some absolute constants \( c, c' > 0 \). 

Therefore, elementary calculus (based on integration by parts for Gamma function) gives the following upper bound
\[
\mathbb{E} \left[ (A_i B_i)^2 \cdot \mathbbm{1}_{ \{ A_i B_i > \varepsilon \sqrt{n} \} } \right] \leq c K^2 \, d \, n \,\epsilon^2\, e^{-c'\frac{\epsilon \sqrt{n}}{K^2 \, d}},
\]
for absolute constants \( c, c' > 0 \).

Thus,
\[
\frac{1}{n} \sum_{i=1}^n \mathbb{E} \left[ \norm{w_i}^2 \cdot \mathbbm{1}_{ \{ \norm{w_i} > \varepsilon \sqrt{n} \} } \right]  =  \mathbb{E} \left[ \norm{w_1}^2 \cdot \mathbbm{1}_{ \{ \norm{w_1} > \varepsilon \sqrt{n} \} } \right]\leq c K^2 \, d \,n \,\epsilon^2\, e^{-c'\frac{\epsilon \sqrt{n}}{K^2 \, d}} \to 0,
\]
provided that $n\geq \frac{3}{2 \epsilon^2\, c'} \, K^4 \, d^2 \log n$ as $n\to \infty$ so that the exponential term is dominating and thus leading the right-hand-side to $0$.

Since the Lindeberg condition is satisfied, by the multivariate CLT, the sequence \( \mathrm{vec}(Z_n) \) converges in distribution to a multivariate Gaussian random vector. Therefore,
\[
Z_n \xrightarrow{d} \mathcal{Z} \in \mathbb{R}^{d \times d}, \quad \text{where } \mathrm{vec}(\mathcal{Z}) \sim \mathcal{N}(0, \Sigma),
\]
with
\[
\Sigma 
=  \mathrm{Cov}(u_{\theta}(x)) \otimes \mathrm{Cov}(v_{\theta}(\ddy)).
\]

\end{proof}

\subsection{Proof of Theorem \ref{thm:power}}
\label{app:proofpower}

Recall that 
$$
[\![\mathsf{\Sigma}_{X\ddot{Y}\cdot Z}]\!]_d = C_{UV} -  C_{UW}C_{WV}.
$$
Assuming the representation learning step has been successful, we can assume that \eqref{eq:gap_optim} is valid with small $\mathcal{E}_{m}$ for learned features $\widehat{U}_\theta$, $\widehat{V}_\theta$ and $\widehat{W}_\theta$, meaning that 
$$
 \norm{\,[\![\mathsf{\Sigma}_{X\ddot{Y}\cdot Z}]\!]_d - \big( C_{\widehat{U}_{\theta} \widehat{V}_{\theta}} -  C_{\widehat{U}_{\theta}  \widehat{W}_{\theta} }C_{ \widehat{W}_{\theta} \widehat{V}_{\theta}} \big)}\leq \mathcal{E}_{m}\ll 1.
$$
%From now on, set for brevity  $\widehat{U} = \widehat{U}_\theta$, $\widehat{V} = \widehat{V}_\theta$ and $\widehat{W} =\widehat{W}_\theta$.
We have
\begin{align*}
\covestUVnnest - \covestUWnnest\covestWVnnest &=  [\![\mathsf{\Sigma}_{X\ddot{Y}\cdot Z}]\!]_d + \big(C_{\widehat{U}_{\theta} \widehat{V}_{\theta}} -  C_{\widehat{U}_{\theta}
\widehat{W}_{\theta}
}C_{\widehat{W_{\theta}
}\widehat{V}_{\theta}
} ) - [\![\mathsf{\Sigma}_{X\ddot{Y}\cdot Z}]\!]_d \big)\notag\\
&= [\![\mathsf{\Sigma}_{X\ddot{Y}\cdot Z}]\!]_d +(\covestUVnnest - C_{\widehat{U}_{\theta}
\widehat{V}_{\theta}
})- (\covestUWnnest -C_{\widehat{U}_{\theta}
\widehat{W}_{\theta}
}) \covestWVnnest + C_{\widehat{U}_{\theta}
\widehat{W}_{\theta}
} ( C_{\widehat{W}_{\theta}
\widehat{V}_{\theta}
}-\covestWVnnest).
\end{align*}

% We have
% \begin{align*}
% \covestUVnnest - \covestUWnnest\covestWVnnest &=  [\![\mathsf{\Sigma}_{X\ddot{Y}\cdot Z}]\!]_d + \big(C_{\widehat{U} \widehat{V}} -  C_{\widehat{U}\widehat{W}}C_{\widehat{W}\widehat{V}} ) - [\![\mathsf{\Sigma}_{X\ddot{Y}\cdot Z}]\!]_d \big)\notag\\
% &= [\![\mathsf{\Sigma}_{X\ddot{Y}\cdot Z}]\!]_d +(\covestUVnnest - C_{\widehat{U}\widehat{V}})- (\covestUWnnest -C_{\widehat{U}\widehat{W}}) \covestWVnnest + C_{\widehat{U}\widehat{W}} ( C_{\widehat{W}\widehat{V}}-\covestWVnnest).
% \end{align*}

Taking the Frobenius norm and using the triangular inequality, we get
\begin{align*}
\norm{\covestUVnnest - \covestUWnnest\covestWVnnest}_F&\geq \norm{[\![\mathsf{\Sigma}_{X\ddot{Y}\cdot Z}]\!]_d}_{F}  - \norm{\big(C_{\widehat{U}_{\theta} \widehat{V}_{\theta}} -  C_{\widehat{U}_{\theta}
\widehat{W}_{\theta}
}C_{\widehat{W_{\theta}
}\widehat{V}_{\theta}
} ) - [\![\mathsf{\Sigma}_{X\ddot{Y}\cdot Z}]\!]_d \big) }_F   %- \norm{ [\![\mathsf{\Sigma}_{X\ddot{Y}\cdot Z}]\!]_d - \big( C_{\widehat{U}_{\theta} \widehat{V}_{\theta}} -  C_{\widehat{U}_{\theta}  \widehat{W}_{\theta} }C_{ \widehat{W}_{\theta} \widehat{V}_{\theta}} \big)}_F
\\
&\hspace{1cm}- \norm{\covestUVnnest - C_{\widehat{U}_\theta \widehat{V}_\theta}}_F- \norm{(\covestUWnnest -C_{\widehat{U}_\theta \widehat{W}_\theta}) \covestWVnnest}_{F} - \norm{C_{ \widehat{U}_\theta \widehat{W}_\theta} ( C_{\widehat{W}_\theta \widehat{V}_\theta}-\covestWVnnest}_F\\
&\geq \norm{[\![\mathsf{\Sigma}_{X\ddot{Y}\cdot Z}]\!]_d}_{F}  -   \norm{\big(C_{\widehat{U}_{\theta} \widehat{V}_{\theta}} -  C_{\widehat{U}_{\theta}
\widehat{W}_{\theta}
}C_{\widehat{W_{\theta}
}\widehat{V}_{\theta}
} \big) - [\![\mathsf{\Sigma}_{X\ddot{Y}\cdot Z}]\!]_d  }_F \\
&\hspace{1cm}- \norm{\covestUVnnest - C_{UV}}_F- \norm{(\covestUWnnest -C_{UW})}_F \norm{\covestWVnnest}- \norm{C_{UW}} \norm{( C_{WV}-\covestWVnnest)}_F.
\end{align*}

Using \eqref{eq:gap_optim}, we get
$$
\norm{  [\![\mathsf{\Sigma}_{X\ddot{Y}\cdot Z}]\!]_d   -        \big(     C_{\widehat{U}_\theta\widehat{V}_\theta} -    C_{\widehat{U}_\theta\widehat{U}_\theta}     C_{\widehat{W}_\theta\widehat{V}_\theta}   \big)  }_{F} \leq \sqrt{d} \norm{  [\![\mathsf{\Sigma}_{X\ddot{Y}\cdot Z}]\!]_d   -        \big(     C_{\widehat{U}_\theta\widehat{V}_\theta} -    C_{\widehat{U}_\theta\widehat{U}_\theta}     C_{\widehat{W}_\theta\widehat{V}_\theta}   \big)  } \leq \sqrt{d} \, \mathcal{E}_{m}.
$$

Furthermore, exploiting the sub-Gaussian properties of features (Assumption \ref{ass:subgaussian}), standard matrix concentration inequalities for the cross-variance estimators and an union bound, we obtain w.p.a.l. $1-\delta$
$$
\frac{\norm{ \widehat{C}_{\widehat{U}_\theta\widehat{V}_\theta} -    C_{\widehat{U}_\theta \widehat{V}_\theta} } }{\norm{C_{\widehat{U}_\theta \widehat{V}_\theta} }} \bigvee \frac{\norm{ \widehat{C}_{\widehat{U}_\theta\widehat{W}_\theta} -    C_{\widehat{U}_\theta \widehat{W}_\theta} }}{\norm{C_{\widehat{U}_\theta \widehat{W}_\theta} }} \bigvee \frac{\norm{ \widehat{C}_{\widehat{W}_\theta\widehat{V}_\theta} -    C_{\widehat{W}_\theta \widehat{V}_\theta} }}{\norm{C_{\widehat{W}_\theta \widehat{V}_\theta} }}  \lesssim K^2 \left( \sqrt{\frac{d}{n}}  \bigvee \sqrt{\frac{\log(\delta^{-1})}{n} } \right)=: \psi_n(\delta),
$$
for some large enough absolute constant $C>0$.

Using \eqref{eq:gap_optim} again, we get $   \norm{C_{\widehat{U}_\theta \widehat{V}_\theta} }   \leq  \sqrt{\norm{C_{\widehat{U}_\theta \widehat{U}_\theta}} }\sqrt{\norm{C_{\widehat{V}_\theta \widehat{V}_\theta}}}   \leq 1 + \mathcal{E}_m$. We have the same upper bound for all 3 cross-covariances in operator norm.

Combining the last four displays, we get w.p.a.l. $1-\delta$ 
\begin{align*}
\norm{\covestUVnnest - \covestUWnnest\covestWVnnest}_F&\geq \norm{[\![\mathsf{\Sigma}_{X\ddot{Y}\cdot Z}]\!]_d}_{F}   - \sqrt{d}\left( \mathcal{E}_{m} +  (1+\mathcal{E}_m)\psi_n(\delta)+ 2(1+\mathcal{E}_m)^2\psi_n(\delta)  \right)\\
&\geq \norm{[\![\mathsf{\Sigma}_{X\ddot{Y}\cdot Z}]\!]_d}_{F}   - \sqrt{d}\left( \mathcal{E}_{m} +  3(1+\mathcal{E}_m)^2\psi_n(\delta)  \right).
\end{align*}

We recall that
$$
\norm{[\![\mathsf{\Sigma}_{X\ddot{Y}\cdot Z}]\!]_d}_{F}^2 = \sum_{j=1}^{d} \sigma_j^2,
$$
and $\mathsf{\Sigma}_{X\ddot{Y}\cdot Z}$ is assumed Hilbert-Schmidt. This implies that if we take $d$ large enough, then we capture enough signal, that is $\sum_{j=1}^d \sigma_j^2 \geq (\sum_{j=1}^{\infty} \sigma_j^2)/2 = \norm{ \pcovs{X}{Y}}_F^2/2 = \epsilon_n^2/2$. Hence the smoothness of the operator $\pcovs{X}{Y}$ measured through the decay rate of its eigenvalues provides a theoretical guidelines on the choice of $d$ and the necessary number of samples to obtain guarantees on the power of the test.

%Concretely, since $\DCE$ is Hilbert–Schmidt, for large $d$ we have $\sum_{j=1}^{d} \sigma_j^2 \geq \norm{\DCE}_F^2/2 \geq \epsilon_n/2$. 

Recall that $c_\alpha$ is the quantile of the chi-square distribution $\chi^2_{d^2}$. Hence if
$$
\epsilon_n^2 > 2d \, \mathcal{E}_m^2 +  \frac{ 2C^2 K^2 ( d^2 + d \log(\delta^{-1}))  + 2c_\alpha }{n},
$$
then we get
$$
\mathbb{P}_{\mathcal{H}_1} \left( \That  > c_\alpha \right) \geq 1-\delta.
$$

Using the large-$d^2$ quantile approximation of chi-square distribution $\chi^2_{d^2}$: 
$
c_{\alpha} \approx d^2 + 2d z_{1-\alpha}, 
$
where $z_{1-\alpha}$ is the quantile of the standard normal. Hence the threshold on $\epsilon_n$ becomes :
$$
\epsilon_n^2 >2  d \, \mathcal{E}_m^2 +  \frac{ 2C^2 K^2 ( d^2 + d \log(\delta^{-1}))  + 2d^2 + 4 d z_{1-\alpha} }{n}.
$$

% The above bound combines several conditions. We need to use $d$ large enough to capture enough of the signal strength $\epsilon_n^2 :=\norm{\mathsf{\Sigma}_{X\ddot{Y}\cdot Z}}_{F}^2$ during the representation learning step. But then in order for the previous condition to be non-vacuous, we need the representation learning stage to be successful in the sense that $d \mathcal{E}_m^2 \ll \epsilon_n^2$. Next the signal strength needs to be larger than the critical value of the test divided by $n$ which scales as $d^2/n$. The other terms corresponds to estimation error when we estimate the true signal strength using empricial covariances. Overall if 
% $$
% \epsilon_n^2  d \, \mathcal{E}_m^2 +  C \frac{d^2 + \log(\delta^{-1})}{n}, 
% $$
% then we can guarantee non asymptotic power of level $1-\delta$.

We use the normal approximation for the $\chi^2_{d^2}$ quantile:
\[
c_{\alpha} \approx d^2 + 2d\,z_{1-\alpha},
\]
where $z_{1-\alpha}$ is the $(1-\alpha)$ standard normal quantile. This gives the condition
\[
\epsilon_n^2 > d\,\mathcal{E}_m^2 + \frac{ C^2 K^2 \big( d^2 + d \log(\delta^{-1}) \big) + d^2 + 2d\,z_{1-\alpha} }{n}.
\]

This inequality balances three things:  
\begin{enumerate}
    \item the signal strength $\epsilon_n^2 = \|\Sigma_{X\ddot{Y}\cdot Z}\|_F^2$,  
\item the error from representation learning ($d\,\mathcal{E}_m^2$), and  
\item the statistical error from estimating covariances, which scales like $d^2/n$ (plus logarithmic terms).
\end{enumerate}

For the test to have non-vacuous power, we need the signal to dominate both the representation error and the estimation noise. In particular, if
\[
\epsilon_n^2 \gtrsim d\,\mathcal{E}_m^2 + \frac{d^2 + d\log(\delta^{-1})}{n},
\]
then the test achieves power at least $1-\delta$.

% We conclude that $c_\alpha/\sqrt{n} - \sqrt{n} \sum_{j=1}^{d} \sigma_j^2 \to -\infty$ when 
% $$
% \epsilon_n^2 > 2 \left( \frac{d^2}{n} \vee \frac{d\, z_{1-\alpha}}{n} \right).
% $$
% This implies that 
% \begin{align*}
%    & \lim_{n \to \infty} \mathbb{P}_{\mathcal{H}_1} \left( \That  > c_\alpha \right)= 1.
% \end{align*}

% \input{sections/unused/_simple_model}
\section{Experimental Details}
\label{sec:experimental_details}

{\bf Implementation details.}
To parameterize $u_\theta$, $v_\theta$, and $w_\theta$, we used multilayer perceptrons (MLPs) with \texttt{n\_hidden\_\{u,v,w\}} hidden layers, each containing \texttt{layer\_size\_\{u,v,w\}} neurons. All networks used \texttt{Tanh} activation functions and Xavier initialization. Optimization was performed using Adam with default hyperparameters, except for the learning rate, which is reported in \Cref{tab:hparams_scit}. Training was carried out with varying \texttt{batch\_size} as reported in \Cref{tab:hparams_scit} for $400$ epochs. The inner model was warmed up for $100$ steps. During bi-level optimization, we alternated a single update of the inner model with a single update of the outer model. The dataset was split into 80\% training and 20\% test sets.

{\bf Dimension pruning.}
For added stability, all models were trained using a latent dimension of \texttt{output\_dim}. At test time, we computed a lower-rank test statistic by performing an SVD of the test-statistic matrix and retaining only the leading $\left\lfloor\texttt{perc\_dim\_prune} \times \texttt{output\_dim} \right\rfloor$ singular triplets.
The resulting statistic was evaluated using a corrected $\chi^2$ distribution with degrees of freedom equal to the retained (pruned) dimension.

% \begin{table}[h]
% \centering
% \caption{Model architecture for $u_\theta$.}
% \begin{tabular}{@{}l l@{}}
% \toprule
% \textbf{Layer} & \textbf{Configuration} \\
% \midrule
% 1 & Linear($d_X$, $LS$), Tanh \\
% 2 & Linear($LS$, $LS$) \\
% ... & \\
% L & \\
% \bottomrule
% \end{tabular}
% \label{tab:hparams_scit}
% \end{table}

{\bf Hyperparameter optimization.}
There is no standard protocol for performing hyperparameter optimization of conditional independence tests.
% Hyperparameter tuning for conditional independence testing is rarely discussed in the literature, especially for methods other than those being proposed. 
In this work, we either use established hyperparameter values or follow a shared tuning protocol with equal computational budgets for all compared methods. 
Only hyperparameters reported in the original papers and exposed in the code repositories were varied.
For example, \textbf{DGCIT} does not expose architectural hyperparameters it in its training algorithm, which limits its adaptivity and our evaluation.
We deliberately avoided altering or re-implementing standard code.

Given these considerations, hyperparameters were selected as follows.
For kernel-based methods, we employed a Gaussian kernel, with the length scale selected using the median heuristic, following prior work \citep{bellot2019gcit, shi2021doublegenerative, scetbon2022anasymptotic,Strobl2018}.
For \textbf{NNLSCIT}, we adopted the hyperparameters reported in the original paper for the same data-generating model \citep{li2023k}.
The methods \textbf{SpectralCIT}, \textbf{RCIT}, \textbf{GCIT}, and \textbf{DGCIT} were tuned using Weights \& Biases \citep{wandb} with \texttt{method} set to \texttt{bayes} and a budget of 30 trials over the hyperparameter grids reported in \Cref{tab:hparams_scit,tab:hparams_rcit,tab:hparams_gcit,tab:hparams_dgcit}. 
For \textbf{GCIT} and \textbf{DGCIT}, the listed hyperparameters correspond to those identified as most influential by the authors, while all other hyperparameters were fixed to the values used in the original papers \citep{bellot2019gcit, shi2021doublegenerative}.
\begin{equation*}
S_{\mathrm{val}}(m) = \widehat{\alpha}(m) - \widehat{\pi}(m),
\end{equation*}
where $\widehat{\alpha}$ and $\widehat{\pi}$ denote the empirical type I error and power of a model $m$, respectively.
The sample size was set to $N = 1000$.
For each trial, $\widehat{\alpha}$ and $\widehat{\pi}$ were estimated from 20 repetitions under both $\mathcal{H}_0$ and $\mathcal{H}_1$ for each conditioning dimension $d_Z \in \{50, 100, 150, 200, 250, 300\}$, and the resulting values were averaged across $d_Z$.

\begin{table}[h]
\centering
\caption{Hyperparameter grid for \textbf{SpectralCIT}.}
\begin{tabular}{@{}l l l@{}}
\toprule
\textbf{Hyperparameter} & \textbf{Values} & \textbf{Description} \\
\midrule
\texttt{output\_dim} & $2, 4, 6, 8, 10$ & SVD truncation rank $d$ \\
\texttt{perc\_dim\_prune}  & $\mathrm{Uniform}(0.85, 1)$ & Percentage of dimensions to prune \\
\texttt{n\_hidden\_\{u,v,w\}} & $1, 2, 3, 4$  & Number of hidden layers of $u_\theta, v_\theta, w_\theta$ \\
\texttt{layer\_size\_\{u,v,w\}} & $128, 256, 512$ & Layer size of $u_\theta, v_\theta, w_\theta$\\
% Hidden layers $v_\theta$ & 1, 2, 3, 4 \\
% Layer size $v_\theta$ & 128, 256, 512 \\
% Hidden layers $w_\theta$ & 1, 2, 3, 4 \\
% Layer size $w_\theta$ & 128, 256, 512 \\
% Activation function & Tanh \\
% Batch ($\text{use\_batchnorm}$) & False \\
% \midrule
\texttt{lr\_\{inner,outer\}} & $\mathrm{LogUniform}\!\left(3\!\times\! 10^{-5}, 1\!\times\! 10^{-2}\right)$ & Learning rate for inner and outer models \\
% Inner learning rate & Log-uniform($3\times 10^{-5}$, $1\times 10^{-2}$) \\
\texttt{reg\_str\_\{inner,outer\}} & $\mathrm{LogUniform}(1, 100)$ & Regularization strength for inner and outer models \\
\texttt{batch\_size} & $128, 256, 512, 1024$ & Batch size \\
% Inner regularization strength & Log-uniform(1, 100) \\
\bottomrule
\end{tabular}
\label{tab:hparams_scit}
\end{table}

\begin{table}[h]
\centering
\caption{Hyperparameter grid for \textbf{RCIT}.}
\begin{tabular}{@{}l l l@{}}
\toprule
\textbf{Hyperparameter} & \textbf{Values} & \textbf{Description}\\
\midrule
\texttt{approx} & Lindsay-Pilla-Basak, Bootstrap & Null approximation scheme \\
\texttt{num\_f} & $50, 100, 150, 200$ & Number of Fourier features for $Z$\\
\texttt{num\_f2} & $5, 10, 15, 20$ & Number of Fourier features for $\Ddot{X}, X, Y$ \\
% $\gamma$ &  1E-10\\
\bottomrule
\end{tabular}
\label{tab:hparams_rcit}
\end{table}

\begin{table}[h]
\centering
\caption{Hyperparameter grid for \textbf{GCIT}.}
\begin{tabular}{@{}l l l@{}}
\toprule
\textbf{Hyperparameter} & \textbf{Values} \\
\midrule
\texttt{lamda} & $0, 5, 10, 15, 20$ &  Regularization parameter $\lambda$\\
\texttt{statistic\_type} & MMD, PCC, DC, KS, RDC & Statistic \\
\bottomrule
\end{tabular}
\label{tab:hparams_gcit}
\end{table}

\begin{table}[h]
\centering
\caption{Hyperparameter grid for \textbf{DGCIT}.}
% DGCIT
% number of functions B 30--50, B=30
% pseudo samples and sample splittings L, M not sensible
% number of boostrap samples J=1000
% hidden nodes = 128
% dimension of input noise 10
\begin{tabular}{@{}l l l@{}}
\toprule
\textbf{Hyperparameter} & \textbf{Values} \\
\midrule
\texttt{batch\_size} & 16, 32, 64 & Batch size \\
\texttt{b}  & 30, 35, 40, 45, 50 & Number of transformation functions $B$ \\
\texttt{j} & 500, 750, 1000, 1250 & Number of bootstrap samples $J$ \\
% Number of iterations & 500, 750, 1000, 1250, 1500 \\
\bottomrule
\end{tabular}
\label{tab:hparams_dgcit}
\end{table}

{\bf Computational resources.}
All experiments were conducted on a machine running Ubuntu 22.04.5 LTS with Linux kernel 5.15.0-134-generic. The system was equipped with two Intel Xeon E5-2695 v4 CPUs (36 physical cores total) and 251 GB of system memory. GPU acceleration was provided by seven NVIDIA GeForce GTX 1080 Ti GPUs, each with 11 GB of memory. The software stack consisted of Python 3.12.9, PyTorch 2.6.0, and CUDA 12.4.

{\bf Hyperparameter sensitivity.} %
We evaluate the robustness of \textbf{SpectralCIT} with respect to its main hyperparameters: the truncation dimension $d$, the learning rate $\eta$, and the regularization strength $\gamma$.
For the truncation dimension, we vary $d$ around the nominal choice $d_\ast = 10$ used in the post-nonlinear benchmark (cf. Fig.~\ref{fig:yang_benchmark}). As shown in Fig.~\ref{fig:ablation_d}, performance stabilizes for $d \geq 8$, with well-controlled Type I error and consistently high power across conditioning dimensions.
For optimization, we scale the learning rate as $\eta = \alpha \eta_\ast$, where $\eta_\ast = (\eta_{\ast,\mathrm{inner}}, \eta_{\ast,\mathrm{outer}}) = (3\times 10^{-5}, 2.1\times 10^{-3})$ is the reference value from Fig.~\ref{fig:yang_benchmark}, and $\alpha \in \{0.1, 0.3, 1, 3, 10\}$. Figure~\ref{fig:ablation_lr} shows that Type I error remains controlled for $\alpha \geq 0.3$, with only minor degradation in power at extreme values, indicating limited sensitivity once $\eta$ is sufficiently large.
Similarly, we vary the regularization strength as $\gamma = \alpha \gamma_\ast$, with $\alpha \in \{0.1, 0.3, 1, 10\}$ and $\gamma_\ast = (\gamma_{\ast,\mathrm{inner}}, \gamma_{\ast,\mathrm{outer}}) = (3.3, 1.9)$. As shown in Fig.~\ref{fig:ablation_reg_str}, both Type I error and power remain stable across this range.
Overall, these results indicate that \textbf{SpectralCIT} is robust to moderate perturbations of $d$, $\eta$, and $\gamma$, with a broad range of values yielding comparable performance.%

\begin{figure}[h]
    \centering
    \begin{minipage}[b]{0.48\textwidth}
        \centering
        \includegraphics[width=\textwidth]{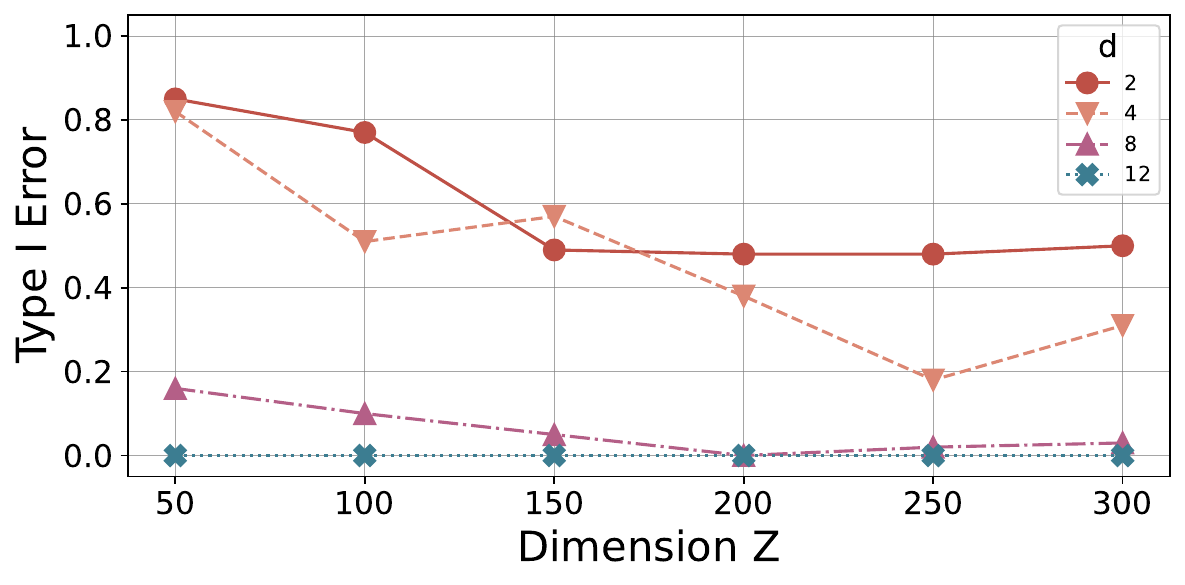}
    \end{minipage}
    \hfill
    \begin{minipage}[b]{0.48\textwidth}
        \centering
        \includegraphics[width=\textwidth]{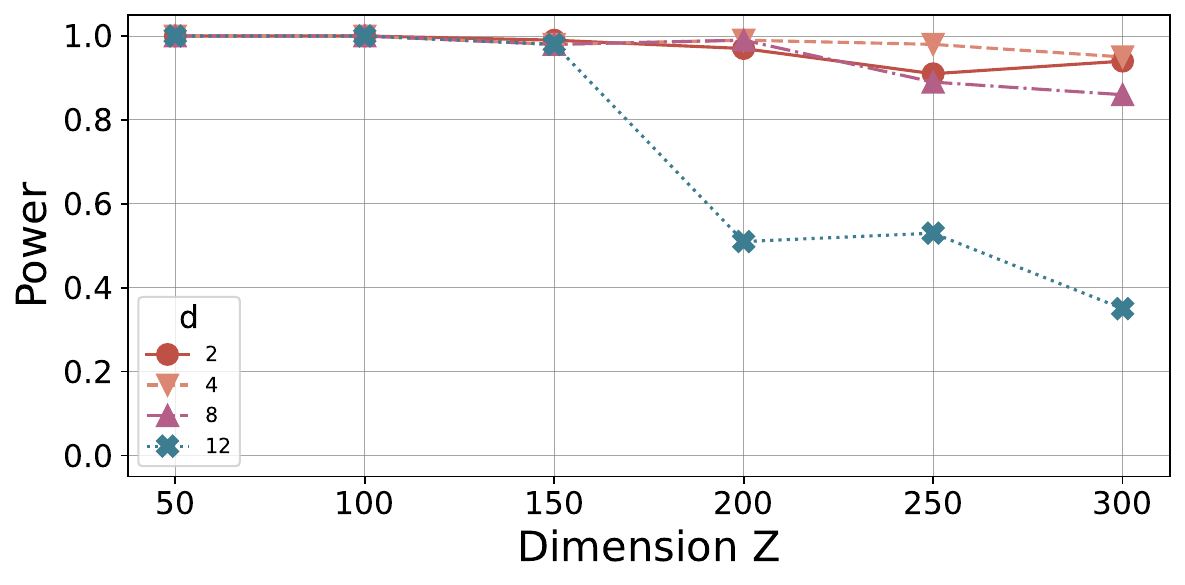}
    \end{minipage}
    \caption{Type I error and power of our method (\textbf{SpectralCIT}) varying truncation dimension $d$ around $d_*=10$.}
    \label{fig:ablation_d}
\end{figure}

\begin{figure}[h]
 \centering
    \begin{minipage}[b]{0.48\textwidth}
        \centering
        \includegraphics[width=\textwidth]{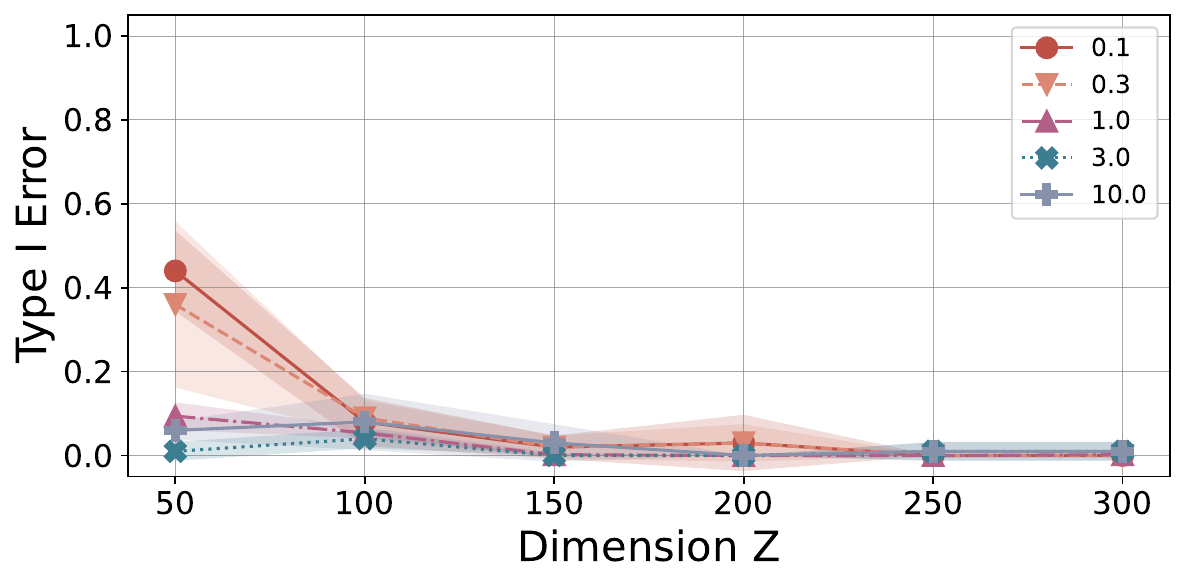}
    \end{minipage}
    \hfill
    \begin{minipage}[b]{0.48\textwidth}
        \centering
        \includegraphics[width=\textwidth]{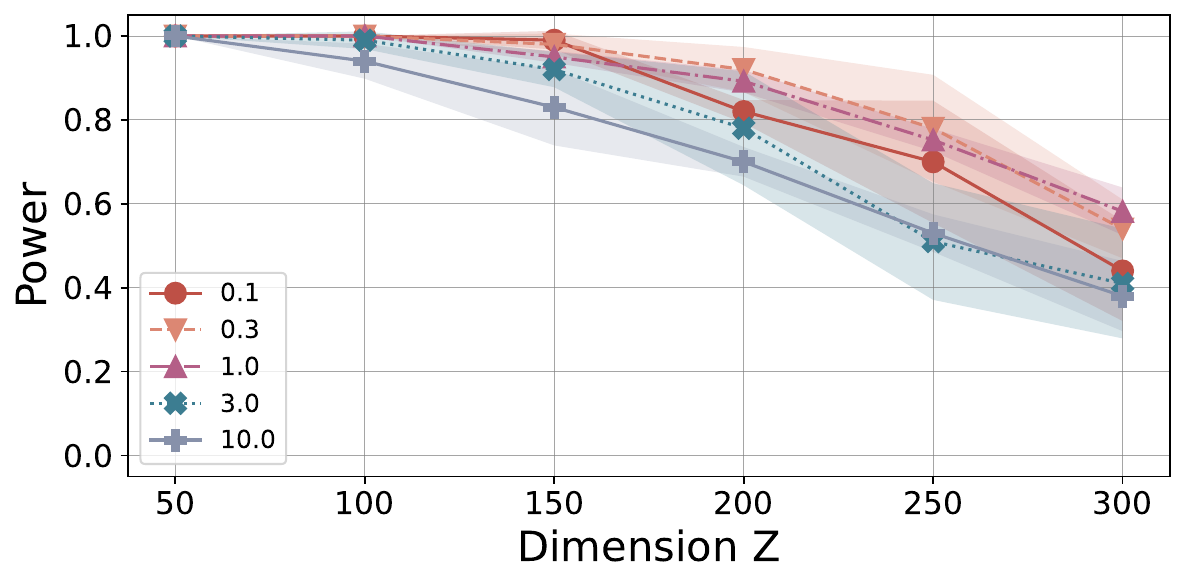}
    \end{minipage}
    \caption{Type I error and power of our method (\textbf{SpectralCIT}) varying the learning rate $\eta$ as $\alpha \eta_*$, where $\alpha \in [0.1, 0.3, 1, 3, 10]$.}
    \label{fig:ablation_lr}
\end{figure}

\begin{figure}[h]
 \centering
    \begin{minipage}[b]{0.48\textwidth}
        \centering
        \includegraphics[width=\textwidth]{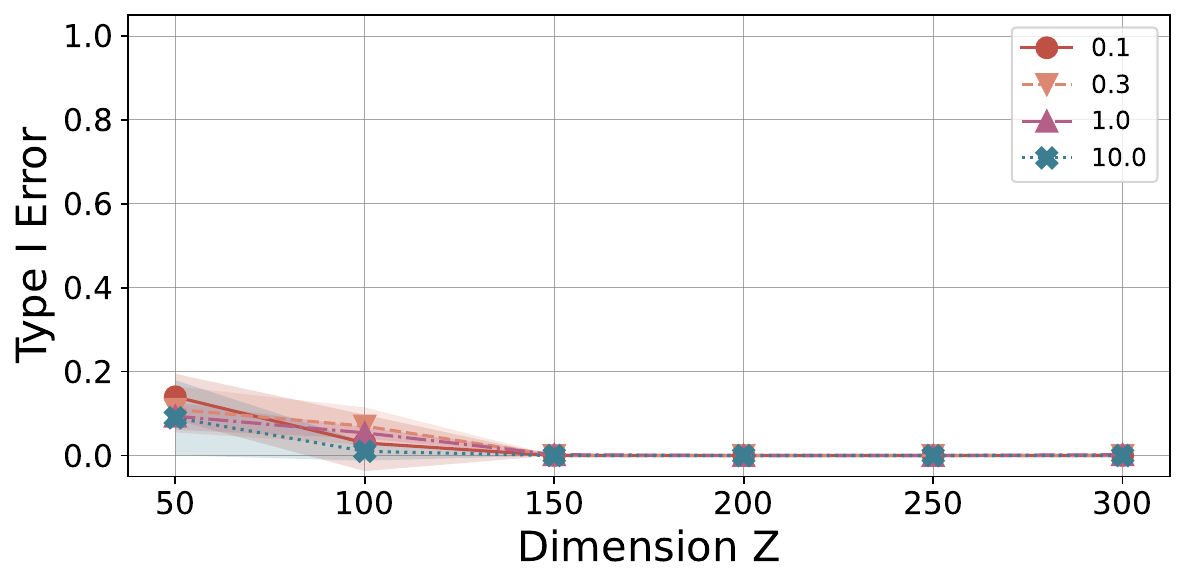}
    \end{minipage}
    \hfill
    \begin{minipage}[b]{0.48\textwidth}
        \centering
        \includegraphics[width=\textwidth]{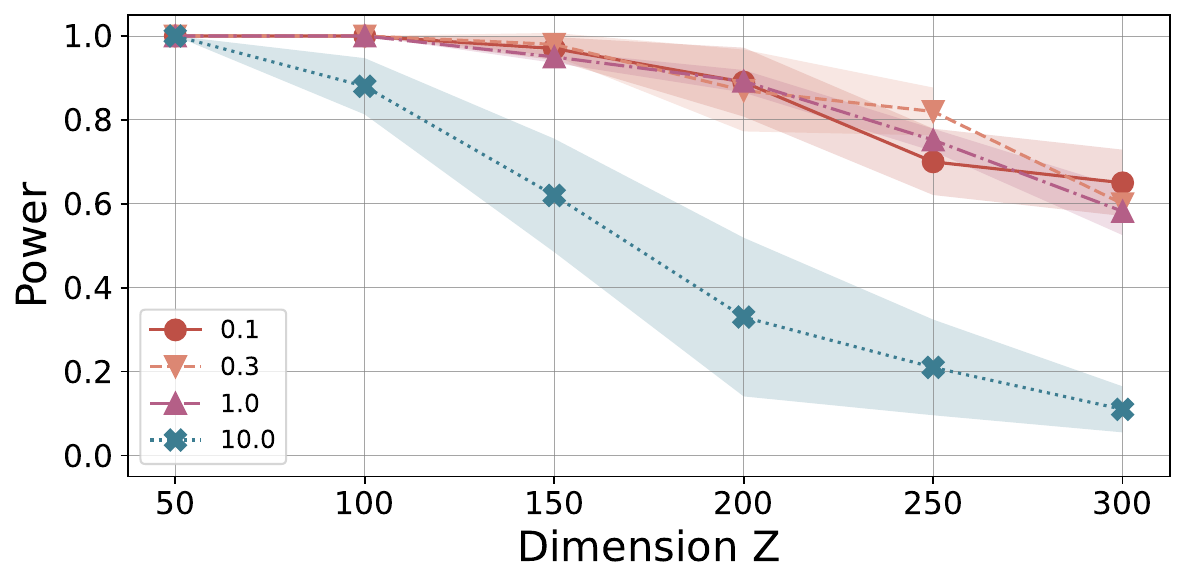}
    \end{minipage}
    \caption{Type I error and power of our method (\textbf{SpectralCIT}) varying the regularization strength $\gamma$ as $\alpha \gamma_*$, where $\alpha \in [0.1, 0.3, 1, 10]$.}
    \label{fig:ablation_reg_str}
\end{figure}

{\bf Computational complexity.}
We further examine the computational cost of \textbf{SpectralCIT} relative to existing kernel-based CI tests.
Experiments are conducted on the post-nonlinear model of \citet{shi2021doublegenerative}, using the same setup as in \Cref{fig:dgcit_kernels}. Each configuration is repeated 10 times, and we report mean and standard deviation results in log seconds.
We fix the sample size to $N = 1000$ and vary the conditioning dimension $d_Z \in \{50, 100, 200, 300\}$.
As shown in \Cref{fig:computational_cost}, \textbf{SpectralCIT} is consistently about $2\times$ faster than \textbf{KCIT} across all values of $d_Z$.
This behavior is expected: although \textbf{SpectralCIT} incurs an additional cost for representation learning, this cost is amortized, and the overall complexity is dominated by low-rank operations in the truncation dimension $d$, rather than by kernel matrix computations scaling with $N$.
While \textbf{RCIT} achieves lower runtime, it fails to maintain Type I error control in high dimensions (cf. \Cref{fig:benchmark_rebuttal_ci}).
Overall, \textbf{SpectralCIT} provides the best trade-off between computational efficiency and validity.

\begin{figure}[h]
 \centering
\includegraphics[width=0.6\textwidth]{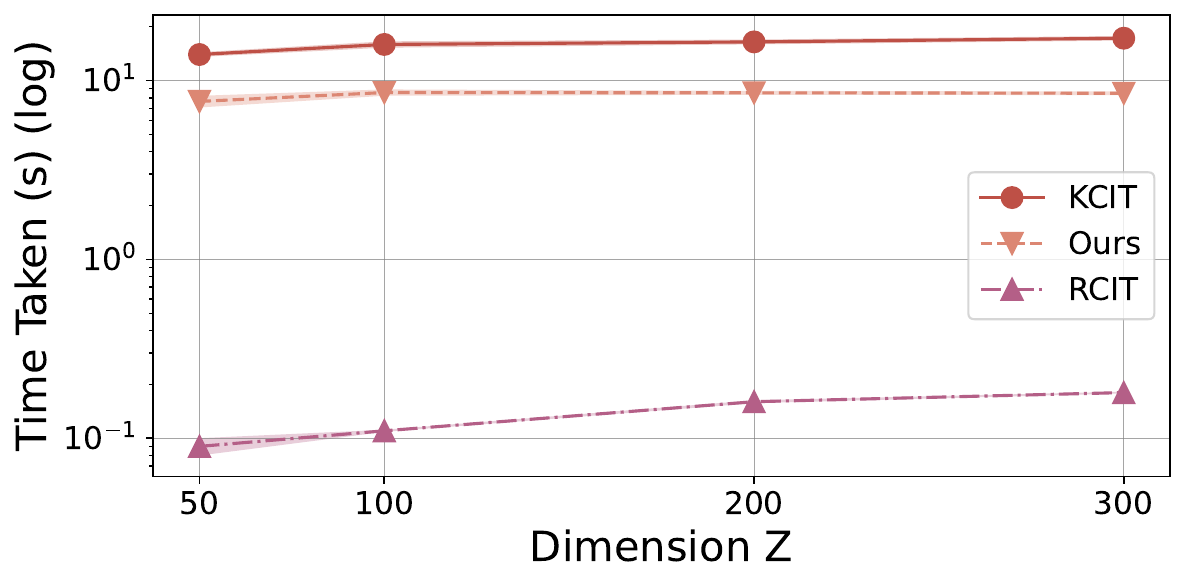}
\caption{Computational time of our method in log-seconds against other kernel methods.}
\label{fig:computational_cost}
\end{figure}

{\bf New high-dimensional benchmark.} %
To strengthen the empirical evaluation, we introduce a high-dimensional benchmark in which the dimensions of $X$ and $Y$ scale with that of $Z$, and the shared noise has low variance (standard deviation $0.15$). We consider the data-generating process
\begin{align*}
    \mathcal{H}_0:  X &= \sin(Z + \varepsilon_X), \quad Y = \cos(Z + \varepsilon_Y), \\
    \mathcal{H}_1: X &= \sin(Z + \varepsilon_X) + \eta, \quad
    Y = \cos(Z + \varepsilon_Y) + \eta,
\end{align*}
where $Z\sim \mathcal{N}(0, \texttt{str\_Z}^2 I_{d_z})$, $\varepsilon_X \sim \mathcal{N}(0, \texttt{noise\_str}^2 I_{d_z})$, $\varepsilon_Y \sim \mathcal{N}(0, \texttt{noise\_str}^2 I_{d_z})$, and $\eta \sim \mathcal{N}(0, \texttt{str\_con\_dep}^2I_{d_z})$ are independent.
We set $\texttt{str\_Z} =0.1$, $\texttt{noise\_str} = 0.25$, and $\texttt{str\_cond\_dep} = 0.15$.
For each dimension, we run each method over 500 repetitions and report the mean and standard error of the Type I error and power (see \Cref{fig:benchmark_rebuttal_ci}).
Hyperparameters are identical to those in the original benchmark.In this more challenging regime, competing methods exhibit loss of validity, whereas \textbf{SpectralCIT} remains valid, albeit slightly conservative.

\begin{figure}[h]
    \centering
    \begin{minipage}[b]{0.48\textwidth}
        \centering
        \includegraphics[width=\textwidth]{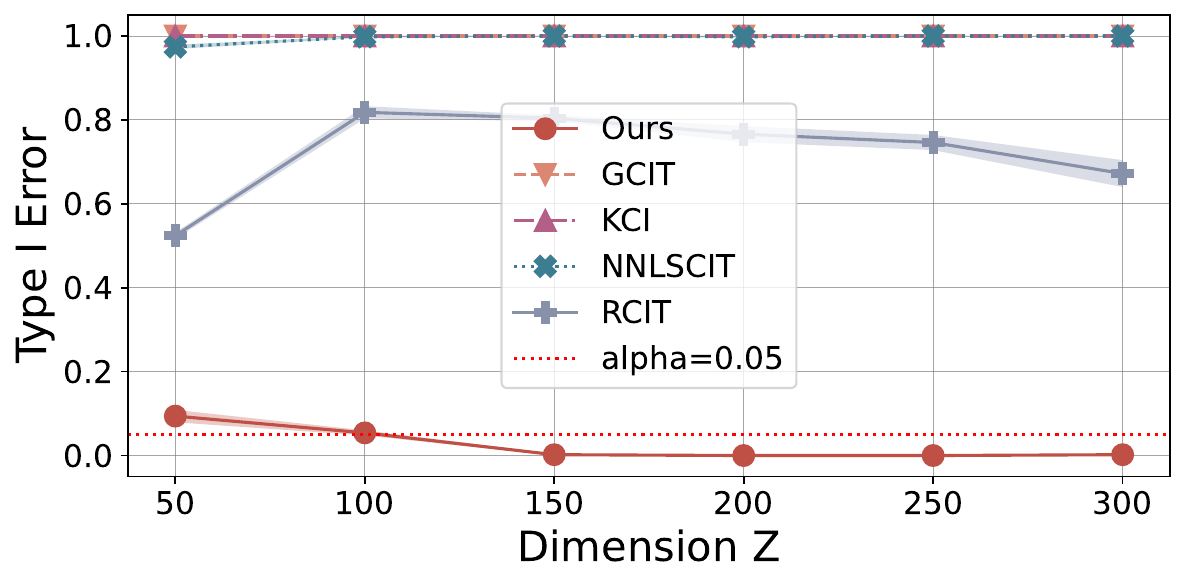}
    \end{minipage}
    \hfill
    \begin{minipage}[b]{0.48\textwidth}
        \centering
        \includegraphics[width=\textwidth]{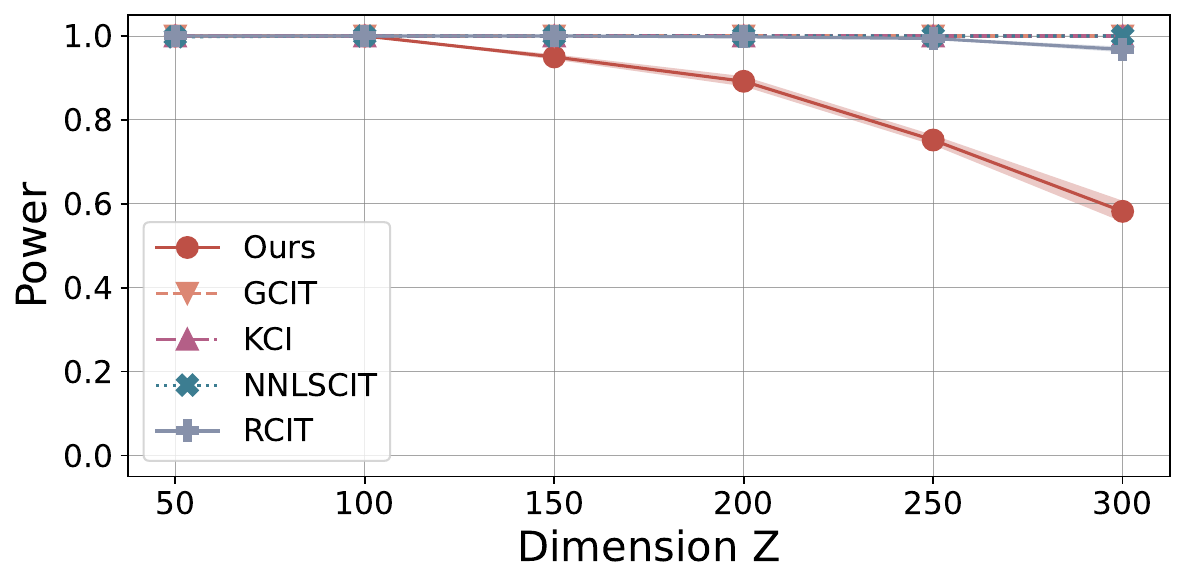}
    \end{minipage}
    \caption{Type I error and power of our method (\textbf{SpectralCIT}) compared to state-of-the-art CI tests across varying dimensionality of the conditioning set $Z$ on new high-dimensional benchmark. We repeated each experiment 5 times, reporting mean and standard deviation.}
    \label{fig:benchmark_rebuttal_ci}
\end{figure}

{\bf On the role of signal strength.}
We present a final ablation showcasing that the dominant factor governing the power of our method is the signal strength under the alternative, rather than the dimensionality of $Z$.%

Concretely, using the same nonlinear equations as in the high-dimensional benchmark, but now with $d_Z = 3$, we isolate the effect of signal strength in a controlled low-dimensional setting.
We vary the signal strength parameter $\texttt{str\_cond\_dep} \in \{0.05, 0.15, 0.5\}$ while keeping $\texttt{str\_Z}$ and $\texttt{noise\_str}$ fixed as before.
\Cref{fig:ablation_signal_str_benchmark_rebuttal_ci} shows that, as the signal increases, the power of the test improves substantially, while Type I error remains controlled throughout. In particular, weak-signal regimes yield low power, whereas moderate-to-strong signal regimes recover high power.

These observations are consistent with our theoretical guarantees (cf. \Cref{thm:power}), which predict that power is governed by the signal strength relative to representation error. This experiment therefore provides empirical support that signal strength, rather than dimensionality, is the primary driver of performance in this setting.

\begin{figure}[H]
    \centering
    \begin{minipage}[b]{0.48\textwidth}
        \centering
        \includegraphics[width=\textwidth]{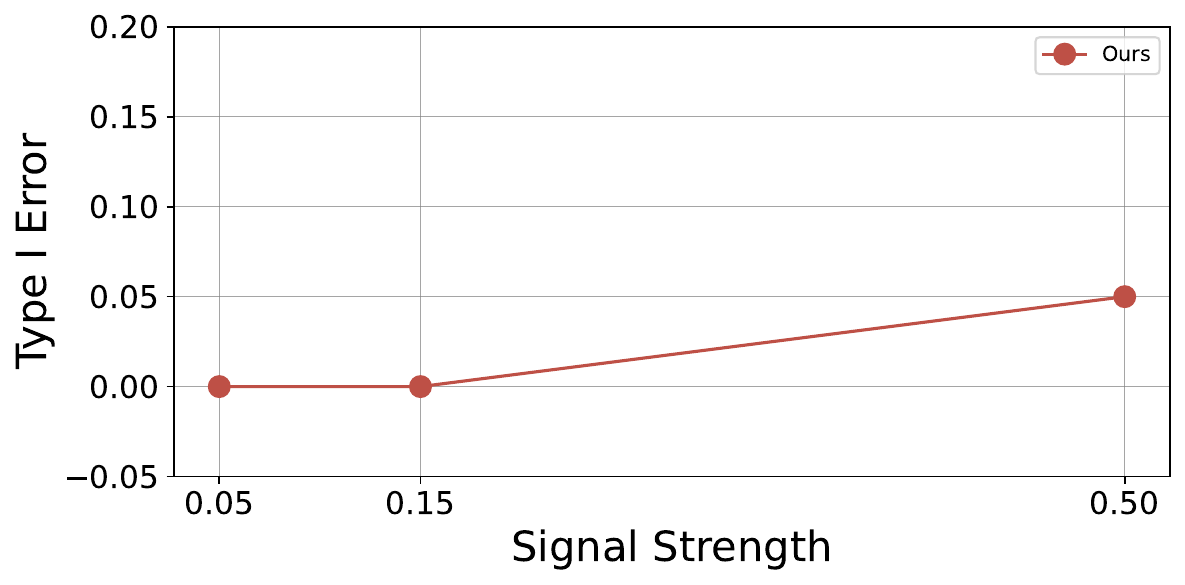}
    \end{minipage}
    \hfill
    \begin{minipage}[b]{0.48\textwidth}
        \centering
        \includegraphics[width=\textwidth]{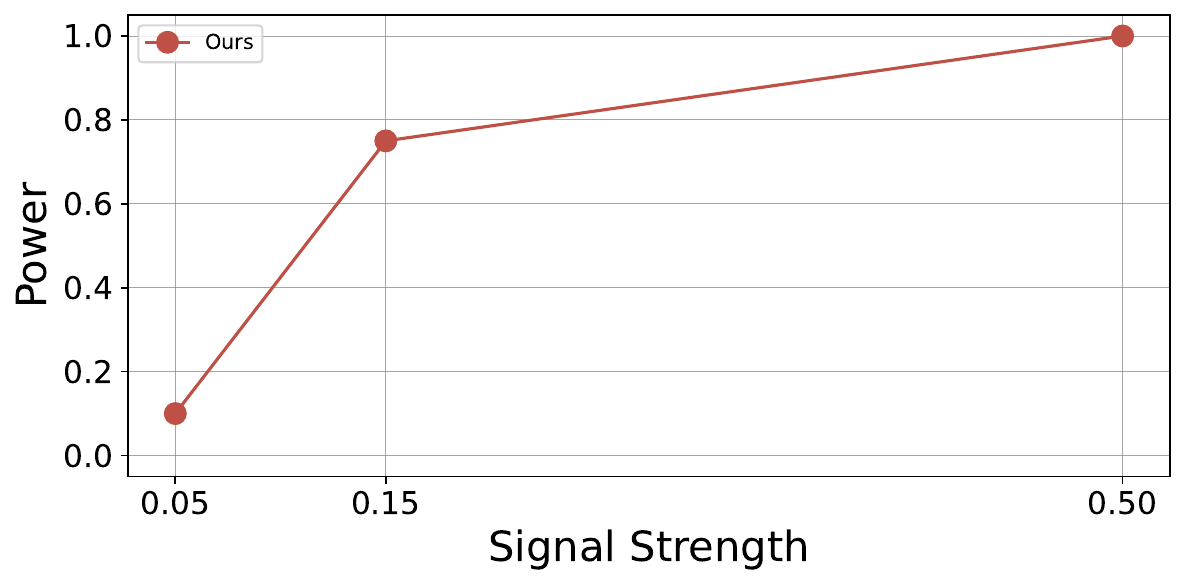}
    \end{minipage}
    \caption{Type I error and power of our method (\textbf{SpectralCIT}) varying signal strength $\texttt{str\_cond\_dep} \in \{0.05, 0.15, 0.5\}$ on data model with $d_Z =3$.}
    \label{fig:ablation_signal_str_benchmark_rebuttal_ci}
\end{figure}

{\bf CI queries on a Sachs-inspired model.} %
We conduct a small-scale experiment to evaluate \textbf{SpectralCIT} as a conditional independence (CI) oracle within the PC algorithm. The goal is not a full PC benchmark, but to assess reliability in the regime of large conditioning sets, where kernel-based methods typically degrade. The setup is based on the 7-node Sachs-inspired DAG shown in \Cref{fig:sachs_dag} (cf. \citep{Sachs2005}), with post-nonlinear structural equations as follows.
Under $\mathcal{H}_0$, there are no additional edges beyond those in $\mathcal{G}$:%
\begin{align}%
Z_1 &\sim \mathcal{N}(0, I_{d_z}), \nonumber \\
X_i &= \tanh(\Bar{Z_1} + \varepsilon_i),
     \quad i = 1,2, \nonumber \\
X_j &= A_j Z_1 + \varepsilon_j, 
     \quad j = 3,4, \nonumber \\
Y_2 &= \cos(\Bar{X_3} + \Bar{X_4} + \varepsilon_{Y_2}),
     \nonumber
\end{align}%
where all noise terms $\varepsilon_\cdot \sim \mathcal{N}(0, 0.25)$  are independent.
Under $\mathcal{H}_1$, we plant one additional edge $X_1 - Y_2$ to create  a spurious dependency between the two downstream  branches by adding a shared noise $\eta \sim \mathcal{N}(0, 0.25)$ between $X_1$ and $Y_2$.%

We evaluate three CI queries representative of those a constraint-based algorithm would encounter, under both $\mathcal{H}_0$ and $\mathcal{H}_1$. Under $\mathcal{H}_1$, we plant the direct edge $X_1 \to Y_2$. Results are in \Cref{tab:sachs-small-ci-dz100}. Type I error is $0.00$ on all three queries. For $X_1 \perp Y_2 \mid X_3, X_4$ and $X_1 \perp Y_2 \mid X_3, X_4, Z_1$, the planted edge renders $\mathcal{H}_1$ non-trivial: power is $1.00$ and $0.75$ respectively, the latter at $d_z = 100$ where adding the hub to the conditioning set makes detection harder. For $Y_1 \perp Y_2 \mid X_1, X_2, X_3, X_4$, the planted edge does not alter the null, so $\mathcal{H}_1$ coincides with $\mathcal{H}_0$ and only Type I error is relevant.%

While limited in scope, these results support the relevance of SpectralCIT to constraint-based causal discovery, particularly in the large-conditioning-set regime.%

\begin{figure}[h]
\centering
\begin{tikzpicture}[
    node distance=1.8cm,
    every node/.style={
        circle, 
        draw, 
        thick,
        minimum size=0.9cm,
        font=\small
    },
    every edge/.style={
        draw, 
        thick, 
        -stealth
    }
]

% Nodes
\node (Z1)              {$Z_1$};
\node (X1) [below left=1.5cm and 2.2cm of Z1]  {$X_1$};
\node (X2) [below left=1.5cm and 0.7cm of Z1]  {$X_2$};
\node (X3) [below right=1.5cm and 0.7cm of Z1] {$X_3$};
\node (X4) [below right=1.5cm and 2.2cm of Z1] {$X_4$};
\node (Y1) [below=1.5cm of X2]                 {$Y_1$};
\node (Y2) [below=1.5cm of X3]                 {$Y_2$};

% Edges from Z1 (hub)
\draw[->] (Z1) -- (X1);
\draw[->] (Z1) -- (X2);
\draw[->] (Z1) -- (X3);
\draw[->] (Z1) -- (X4);

% Edges to Y1
\draw[->] (X1) -- (Y1);
\draw[->] (X2) -- (Y1);

% Edges to Y2
\draw[->] (X3) -- (Y2);
\draw[->] (X4) -- (Y2);

% Planted edge under H1 (dashed red)
\draw[->, dashed, red, thick] (X1) 
    to[out=-60, in=150] 
    node[midway, above, 
         draw=none, 
         fill=none,
         text=red, 
         font=\footnotesize] {$\mathcal{H}_1$ only} 
    (Y2);

\end{tikzpicture}
\caption{
Sachs-inspired DAG on 7 nodes. $Z_1$ acts as a central hub (analogous to PKA), with four children $X_1, \ldots, X_4$. 
The nodes $Y_1$ and $Y_2$ are downstream outcomes with two parents each, creating v-structures. 
The dashed red edge $X_1 \to Y_2$ is present only under the alternative hypothesis $\mathcal{H}_1$.
}
\label{fig:sachs_dag}
\end{figure}

\begin{table}[h]
\centering
\caption{Results for three CI queries over the Sachs-inspired causal graph with $d_z = 100$.}
\label{tab:sachs-small-ci-dz100}
\begin{tabular}{@{}lll@{}}
\toprule
CI query & Type I error ($\mathcal{H}_0$) & Power ($\mathcal{H}_1$) \\
\midrule
$X_1 \perp\!\!\!\perp Y_2 \mid \{X_3, X_4\}$ & $0.00$ & $1.00$ \\
$X_1 \perp\!\!\!\perp Y_2 \mid \{X_3, X_4, Z_1\}$ & $0.00$ & $0.75$ \\
$Y_1 \perp\!\!\!\perp Y_2 \mid \{X_1, X_2, X_3, X_4\}$ & $0.00$ & -- \\
\bottomrule
\end{tabular}
\end{table}

%%%%%%%%%%%%%%%%%%%%%%%%%%%%%%%%%%%%%%%%%%%%%%%%%%%%%%%%%%%%%%%%%%%%%%%%%%%%%%%
%%%%%%%%%%%%%%%%%%%%%%%%%%%%%%%%%%%%%%%%%%%%%%%%%%%%%%%%%%%%%%%%%%%%%%%%%%%%%%%

\end{document}